\documentclass[11pt]{article} 
\usepackage[utf8]{inputenc} 
\usepackage{multirow}

\usepackage[margin=1.25in]{geometry} 
\usepackage{graphicx} 

\usepackage{booktabs} 
\usepackage{array} 
\usepackage{paralist} 
\usepackage{verbatim} 
\usepackage{mathrsfs} 
\usepackage{amssymb}
\usepackage{amsthm}
\usepackage{amsmath,amsfonts,amssymb}
\usepackage{esint}
\usepackage{graphics}
\usepackage{enumerate}
\usepackage{mathtools}
\usepackage{xfrac}
\usepackage{bbm}

\usepackage[nocompress]{cite}

\usepackage{subcaption}
\usepackage{stmaryrd} 
\usepackage{lmodern} 
\usepackage{anyfontsize} 

\let\oldsquare\square 

\usepackage{mathabx}

\renewcommand{\square}{\oldsquare}

\usepackage{tocloft}

\usepackage[usenames,dvipsnames]{xcolor}
\usepackage[colorlinks=true, pdfstartview=FitV, linkcolor=blue, citecolor=blue, urlcolor=blue]{hyperref}
\usepackage[normalem]{ulem}

\usepackage{thm-restate}

\newcommand{\la}{\lambda}
\newcommand{\NN}{\mathbb{N}}

\newcommand{\rred}{\color{red}}

\numberwithin{equation}{section}

\newtheorem{theorem}{Theorem}[section]

\newtheorem{corollary}[theorem]{Corollary}
\newtheorem{proposition}[theorem]{Proposition}
\newtheorem{lemma}[theorem]{Lemma}
\theoremstyle{definition}
\newtheorem{definition}[theorem]{Definition}

\newtheorem{remark}[theorem]{Remark}

\newtheorem{assumption}{Assumption}

\let\originalleft\left
\let\originalright\right
\renewcommand{\left}{\mathopen{}\mathclose\bgroup\originalleft}
\renewcommand{\right}{\aftergroup\egroup\originalright}

\newcommand{\vertiii}{\vert\kern-0.3ex\vert\kern-0.25ex\vert}

\newcommand*{\N}{\ensuremath{\mathbb{N}}}
\newcommand*{\Z}{\ensuremath{\mathbb{Z}}}
\newcommand*{\R}{\ensuremath{\mathbb{R}}}
\newcommand*{\Zd}{\ensuremath{\mathbb{Z}^d}}
\newcommand*{\Rd}{\ensuremath{\mathbb{R}^d}}
\newcommand{\eps}{\varepsilon}

\renewcommand*{\tilde}{\widetilde}
\renewcommand*{\hat}{\widehat}

\DeclareSymbolFont{boldoperators}{OT1}{cmr}{bx}{n}
\SetSymbolFont{boldoperators}{bold}{OT1}{cmr}{bx}{n}

\newcommand{\cu}{\square}
\newcommand{\F}{\mathcal{F}}
\renewcommand{\P}{\mathbb{P}}
\newcommand{\E}{\mathbb{E}}
\newcommand{\X}{\mathcal{X}}

\renewcommand{\O}{\mathcal{O}}

\newcommand{\indc}{1}

\DeclareMathOperator{\var}{var}

\newcommand{\RR}{\mathbb{R}}

\newcommand{\e}{\varepsilon}

\newcommand{\avsum}{\mathop{\mathpalette\avsuminner\relax}\displaylimits}

\makeatletter
\newcommand\avsuminner[2]{%
	{\sbox0{$\m@th#1\sum$}%
		\vphantom{\usebox0}%
		\ooalign{%
			\hidewidth
			\smash{\,\rule[.23em]{8.8pt}{1.1pt} \relax}%
			\hidewidth\cr
			$\m@th#1\sum$\cr
		}%
	}%
}
\makeatother

\makeatletter
\newcommand\avsuminnerr[2]{%
	{\sbox0{$\m@th#1\sum$}%
		\vphantom{\usebox0}%
		\ooalign{%
			\hidewidth
			\smash{\,\rule[.23em]{6pt}{0.7pt} \relax}%
			\hidewidth\cr
			$\m@th#1\sum$\cr
		}%
	}%
}
\makeatother

\def\XXint#1#2#3{{\setbox0=\hbox{$#1{#2#3}{\int}$}
		\vcenter{\hbox{$#2#3$}}\kern-.5\wd0}}

\makeatletter 
\newcommand{\negphantom}{\v@true\h@true\negph@nt} 
\newcommand{\neghphantom}{\v@false\h@true\negph@nt} 
\newcommand{\negph@nt}{\ifmmode\expandafter\mathpalette 
	\expandafter\mathnegph@nt\else\expandafter\makenegph@nt\fi} 
\newcommand{\makenegph@nt}[1]{%
	\setbox\z@\hbox{\color@begingroup#1\color@endgroup}\finnegph@nt} 
\newcommand{\finnegph@nt}{%
	\setbox\tw@\null 
	\ifv@ \ht\tw@\ht\z@\dp\tw@\dp\z@\fi \ifh@\wd\tw@-\wd\z@\fi\box\tw@} 
\newcommand{\mathnegph@nt}[2]{%
	\setbox\z@\hbox{$\m@th #1{#2}$}\finnegph@nt} 
\makeatother

\usepackage{scalerel,stackengine}
\stackMath
\newcommand\reallywidehat[1]{%
	\savestack{\tmpbox}{\stretchto{%
			\scaleto{%
				\scalerel*[\widthof{\ensuremath{#1}}]{\kern.1pt\mathchar"0362\kern.1pt}%
				{\rule{0ex}{\textheight}}
			}{\textheight}%
		}{2.4ex}}%
	\stackon[-8.5pt]{#1}{\tmpbox}%
}

\usepackage{titlesec}

\newcommand{\addperiod}[1]{#1.}
\titleformat{\section}
{\normalfont\Large}{\thesection.}{0.5em}{}
\titleformat*{\subsection}{\normalfont\large}
\titleformat{\subsubsection}[runin]
{\bfseries}
{\thesubsubsection.}
{0.5em}
{\addperiod}
\titleformat*{\subsubsection}{\bfseries}
\titleformat*{\paragraph}{\bfseries}
\titleformat*{\subparagraph}{\large\bfseries}

\usepackage{algorithm}
\usepackage{algpseudocode}
\newcommand{\chenghui}[1]{\textcolor{cyan}{Chenghui: #1}}
\newcommand{\KL}{\mathrm{KL}}

\definecolor{Blue}{rgb}{0,0,1}
\definecolor{Red}{rgb}{1,0,0}

\newcommand{\veps}{\e_n}
\newcommand{\dx}{\, \dd \vx}

\newcommand{\dt}{\, \dd t}

\newcommand{\vol}{\mathrm{vol}}
\newcommand{\divergence}{\mathrm{div}}

\newcommand{\vc}{\mathbf{c}}
\newcommand{\vx}{x}

\newcommand{\0}{\mathbf{0}}

\newcommand{\T}{T}

\newcommand{\Information}{\mathrm{I}}

\newcommand{\Cut}{\mathrm{Cut}}

\newcommand{\x}{x}

\DeclareMathOperator{\Prob}{\mathbb{P}}
\newcommand{\M}{\mathcal{M}}

\newcommand{\MM}{\mathbf{M}}

\definecolor{mygreen}{rgb}{0.1,0.75,0.2}

\newcommand{\nc}{\normalcolor}
\renewcommand{\div}{\mathrm{div}}

\renewcommand{\P}{\mathbb{P}}

\makeatletter
\newcommand{\defeq}{:=}
\newcommand{\aban@defeq}{%
	\vbox{\offinterlineskip\check@mathfonts
		\ialign{\hfil##\hfil\cr
			\fontsize{\ssf@size}{\z@}\normalfont def\cr
			\noalign{\kern1\p@}
			$\m@th=$\cr
			\noalign{\kern-.5\fontdimen22\textfont2}
		}%
	}%
}

\newcommand{\dd}{\mathrm{d}}

\def\La{\mathcal{L}}

\definecolor{darkred}{rgb}{0.6,0.1,0.1}
\definecolor{darkgreen}{rgb}{0.1,0.6,0.1}
\definecolor{darkblue}{rgb}{0.1,0.1,0.6}

\title{\bf \Large Minimax Rates for the Estimation of Eigenpairs of Weighted Laplace-Beltrami Operators on Manifolds}

\author{Nicol\'as Garc\'ia Trillos
	\thanks{Department of Statistics, University of Wisconsin-Madison.
		{\footnotesize {garciatrillo@wisc.edu}.}
	}
	\and 
	Chenghui Li
	\thanks{Department of Statistics, University of Wisconsin-Madison.
		{\footnotesize {cli539@wisc.edu}.}
	}
	\and 
	Raghavendra Venkatraman
	\thanks{Department of Mathematics, The University of Utah.
		{\footnotesize {raghav@math.utah.edu}.}
	}
}
\date{\today}

\usepackage[nottoc,notlot,notlof]{tocbibind}

\begin{document}
	
	\maketitle
	
	\begin{abstract}
		We study the problem of estimating eigenpairs of elliptic differential operators from samples of a distribution $\rho$ supported on a manifold $\M$. The operators discussed in the paper are relevant in unsupervised learning and in particular are obtained by taking suitable scaling limits of widely used graph Laplacians over data clouds. We study the minimax risk for this eigenpair estimation problem and explore the rates of approximation that can be achieved by commonly used graph Laplacians built from random data. More concretely, assuming that $\rho$ belongs to a certain family of distributions with controlled second derivatives, and assuming that the $d$-dimensional manifold $\M$ where $\rho$ is supported has bounded geometry, we prove that the statistical minimax rate for approximating eigenvalues and eigenvectors in the $H^1(\M)$-sense is $n^{-2/(d+4)}$, a rate that matches the minimax rate for a closely related density estimation problem. To the best of our knowledge, our results are the first statistical lower bounds for this type of eigenpair estimation. We then revisit the literature studying Laplacians over proximity graphs in the large data limit and prove that, under slightly stronger regularity assumptions on the data generating model, eigenpairs of these graph-based operators can induce manifold agnostic estimators with an error of approximation that, up to logarithmic corrections, matches our lower bounds, providing in this way a concrete statistical basis for the claim that graph Laplacian based estimators are, essentially, optimal for this estimation problem. In addition, our analysis allows us to expand the existing literature on graph-based learning in at least two significant ways: 1) we consider stronger norms to measure the error of approximation than the ones that had been analyzed in the past; 2) our rates of convergence are uniform over a family of smooth distributions and do not just apply to densities with special symmetries, and, as a consequence of our lower bounds, are essentially sharp when the connectivity of the graph is sufficiently high.
		
		Our work showcases novel connections between PDE theory and statistics and draws inspiration from recent advances in quantitative homogenization of partial differential equations on random media, here in the setting of random geometric graphs. It also sets the stage for the analysis of other important estimation problems that we believe are of relevance to the modern literatures of operator learning and inverse problems.

		\nc
	\end{abstract}
	
	\setcounter{tocdepth}{2}  
	\tableofcontents

	\section{Introduction}


	Let $x_1, \dots, x_n$ be data points randomly sampled from a distribution $\rho$ supported on a $d$-dimensional manifold $\M$ embedded in $\R^D$. One of the fundamental constructions in manifold learning and graph based learning is the graph Laplacian operator
	\begin{equation}
	\Delta_n u(x_i) \defeq  \frac{1}{n\veps^{d+2}}\sum_{j=1}^n \eta\left(\frac{|x_i- x_j|}{\veps} \right)( u(x_i) - u(x_j)), \quad i=1, \dots, n,
	\end{equation} 
	where $\veps$ is a suitably chosen connectivity parameter, $\eta$ is a non-negative and decreasing function, and $\lvert \cdot \rvert$ denotes the Euclidean norm in the ambient space $\R^D$. The spectra of these operators, or of normalized versions thereof, are prominently used in machine learning, specifically in unsupervised learning settings, where they are at the core of very important and widely used data analysis methodologies such as spectral clustering and diffusion maps; see \cite{von2007tutorial,li2019metrics,koelle2024consistency,coifman2006diffusion}.
	
	At least since as far back as \cite{belkin2005towards}, multiple works in the literature have explored the connection between the graph Laplacian $\Delta_n$, its spectrum, and  differential operators of the type
	\[  \Delta_\rho f \defeq  - \frac{1}{\rho} \divergence(\rho^2 \nabla f), \quad f \in C^2(\M),\]
	and their spectra; some examples of these works include \cite{Shi2015,trillos2019error,calder2019improved} but a more detailed list will be presented in our literature review in Section \ref{sec:literature}. Under different norms, and under different assumptions on the manifold $\M$ and the density $\rho$ used to generate the data, all these works provide high probability \textit{upper bounds} on the error of estimating eigenpairs of $\Delta_\rho$ with eigenpairs of graph Laplacians constructed from samples drawn from $\rho$. While these works have developed a remarkable set of tools for the analysis of such estimation problems, including tools borrowed from fields such as the calculus of variations, analysis of PDEs, optimal transport, and probability theory, their estimates are not sharp and there is still room for improvement. More importantly, to the best of our knowledge, no work in the literature has explored the following fundamental statistical question:
	\begin{center}
		\textit{What is the best estimator, built from finitely many observations, for the eigenpairs of the operator $\Delta_\rho$ when the distribution $\rho$ used to sample the data is unknown?}
	\end{center}
	It is reasonable to ask this question since there may exist multiple ways, perhaps some yet unexplored, for estimating eigenpairs of the operator $\Delta_\rho$ that do not rely on computing the eigenvalues and eigenvectors of a standard graph Laplacian built from data. For practitioners, it is important to know whether there exist other, more statistically efficient methods to tackle this estimation problem. 
	
	In this work, we follow the framework of minimax theory in statistical analysis and make precise modeling assumptions on the family of data generating distributions under which we:
	\begin{enumerate}
		\item Provide a sharp minimax rate for the estimation of eigenpairs of $\Delta_\rho$. In particular, from our analysis we can conclude that \textit{there is no estimator} built from $n$ data points that can approximate eigenpairs of $\Delta_\rho$ at a rate in $n$ that is faster than our minimax rate. Our rates are sharp in the sense that, when the manifold $\M$ is known and only the distribution $\rho$ is unknown, we can exhibit a concrete and mathematically well-defined estimator, albeit impractical in applications, whose error of approximation matches our lower bound. 
		\item Present new theoretical developments in the study of estimators based on the extraction of eigenpairs of graph Laplacians by developing several new ideas motivated by the theory of homogenization of elliptic PDEs. Our analysis allows us to prove that graph Laplacian based estimators are, up to logarithmic factors and assuming slightly more regularity on the data generating model, optimal for the eigenpair estimation problem, with the advantage that they are \textit{manifold agnostic}. As discussed in Sections \ref{ss.ideas} and \ref{sec:literature}, our work substantially improves the analysis of existing papers in the literature on manifold learning. In addition, thanks to our lower bounds, our analysis of graph-Laplacian based estimators can be seen to be essentially sharp when the graph has sufficiently high connectivity.
		
		%
		
	\end{enumerate}

	In order to prove our results, we take advantage of a novel combination of PDE tools and statistical analysis that we believe is of interest to the broader statistics, numerical analysis, and machine learning communities. Indeed, we believe that the questions explored in this paper, as well as the techniques used to address them, are of importance beyond the specific estimation problem described here and in particular are of relevance to the modern literatures of operator learning and inverse problems. 
	%
	We proceed now to present the precise setting for our estimation problem.

	\subsection{Setup}
	
	As discussed previously, our goal is to use random samples from a distribution $\rho$ supported on a low dimensional manifold $\M$ to estimate eigenvalues and eigenfunctions of the weighted Laplace-Beltrami operator $\Delta_\rho$, where, to avoid introducing cumbersome notation, we interpret $\divergence$ and $\nabla$ as the divergence and gradient operators on the manifold $\M$ where $\rho$ is supported (to see how the operator $\Delta_\rho$ looks like in suitable coordinates, see Appendix \ref{App:GeoBack}, specifically \eqref{eqn:LaplacianCoordinates}.) In this section, we make our estimation problem mathematically precise.

	Throughout the paper, we assume that the manifold $\M$ belongs to the following family of manifolds with bounded geometry.
	\begin{definition}[Manifold class $\MM$]
		Given $K, R, i_0, L_I, L_E$ positive constants, and $\alpha \in [0,1]$, we denote by $\MM$ the family of 
		$d$-dimensional manifolds $\M$ embedded in $\R^D$ that are smooth, compact, orientable, connected, and have no boundary, and in addition satisfy: the total volume of $\M$ is 1 (according to $\M's$ volume form), the Riemannian sectional curvature is in absolute value bounded by a constant $K$, $\M$ has reach bounded by the constant $R$, $\M$'s injectivity radius is lower bounded by $i_0$, the rate of change of sectional curvature is $\alpha$-H\"older continuous {(see~\eqref{eqn:RegularGeometry1} for a precise definition)}, and the rate of change of second fundamental form of $\M$ is bounded by $L_E$ {(see~\eqref{e.geodesic third derivative expansion} for a precise definition)}.
		\label{def:ManifoldClass}
	\end{definition} 
	In Appendix \ref{App:GeoBack}, we provide a brief review of some basic notions in Riemannian geometry that we use through the paper and in particular discuss some of the geometric quantities appearing in the previous definition.
	
	\medskip
	
	Given a manifold $\M$ in the class $\MM$, we assume $x_1, \dots, x_n$ to be samples drawn from a distribution supported on $\M$ with probability density function $\rho: \M \rightarrow \R$ with respect to $\M$'s volume form. We henceforth use $\dx$ to denote integration with respect to the Riemannian volume form associated with $\M$ whenever no confusion arises from doing so. The density $\rho$ will be assumed to belong to the class of regular densities $\mathcal{P}_\M$ defined next.
	
	\begin{definition}[Density class $\mathcal{P}_\M$] 
		\label{def:DensityClass}
		For fixed positive constants $\rho_{\min}, \rho_{\max}, c_1, c_2$, and for any given manifold $\M$ in the class $\MM$, we denote by $\mathcal{P}_\M$ the class of probability density functions $\rho: \M \rightarrow \R$ that are $C^2(\M)$ and satisfy: 
		\begin{equation}\label{eq:rho bound}
		\rho_{\min} \leq \rho(x) \leq \rho_{\max},
		\end{equation}
		and $\rho$'s first and second derivatives are upper bounded by the constants $c_1, c_2$ as follows:
		\begin{align}
		\label{eq:BoundDerivatives}
		\|\nabla\rho\|_{L^\infty(\M)} \le c_1\, , \quad \|\nabla^2\rho\|_{L^\infty(\M)} \le c_2.
		\end{align}
		\label{def:DensityClass1}
	\end{definition}

	The family of densities $\mathcal{P}_\M$ introduced above appears as a natural model class in well-studied non-parametric estimation problems such as density estimation \cite{kerkyacharian1992density,birge1995estimation}. It has also been used as a concrete setting in the theoretical analysis of graph Laplacians in the large data limit; see, e.g., \cite{trillos2019error,calder2019improved,trillos2023large} and references therein. In order to formulate our estimation problem in a simpler, yet interesting enough setting, we will restrict the family of models slightly and incorporate an eigengap assumption on the operators $\Delta_\rho$. To describe this additional constraint, it will be helpful to first recall some notions associated to these operators and discuss some basic properties.

	Standard results in the theory of elliptic operators (see, e.g., \cite{evans}) imply that, under the assumption that $\rho \in \mathcal{P}_\M$, the operator $\Delta_\rho$ is self-adjoint with respect to the inner product
	\[ \langle  f , g \rangle_{L^2(\M, \rho)}\defeq \int_{\M} f(\vx) g(\vx) \rho \dx,  \quad f, g \in L^2(\M, \rho),  \] 
	and there exists a complete \textit{orthonormal} family $\{ f_l \}_{l \in \N}$ (w.r.t. $L^2(\M, \rho)$ inner product) of eigenfunctions of $\Delta_\rho$ with corresponding eigenvalues 
	\[ 0 = \lambda_1 < \lambda_2 \leq \dots\leq \lambda_l \leq \dots  \rightarrow \infty \]
	that may be repeated according to multiplicity; in the sequel, we often use the notation  ${L}^2(\rho)$ to denote the space $L^2(\M,\rho)$ whenever the manifold $\M$ is clear from context. Explicitly, the eigenpair $(\lambda_l, f_l)$ satisfies the equation:
	\begin{equation}
	\label{eq:eigenform}
	\Delta_{\rho} f_l  =\lambda_l f_l.
	\end{equation}
	From basic regularity theory of elliptic partial differential equations (see, e.g., \cite{evans}) it follows that equation \eqref{eq:eigenform} holds at every point $x \in \M$.
	
	A quadratic form, often referred to as \textit{Dirichlet form} or energy, can be associated to the operator $\Delta_\rho$. Precisely, 
	\begin{equation}\label{Equ: dirichlet energy}
	D_\rho(f)\defeq \begin{cases} \int_{\M} |\nabla f(\vx)|^2 \rho^2(\vx) \dx, & \text{ if } f \in H^1(\M), \\  +\infty, & \text{ if } f \in L^2(\M) \setminus H^1(\M), \end{cases}
	\end{equation}
	where here and in the remainder we use $H^1(\M)$ to denote the space of (equivalence classes of) square integrable (w.r.t. $\dx$) real-valued functions on $\M$ with distributional derivatives that are also square integrable. For an element $f \in H^1(\M)$, its $H^1(\M)$ norm is defined according to
	\[ \lVert f \rVert_{H^1(\M)} \defeq  \left(\int_{\M} |f(x)|^2 \dx\right)^{1/2} + \left(\int_{\M} |\nabla f(x)|^2 \dx\right)^{1/2};   \]
	the second term on the right hand side of the above expression will be referred to as the $H^1(\M)$ \textit{semi}-norm of $f$. In terms of the Dirichlet energy $D_\rho$, the eigenvalues of $\Delta_\rho$ admit the following variational formulation 
	\begin{equation}
	\label{eq:minimax principle for laplacian}
	\lambda_{l}=\min _{S \in \mathfrak{S}_{l}} \max _{f \in S^\perp \backslash\{0\}} \frac{D_\rho(f)}{\|f\|_{L^{2}(\M, \rho)}^{2}}\,,
	\end{equation}
	where $\mathfrak{S}_{l}$ denotes the set of all linear subspaces of $L^2(\rho)$ of dimension $l \in \N$. 
	
	
	

	\medskip
	
	For fixed $l$, we define the \textit{spectral gap} $\gamma_l$ associated to $\Delta_\rho$ by
	\begin{equation}\label{eq-def:gamma_l}
	\gamma_l  \defeq  \min_{k\in \mathbb{N}}\{|\lambda_{l} - \lambda_{k}| : \lambda_{l} \neq \lambda_{k} \} \,, 
	\end{equation}
	where $\lambda_{k}, \lambda_l$ are the $k$-th and $l$-th eigenvalues of $\Delta_\rho$. We use this definition to refine the family $\mathcal{P}_\M$ introduced earlier.
	
	\begin{definition}[Density class $\mathcal{P}_{\M, l}$] 
		For a given natural number $l \geq 2$, and a given constant $\gamma>0$, we define $\mathcal{P}_{\M, l}$ to be the family of densities $\rho$ in $\mathcal{P}_\M$ for which 
		\[\gamma_l\ge \gamma.\]
	\end{definition}

	With the above definitions, our estimation problem can now be described as that of approximating, for a fixed $l$, the eigenpairs $(\lambda_l, f_l)$ associated to the operator $\Delta_\rho$ using a collection of $n$ i.i.d. points $\X_n =\{ x_1, \dots, x_n \}$ sampled from the density $\rho$, which we assume belongs to the class $\mathcal{P}_{\M, l}$ for a given $\M \in \MM$. More concretely, in this paper we study the minimax risk
	\begin{equation}\label{eq:eigenpair}
	\inf_{\hat{f}_l,\hat{\lambda}_l}\sup_{\M\in\MM,\rho\in\mathcal{P}_{\M, l}}\E_{\X_n\sim \rho} \left[ |\lambda_l - \hat{\lambda}_l|+ \lVert f_l - \hat{f}_l \lfloor_{\M} \rVert_{H^1(\M)} \right],
	\end{equation}
	where the inf ranges over all measurable maps $(\hat{\lambda}_l , \hat{f}_l)$ from $\R^{D \times n}$ into $\R \times C(\R^D)$; here $C(\R^D)$ denotes the space of real-valued continuous functions on the ambient space $\R^D$. We use $\hat{f}_l \lfloor_{\M} $ to denote the restriction of $\hat{f}_l$ to the manifold $\M$ and use the convention that the $H^1(\M)$ norm of $\hat{f}_l \lfloor_{\M} $ is $\infty$ if this function is not an element of $H^1(\M)$. Under the assumption that $\rho \in \mathcal{P}_{\M, l}$, the eigenfunction $f_l$ is defined up to rotation in the corresponding eigenspace, and we will thus interpret $ \lVert f_l - \hat{f}_l \lfloor_{\M} \rVert_{H^1(\M)}$ as the smallest of the $H^1(\M)$ distances between $\hat{f}_l \lfloor_{\M}$ and the unit-norm eigenfunctions $f_l$ corresponding to $\la_l$; for simplicity, however, we will tacitly assume that the eigenvalue $\lambda_l$ is simple and that the sign of $f_l$ has been chosen to best align with $\hat{f}_l \lfloor_{\M}$. Implicit in the formulation \eqref{eq:eigenpair} is that in our estimation problem the density $\rho$ \textit{and} the manifold $\M$ where $\rho$ is supported are both unknown. We also note that, although here we restrict our attention to the estimation of a single eigenpair of $\Delta_\rho$, it is straightforward to extend the problem to one where the goal is to estimate a fixed finite number of eigenpairs. Our main results, which we present next, can easily be generalized to that setting.  Finally, we note that the motivation for considering a norm that incorporates the error of approximation of eigenfunction gradients comes from the fact that in several applications, specifically, when eigenfunctions are used to define spectral embeddings for manifold learning, gradients describe first-order geometric information of these embeddings and, in particular, the tangent planes of the embedded manifold; see, e.g., \cite{li2019metrics,koelle2024consistency,BoMeila}. 
	\nc 
	

	\subsection{Main Results}
	\label{sec:MainResultsDiscussion}
	
	In our first main result, we state a lower bound for our problem \eqref{eq:eigenpair} as a function of the number of data points $n$.

	\begin{theorem}\label{thm:lower bound} 
		
		There exists a constant~$c  > 0$, depending on the parameters describing the families~$\MM$ and $\mathcal{P}_{\M,l}$, such that
		\begin{equation}\inf_{\hat{f}_l,\hat{\lambda}_l}\sup_{\M\in\MM,\rho\in\mathcal{P}_{\M,l}}\E_{\X_n\sim\rho}\Biggl[ |\lambda_l - \hat{\lambda}_l | +  \lVert f_l - \hat{f}_l \lfloor_{\M} \rVert_{H^1(\M)}  \Biggr] \ge c\la_{l}(\mathbb{T}^d,\mathbbm{1}) n^{-\frac{2}{d+4}} \nc \,. 
		\label{eq:minimaxTheorem}
		\end{equation}
		In particular, 
		\[\inf_{\hat{f}_l,\hat{\lambda}_l}\sup_{\M\in\MM,\rho\in\mathcal{P}_{\M,l}}\E_{\X_n\sim \rho}\Biggl[ |\lambda_l - \hat{\lambda}_l | +  \lVert f_l - \hat{f}_l \lfloor_{\M} \rVert_{H^1(\M)}  \Biggr] \ge cl^{\frac{2}{d}} n^{-\frac{2}{d+4}}\,. \]
		In the above, $\la_{l}(\mathbb{T}^d,\mathbbm{1})$ denotes the $l$-th eigenvalue of the operator $\Delta_{\rho}$ when $\M$ is the standard flat torus $\mathbb{T}^d$ of dimension $d$ and $\rho$ is the constant density function over this torus.  
		
	\end{theorem}

	We observe that the rate $n^{-\frac{2}{d+4}}$ coincides with the minimax rate for the closely related \textit{density estimation problem} when the error of approximation is measured in a $L^2$ sense and the class of data generating models is $\mathcal{P}_\M$. In precise terms, it follows from well-known existing theory that
	\begin{equation}
	\inf_{\hat{\rho}_n} \sup_{ \M \in \MM, \rho \in \mathcal{P}_\M} \E_{\X_n \sim \rho}\left[ \lVert \hat{\rho}_n - \rho \rVert_{L^2(\M)}  \right] \geq  c n^{-\frac{2}{d+4}}\,,
	\end{equation}
	and it is also well-known that this rate is achieved by \textit{kernel density estimators} $\hat{\rho}_n$ of the form
	\[\hat{\rho}_n(x) \defeq \frac{1}{n} \sum_{i=1}^n K_{r_n}( x_i- x), \text{ for } x\in\M,\]
	for suitable kernels $K_{r_n}$ with bandwidth roughly defined as $r_n \sim  \ n^{- \frac{1}{d+4}} $. As it turns out, these kernel density estimators can actually be turned into minimax optimal estimators for the eigenpairs of the operator $\Delta_\rho$, at least when the manifold $\M$ is known. More precisely, if for the moment we assume to know the manifold $\M$ and we also assume that the unknown density $\rho$ belongs to the class $\mathcal{P}_{\M, l}$, it is possible to use an optimal density estimator $\hat{\rho}_n$ to define the \textit{plug-in} estimator $(\lambda_{\hat{\rho}_n}, f_{\hat{\rho}_n})$,  i.e., the $l$-th eigenpair of the operator
	\begin{equation} 
	\Delta_{\hat{\rho}_n}  f = -\frac{1}{\hat{\rho}_n} \divergence(\hat{\rho}_n^2 \nabla  f  ). 
	\label{eqn:PlugInPDE}
	\end{equation}
	That the estimator $(\lambda_{\hat{\rho}_n}, f_{\hat{\rho}_n})$ reaches the rate $n^{-\frac{2}{d+4}}$ follows from standard computations in the perturbation analysis of elliptic differential operators combined with the approximation error rates for kernel density estimation in the $L^2$-sense. Precisely, one can prove that
	\[ \E \left( \lvert \lambda_l - \lambda_{\hat{\rho}_n} \rvert  + \lVert f_l -  f_{\hat{\rho}_n}\rVert_{H^1(\M)} \right) \leq C_l \E \left( \lVert \rho- \hat{\rho}_n \rVert_{L^2(\M)}\right) \leq C_l n^{-\frac{2}{d+4}} ;\]
	some details of this computation are presented in Appendix \ref{sec:upper bound kde} for the convenience of the reader. Granted that the estimator $(\lambda_{\hat{\rho}_n}, f_{\hat{\rho}_n})$ is far from practical, as it involves solving a differential equation and, more importantly, relies on the assumption that the manifold $\M$ is known ---in order to make sense of the divergence and gradient operators appearing in the definition of the operator $\Delta_{\hat{\rho}_n}$---, it is still useful to consider this estimator because it serves as a benchmark when discussing the performance of other estimators that are more practical and that do not depend on prior knowledge of the manifold $\M$. Moreover, this construction suggests a connection between the density and eigenpair estimation problems and motivates the hypothesis that the lower bound for density estimation is also a lower bound for eigenpair estimation. While this is true \textit{a posteriori}, the lower bound for the eigenpair estimation problem does not actually follow directly from the lower bound for density estimation. This is because it is possible to construct regular densities $\rho, \tilde{\rho}$ over the same manifold (e.g., over the $d$-dimensional flat torus) that, although not close in an $L^2$ sense, are such that their induced operators $\Delta_\rho$ and $\Delta_{\tilde \rho}$ have eigenpairs that are very close to each other; see, e.g., \cite{MR3082248,MR4711957}. We are thus forced to directly attempt to employ information-theoretic tools used in non-parametric estimation problems, but now with new elements at play, given that the objects to be estimated require solving partial differential equations. The details of our analysis are presented in Section \ref{sec:lower bound}. While here we have focused on studying the eigenpair estimation problem under a metric that measures the error of approximation of eigenvalues, eigenfunctions, and eigenfunction gradients, in the future it would be worth studying similar estimation problems under other metrics of interest to practitioners; see Remark \ref{remark:decoupled distance}.
	
	%

	\begin{remark}
		It is worth highlighting that in order to obtain lower bounds for \eqref{eq:eigenpair} it suffices to study lower bounds for the minimax risk that is defined as in $\eqref{eq:eigenpair}$ but where the supremum ranges over the family of data generating distributions $\mathcal{P}_{\mathbb{T}^d,l}$ only. Our proof, which thus focuses on the flat torus setting, can be adapted to obtain similar lower bounds (in terms of $n$) for the minimax risk when the supremum ranges over $\mathcal{P}_{\M,l}$ for a fixed (thus implicitly known) manifold $\M \in \MM$. A complete discussion of these points will be presented in Section \ref{sec:lower bound}. In particular, see Remark \ref{rem:OnExtendingTorus}. 
	\end{remark}

	\medskip

After enunciating our lower bounds in Theorem \ref{thm:lower bound}, we revisit the use of graph Laplacians to approximate eigenpairs of $\Delta_\rho$. In order to prove our strongest approximation results, and due to limitations coming from the regularity theory of elliptic PDEs, we will have to impose slightly stronger assumptions on the model $(\M, \rho)$ than the ones we have adopted up to this point. However, before discussing the additional assumptions on the model $(\M, \rho)$, we first introduce some notation and assumptions on the construction of the graph Laplacian that are used in the sequel. \nc 
	

	We use $\underline{L}^2(\X_n)$ to denote the space of mappings $u: \X_n \rightarrow \R$ and endow this space with the inner product $\langle \cdot , \cdot \rangle_{\underline{L}^2(\X_n)} $ defined according to
	\[ \langle u , v \rangle_{\underline{L}^2(\X_n)} \defeq \frac{1}{n}\sum_{x \in \X_n} u(x) v(x).   \]
	For convenience, in the sequel we work with a rescaled version of the graph Laplacian $\Delta_n$ defined as 
	\begin{equation}
	\mathcal{L}_{\e_n,n } := \frac{2}{\sigma_\eta} \Delta_n, 
	\label{eq:Normalized}
	\end{equation}
	where the constant $\sigma_\eta$ is defined as
	\begin{align} \label{e.sigmaeta}
	\sigma_\eta\defeq \int_{\R^d} |y_1|^2 \eta(|y|) \dd y\,;
	\end{align}
	in the above, $y_1$ represents the first coordinate of $y\in\R^d$. We note that $\sigma_\eta/2$ is the factor needed to appropriately scale the graph Laplacian to guarantee that its spectrum converges to that of $\Delta_\rho$. Throughout the paper, we make the following assumptions on $\eta$ and the connectivity $\e_n$, which together specify the graph Laplacian $\mathcal{L}_{\e_n,n }$.  
	\begin{assumption}
		The function $\eta : [0,\infty) \to [0,\infty)$ used to build the graph Laplacian is assumed to be a non-increasing function with support on the interval $[0, 1]$. We assume that $\eta$'s restriction to $[0, 1]$ is Lipschitz continuous —the assumption that $\eta$ is a non-increasing function can be relaxed by using arguments similar to those in the recent work \cite{trillos2023large}. We further assume that $\eta$ satisfies $\eta(0)=1$, $\eta(1)=0$, $\eta(\frac{1}{2})>0$, and $\|\eta'\|_{L^{\infty}([0,1])}\leq C_\eta$. Finally, without loss of generality, we assume 
		\begin{align} \label{e.normalize}
		\int_{\R^d} \eta(|x|)\dx=1.
		\end{align}	
		\label{assump.Eta}
	\end{assumption}
	\begin{assumption}
		The connectivity parameter $\e_n$ is assumed to satisfy
		\begin{equation}\label{eq:assumption:eps small}
		C\frac{ (\ln n)^{1/d}}{n^{1/d}} < \e_n < \frac{1}{2}\min\{1,i_0,K^{-\frac{1}{2}},R/2\},
		\end{equation}
		where the constant~$C > 0$ on the left hand side is a geometric constant (depending on the parameters determining the families $\MM$ and $\mathcal{P}_\M$ for $\M \in \MM$) to ensure that the weighted proximity graph~$(\X_n,\omega^{\e_n})$ (with $\omega^{\e_n}(x_i, x_j):=\eta(\frac{|x_i - x_j|}{\e_n})$) is connected with probability~$1 - C\e_n^{-d}\exp(-Cn \e_n^d)$;  the geometric quantities $i_0,K,R$ on the right hand side were introduced in Definition \ref{def:ManifoldClass}. Note that the upper bound on $\e_n$ is a standard assumption in the literature of graph Laplacians over proximity graphs; see, e.g., \cite{trillos2019error,calder2019improved}.
		\label{assump:Connectivity}
	\end{assumption}

	The graph Laplacian $\mathcal{L}_{\e_n,n }$ is easily seen to be a positive semi-definite operator with respect to the inner product $\langle \cdot , \cdot \rangle_{\underline{L}^2(\X_n)}$; see, e.g., \cite{von2007tutorial}. We will henceforth list $\mathcal{L}_{\e_n,n }$'s eigenvalues in increasing order as
	\begin{align}
	0= \lambda_{n,1} <  \lambda_{n,2} \le \dots\le \lambda_{n,l}\le \dots \le \lambda_{n,n},
	\end{align}
	and use $\phi_{n,1},\dots,\phi_{n,n}$ to denote a corresponding orthonormal (w.r.t $\langle \cdot , \cdot \rangle_{\underline{L}^2(\X_n)}$) basis of eigenvectors. Thanks to the lower bound in Assumption \ref{assump:Connectivity}, we can indeed assume, without the loss of generality, that $\lambda_{n,2}>0$.


	Having introduced the previous notation and assumptions, we now present two results that quantify the error of approximation of estimators based on eigenpairs of $\mathcal{L}_{\e_n,n }$. For the first result, we measure the difference between the $l$-th eigenvector of $\mathcal{L}_{\e_n,n }$ and the $l$-th eigenfunction of $\Delta_\rho$ in a \textit{discrete} $H^1$-type norm that we introduce next. Given $u: \X_n \to \R$, we define its discrete $\underline{H}^1(\X_n)$ semi-norm as
	\begin{equation}\label{e.def-H_1}
	\|u\|^2_{\underline{H}^1(\X_n)} \defeq \frac1{n^2\veps^{d+2} }\sum_{x\in \X_n}\sum_{y \in \X_n} \eta\left(\frac{|x-y|}{\veps}\right) (u(x) - u(y))^2,
	\end{equation}
	and use the above definition to introduce the scale-invariant error:
	\begin{equation}\label{eq-def:E_l}
	\mathcal{E}_l (\X_n)  \defeq  \frac{|\la_{n,l} - \la_{l}|}{\la_{l}} + \frac{\gamma_l}{\lambda_l} \lVert \phi_{n,l}- f_l \rVert_{\underline{L}^2(\X_n)}
	+\frac{ \gamma_l }{\sqrt{\lambda_l}} \lVert \phi_{n,l} - f_l \rVert_{\underline{H}^1(\X_n)} \,,
	\end{equation}
	where we use $\la_{l}$ and $f_{l}$ \nc to denote the $l$-th eigenvalue and eigenfunction of $\Delta_\rho$ for $\rho$ the distribution used to sample the data $\X_n$, and use $\gamma_l$ for the spectral gap defined in \eqref{eq-def:gamma_l}.

 As mentioned earlier, we must impose some slightly stronger regularity assumptions on the geometry of $\M$ and the density $\rho$ for our next main theorem.

      \begin{assumption}  
      Suppose that $\M \in \MM$ for $\alpha>0$. Assume also that $\rho$ belongs to the class of densities $\mathcal{P}_\M^{2,\alpha}$ defined as the set of densities in $\mathcal{P}_\M$ whose second derivatives are H\"older continuous with H\"older constant less than $c_{2,\alpha}$.       
\label{assump:MoreRegularity}
      \end{assumption}

	\nc

	\begin{theorem}
		\label{t.upperbound}
		Assume $\X_n=\{ x_1, \dots, x_n \}$ are i.i.d. sampled from a distribution $\rho$ satisfying Assumption \ref{assump:MoreRegularity}. There exist constants~$0<c<1, B,C>1$ only depending on the parameters defining $\MM$ and $\mathcal{P}_{\M}^{2,\alpha}$ (potentially including $\alpha$), such that the following holds: If
		\begin{equation} \label{e.howhighup}
		{\veps   \sqrt{\lambda_{l}} < \min\bigl( 1, B \lambda_l^{-\frac{d-1}{2}} \bigr)}\,,
		\end{equation}
		then 
		\begin{equation} \label{e.quench}
		\P \Bigl[ \mathcal{E}_l(\X_n) >  C\log(\e_n^{-1})\veps^2\Bigr] \leqslant Cn\e_n^{-d} \exp \left( -cn \e_n^{d+4}\right)\,.
		\end{equation}
		In particular, when choosing $\e_n\sim \left(\frac{ \log(n)}{ n}\right)^{\frac{1}{d+4}}$, we have
		\begin{equation} \label{e.anneal}
		\E_{\X_n\sim\rho}\Bigl[ \mathcal{E}_l(\X_n) \Bigr] \leqslant  C \left(\frac{1}{n}\right)^{\frac{2}{d+4}}\frac{\log n}{\log\log n}\,.
		\end{equation}
	\end{theorem}

	A few remarks are in order.

    \begin{remark}
   Let us start by discussing the need for the additional regularity assumptions on $\M$ and $\rho$ that are contained in Assumption \ref{assump:MoreRegularity}.  In the preparatory lemmas leading to the proof of Theorem \ref{t.upperbound}, we rely on the fact that we can control the $C^3$-norm of the function $f_l$. The problem with the initial assumptions on $\rho$ and $\M$ is that if we only require the coefficients of the operator $\Delta_\rho$ (see \eqref{eqn:LaplacianCoordinates}) to be in the borderline space $C^2$, then we will not be able to guarantee that $f_l$ is $C^3$. To guarantee the desired regularity of $f_l$, we thus require the coefficients of the operator to belong to $C^{2, \alpha}$ for $\alpha$ \textit{strictly larger than zero} and this is in turn guaranteed under Assumption \ref{assump:MoreRegularity}. For a discussion on the failure of the ``expected" regularity estimates in the borderline case $\alpha=0$ see the counterexamples in \cite[Section 2.2]{FernndezReal2022}, and for the regularity result that we implicitly use in the sequel see \cite[Corollary 2.29]{FernndezReal2022}.    
    \end{remark}
     
	\begin{remark}
	Up to the logarithmic factors in~\eqref{e.anneal},  we can already draw comparisons with the lower bound we obtained in Theorem \ref{thm:lower bound}. First, note that our upper bound with graph Laplacian based estimators holds under slightly stronger assumptions. However, since the result holds for all $\alpha>0$ (although constants in the error bound \eqref{e.anneal} may degenerate as $\alpha $ gets close to zero), morally speaking we can interpret that the rate of convergence in $n$ that graph Laplacians achieve (again, up to logarithmic corrections) does match the lower bound in our first main result. Secondly, note that in Theorem \ref{t.upperbound} we measure the error of approximation in a \textit{discrete} $H^1$-norm. In Theorem \ref{thm:Data dependent construction} below, we will measure the error of approximation in the $H^1(\M)$ norm of a suitable extension of graph Laplacian eigenvectors.

	\end{remark}
	
	\nc 
	%
	%

\begin{remark}[On the optimality of the scalings in Theorem~\ref{t.upperbound}]
		 We remark on the optimality of Theorem~\ref{t.upperbound} and further perspectives. In light of the lower bound in Theorem~\ref{thm:lower bound} and the discussion in the previous remark, the exponent~$\frac{2}{d+4}$ in~\eqref{e.anneal} cannot be improved, and in this sense, Theorem~\ref{t.upperbound} is sharp in the regime~$\e_n \sim \Bigl( \frac{\log(n)}{n} \Bigr)^{\frac{1}{d+4}}$. We expect that the logarithmic terms in~\eqref{e.anneal} can be removed, and this point will be addressed elsewhere. 
        
        \medskip The quantity~$\e_n \sqrt{\lambda_l}$ appearing on the left-hand side of~\eqref{e.howhighup} is scale-invariant. We expect that, by analogy with~\cite{MR3082248,armstrong2023optimal}, convergence rates such as~\eqref{e.quench}-\eqref{e.anneal} hold so long as~$\e_n \sqrt{\lambda_l} < 1. $ 
        
        \medskip It is also natural to wonder what the optimal convergence rates are for intermediate graph connectivities satisfying~$  \bigl( \frac{\log (n)}{n}\bigr)^{\frac{1}d} \ll \e_n \ll \bigl(\frac{\log(n)}{n} \bigr)^{\frac1{d+4}}.$ Such questions could potentially be addressed using some of the tools in~\cite{armstrong2023optimal}, where for very sparse graphs (above percolation) optimal convergence rates (in the sense that they agree with the rates of convergence in periodic homogenization), that are valid for~$\e_n \sqrt{\lambda_l} < 1 $, are obtained. The graphs considered in~\cite{armstrong2023optimal} are built from supercritical continuum percolation cluster with \emph{uniform} intensity that is greater than the critical intensity that guarantees the existence of a unique unbounded connected percolation cluster. This point, which is beyond the scope of this paper, will be addressed elsewhere. 
	\end{remark}

	\medskip

	In our second result, we use Theorem \ref{t.upperbound} in combination with a suitable extension operator to define an estimator that approximates $(\lambda_{l}, f_{l})$ in the original norm from Theorem \ref{thm:lower bound} at a rate $n^{-\frac{2}{d+4}}$ (up to logarithmic corrections and under the slightly stronger regularity requirements in Assumption \ref{assump:MoreRegularity}). To state this result, we follow similar ideas as in~\cite{trillos2019error} and define a kernel $\psi$ closely connected to the kernel $\eta$ in the definition of the graph Laplacian. Precisely, one can define
		$\psi(t)\defeq \int_{t}^\infty \eta(s) s \dd s$ and set
		\begin{equation}
		k_r(x,y) := \frac{1}{r^d} \psi\left(\frac{|x-y|}{r} \right),
		\label{eq:ExtensionKernel}
		\end{equation}
		for $r>0$. Given $u: \X_n \rightarrow \R$, we define its extension 
		\begin{equation}
		\Lambda_ru(x) \defeq  \frac{1}{\theta(x)} \Lambda_r^0 u (x)\,,
		\label{eq:ExtensionLambda}
		\end{equation}
		where
		\[ \Lambda^0_r u(x) \defeq \frac{1}{n}\sum_{j=1}^n u(x_j) k_r(x,x_j), \]
		and $\theta: \M \rightarrow \R$ is a normalization function defined via
		\begin{equation*}
		\theta(x) \defeq \Lambda_r^0 \mathbf{1} (x) = \frac{1}{n} \sum_{j=1}^nk_r(x, x_j)\,.
		\end{equation*}
    We consider the extension of the eigenvector $\phi_{n,l}$:
    \begin{equation}
        \tilde{\phi}_{n, l} := \Lambda_{r/2} \phi_{n,l},
        \label{def:ExtensionEigenvector}
    \end{equation}
    for the choice $r= \e_n$. Importantly, note that $\tilde{\phi}_{n , l}$ is completely data-driven and in particular does not require knowledge of either $\rho$ or $\M$, in contrast with the plug-in estimator $(\lambda_{\hat{\rho}_n}, f_{\hat{\rho}_n})$, which is based on density estimation and requires knowledge of $\M$ to formulate the PDE \eqref{eqn:PlugInPDE}.

	\begin{theorem}\label{thm:Data dependent construction}
	Suppose that the pair $(\M, \rho)$ satisfies Assumption \ref{assump:MoreRegularity}. Let $(\lambda_{n,l}, \phi_{n,l})$ be the~$l$-th eigenpair of the graph Laplacian $\mathcal{L}_{n , \e_n}$ defined in \eqref{eq:Normalized} constructed with samples from $\rho$. Consider the estimator $(\lambda_{n,l}, \tilde{\phi}_{n,l})$ with $\tilde{\phi}_{n,l}$ as defined in \eqref{def:ExtensionEigenvector} for the choice $\e_n\sim \left(\frac{ \log(n)}{ n}\right)^{\frac{1}{d+4}}$. Then 
		\begin{equation} \label{e.mainestimatewithextension}
		\E\left[\left|{\la}_{n, l}-\la_{l}\right|+\gamma_l \left\|\tilde{\phi}_{n,l}-f_{l}\right\|_{H^1(\M)}\right] \leq C\la_{l} n^{-\frac{2}{ d+ 4  }} \frac{\log n}{\log\log n} \,,
		\end{equation}    
		where the spectral gap~$\gamma_l$ is as defined in~\eqref{eq-def:gamma_l}. 
		In particular, if $\rho \in \mathcal{P}_{\M, l}^{2,\alpha} := \mathcal{P}_{\M}^{2,\alpha} \cap   \mathcal{P}_{\M, l}$, then
		\begin{equation*} 
		\E\left[\left|{\la}_{n,l}-\la_{l}\right|+\left\|\tilde{\phi}_{n,l}-f_{l}\right\|_{H^1(\M)}\right] \leq C\la_l  \max\left\{\frac{1}{\gamma_{}},1\right\} n^{-\frac{2}{d+4}}\frac{\log n}{\log\log n} \,.
		\end{equation*}  
	\end{theorem}

\begin{remark}
Theorem \ref{thm:Data dependent construction} implies that, after a suitable extension, graph Laplacian eigenpairs are, omitting logarithmic corrections and making slightly stringer regularity assumptions on the data generating model $(\rho, \M)$, essentially optimal for estimating eigenpairs of $\Delta_\rho$. Note also that the optimal bandwidth choice for both the graph Laplacian and the extension operator $\Lambda_r$ is, roughly speaking, $\e_n \sim n^{-1/(d+4)}$, which matches the optimal bandwidth choice for kernel density estimators. This in particular suggests that, from a statistical perspective, it is not always beneficial to consider sparser proximity graphs to capture the geometry of $\rho, \M$, even if sparser graphs may entail a slight computational advantage. Given that Laplacians over proximity graphs constructed with the Euclidean distance are manifold agnostic, this result also implies that additional knowledge of the manifold $\M$, or the use of more sophisticated constructions to estimate specific geometric features of $\M$, does not necessarily translate into additional approximation power of weighted Laplace-Beltrami eigenpairs.
\end{remark}

	\begin{remark}
		By Weyl's law, we have $\la_l\sim l^{\frac{2}{d}}$. This, in turn, implies $\gamma_l\sim l^{\frac{2-d}{d}}$. By plugging these rates into the lower and upper bound, we obtain
		\begin{equation*} 
		cl^{\frac{2}{d}} n^{-\frac{2}{d+4}} \leq \E\left[\left|{\la}_{n,l}-\la_{l}\right|+\left\|\tilde {\phi}_{n,l}-f_l\right\|_{H^1(\M)}\right] \leq C \max\{l,l^{\frac{2}{d}}\} \nc n^{-\frac{2}{d+4}}\frac{\log n}{\log\log n}.
		\end{equation*}  
	\end{remark}

	\subsection{Key Ideas in the Proofs of Theorems~\ref{t.upperbound} and~\ref{thm:Data dependent construction}.} \label{ss.ideas}
	The overarching theme in the proofs of both Theorems~\ref{t.upperbound} and~\ref{thm:Data dependent construction} is that, above the connectivity regime of random geometric graphs, i.e., when~$\veps > C_d \bigl( \frac{\ln n}{n}\bigr)^{\sfrac1{d}}$ (see Assumption \ref{assump:Connectivity}), the environment of the random geometric graph ``appears Euclidean'' on large scales. The precise way to quantify this is via the graph-based functional inequalities that we present in section \ref{ss.graphpoincare}, in particular in Lemma~\ref{lem:local poincare inequality} and Proposition~\ref{l.msp}. As explained in our literature review section, the idea of using functional inequalities to study the large-scale behavior of harmonic functions on random graphs goes back to~\cite{AD}, which studies a supercritical bond percolation setting. The last named author, together with Armstrong, in~\cite{armstrong2023optimal,AV2} adapted these techniques to study eigenvectors of the graph Laplacian arising from continuum percolation in random media modeled by a uniform Poisson point process. Inspired by some of these ideas, the present paper studies graph Laplacians on random geometric graphs on closed manifolds in a high-connectivity setting, which, after a very carefulanalysis, allows second-order convergence rates; it is important to highlight that our estimates hold for all distributions in the class $\mathcal{P}_{\M}$ for a manifold $\M$ that has sufficiently regular geometry. The following discussion is intended to highlight the fundamental innovations in this paper in relation to existing works in the literature of graph Laplacians.
	
	
	
	The proof of Theorem \ref{t.upperbound} follows from a careful analysis of rates of convergence for solutions to the graph Poisson equation through functional inequalities on graphs. These estimates are proven in section~\ref{ss.poisson}. To illustrate how they are used in the proof of~\ref{t.upperbound}, and to contrast our proof strategy with the approach followed in previous works, we focus on the estimation of eigenvalues and follow \cite{calder2019improved} to write:
	\begin{align*}
	\begin{split}
	\la_{n,l} \langle \phi_{n,l} , f_{l} \rangle_{\underline{L}^2(\X_n)} & = \langle \mathcal{L}_{\e_n,n } \phi_{n,l} , f_{l} \rangle_{\underline{L}^2(\X_n)}  
	\\& = \langle \phi_{n,l} , \mathcal{L}_{\e_n,n } f_l \rangle_{\underline{L}^2(\X_n)}  
	\\& = \langle \phi_{n,l} ,  \Delta_\rho f_l \rangle_{\underline{L}^2(\X_n)} +  \langle \phi_{n,l} , \mathcal{L}_{\e_n,n } f_l - \Delta_\rho f_l \rangle_{\underline{L}^2(\X_n)}
	\\& = \la_{l}   \langle \phi_{n,l} , f_l \rangle_{\underline{L}^2(\X_n)} +  \langle \phi_{n,l} , \mathcal{L}_{\e_n,n } f_l - \Delta_\rho f_l \rangle_{\underline{L}^2(\X_n)}.
	\end{split}
	\end{align*}
	Here we have used the fact that the graph Laplacian $\mathcal{L}_{\e_n,n }$ is self-adjoint with respect to the inner product $\langle \cdot , \cdot\rangle_{\underline{L}^2(\X_n)}$ as well as some elementary algebraic manipulations. In the above, we consider the restrictions of $f_l$ and $\Delta_\rho f_l$ to the data, a valid operation given that both functions are continuous. Rearranging the above expression, we arrive at the identity:
	\begin{equation}
	\la_{n,l} - \la_{l}  = \frac{\langle \phi_{n,l} , \mathcal{L}_{\e_n,n } f_l - \Delta_\rho f_l \rangle_{\underline{L}^2(\X_n)}}{ \langle \phi_{n,l}, f_l \rangle_{\underline{L}^2(\X_n)}  }.
	\label{eq:EigenvalueEstimate}
	\end{equation}
	To analyze the right-hand side of \eqref{eq:EigenvalueEstimate}, we can first rely on a priori approximation estimates to argue that, with very high probability, the denominator can be assumed to be, in absolute value, greater than $1/2$; see Proposition \ref{prop:f_n,f_0 angle} below. From this, it follows that the difference between the eigenvalues is determined by the size of the quantity $\langle \phi_{n,l} , \mathcal{L}_{\e_n,n } f_l - \Delta_\rho f_l \rangle_{\underline{L}^2(\X_n)}$. If we used \textit{pointwise consistency} estimates —which for a fixed smooth function $f: \M \rightarrow \R$ give bounds on the difference between $\mathcal{L}_{\e_n,n } f (x_i)$ and $\Delta_\rho f(x_i)$ over all the $x_i \in \X_n$— to bound this term, we would recover the results from \cite{calder2019improved} and obtain a convergence rate for eigenvalues in the order $\mathcal{O}(n^{-\frac{1}{d+4}})$ after optimizing over the value of $\e_n$.
	
	To make progress and obtain faster rates, we take a step back and use inspiration from PDE theory to give a better control on the term $\langle \phi_{n,l} , \mathcal{L}_{\e_n,n } f_l - \Delta_\rho f_l \rangle_{\underline{L}^2(\X_n)}$. The key idea is to estimate $ \mathcal{L}_{\e_n,n } f_l - \Delta_\rho f_l $ in a much weaker norm than $L^\infty(\X_n)$ (or even $\underline{L}^2(\X_n)$), while  controlling $\phi_{n,l}$ with a stronger norm. To control $\mathcal{L}_{\e_n,n } f_l - \Delta_\rho f_l$, we use the discrete $\underline{H}^{-1}(\X_n)$ semi-norm introduced next.
	\begin{definition}[Discrete $\underline{H}^{-1}(\X_n)$ norm] 
		For any given $h : \X_n \to \R$ we define 
		\begin{equation}\label{e.def:H_-1}
		\|h\|_{\underline{H}^{-1}(\X_n)} \defeq \sup\biggl\{ \langle g , h \rangle_{\underline{L}^2(\X_n)} \bigg\vert g: \X_n \rightarrow \R, \sum_{ x \in \X_n} g(x) = 0, \|g\|_{\underline{H}^1(\X_n)} \leqslant 1 \biggr\}.
		\end{equation}
		This semi-norm can be interpreted as dual (pivoting on the $\underline{L}^2(\X_n)$ inner product) to the $\underline{H}^1(\X_n)$ semi-norm defined in \eqref{e.def-H_1}. 
	\end{definition}

	
	In order to get intuition on what one could potentially gain by the use of the weak semi-norm~$\underline{H}^{-1}(\X_n),$ we consider the simplified setting of~$\M = [0,1]^d,$ and assume~$\X_n$ to be i.i.d samples from the uniform distribution on~$\M$ (i.e., we have~$\rho \equiv 1$). \footnote{The fact that this choice of~$\M$ has boundary, in contrast with the setting of our paper, will be unimportant. Our goal here is simply to motivate the proof without getting bogged down by the (nontrivial) details of working on a general curved manifold and a non-uniform density.} Let~$e \in \Rd, |e| = 1$ denote a fixed direction, and consider the affine function~$\ell_e(x) := x \cdot e. $ With suitable regularity assumptions, this example captures the heart of the matter, since, locally, a sufficiently regular function (such as a continuum eigenfunction of~$\Delta_\rho$) is affine (up to small higher order errors); in the actual proof of our results, this is where Assumption \ref{assump:MoreRegularity} will be needed. It is clear that~$\ell_e$ is harmonic in~$\M$ in the continuum sense (that is,~$ 0=-\Delta \ell_e = \Delta_\rho \ell_e$). By definition,
	\begin{equation*}
	\Delta_n \ell_e(x) = \frac1{n\veps^{d+2}} \sum_{y \in \X_n} \eta\biggl(\frac{|x - y|}{\veps} \biggr) e \cdot (x- y )\,. 
	\end{equation*}
	A bound using an elementary concentration inequality implies that for any~$x \in \X_n \cap \M,$ for any~$t > 0,$
	\begin{equation}
	\P \biggl[ |\Delta_n \ell_e(x) | > t \biggr] \leqslant 2\exp \bigl( - c n \veps^{d+2}t^2\bigr)\,. 
    \label{eq:ConcenBoundell_e}
	\end{equation}
	The crucial point to notice is that the variance of the random variable $\Delta_n \ell_e (x)$ is large and scales like~$\frac1{\veps^2}.$ The proof in~\cite{calder2019improved} then carries out a bias-variance optimization to obtain an expectation estimate of the form 
	\begin{equation*}
	\E_{\X_n\sim\rho} \biggl[ |\Delta_n \ell_e(x)|\biggr] \sim n^{-\frac{1}{d+4}} \quad \mbox{ for all }x\in \M\,,
	\end{equation*}
	by choosing~$\e_n$ to be, roughly, $n^{-\frac1{d+4}}$. 
	
	In contrast, our viewpoint here is to estimate a weak norm of~$\Delta_n \ell_e.$ To this end, for any test function~$g:\M \to \RR$ let us notice, by symmetry, that 
	\begin{equation*}
	\frac{1}{n}\sum_{x \in \X_n} \Delta_n \ell_e(x) g(x) = \frac{1}{n^2\veps^{d+2}} \sum_{i,j=1}^d \eta\biggl( \frac{|x_i - x_j|}{\veps} \biggr) e\cdot ( x_j- x_i) \bigl( g(x_j) - g(x_i) \bigr)\,. 
	\end{equation*}
	This expression resembles a~$U-$statistic and one can show (see Proposition \ref{Prop:ConcenSmoothg}), in fact, that for any function~$g,$\footnote{that vanishes near~$\partial \M$ in the current example.} with~$\int_{\M} |\nabla g|^4 \dx  \leqslant 1,$ and $\lVert g \rVert_{L^\infty(\M)}\leq 1$, it follows that for any~$t > 0,$
	\begin{equation} \label{e.introestimate}
	\P \biggl[ \Biggl| \frac{1}{n}\sum_{i=1}^n \Delta_n \ell_e(x_i) g(x_i) \Biggr| > t  \biggr] \leqslant 2 \exp \bigl( - c   n\veps^d t^2 \nc \bigr) \,, 
	\end{equation}
	where $0<c<1$ only depends on geometric quantities; this  bound is indeed sharper than the concentration bound in \eqref{eq:ConcenBoundell_e}. At this point, we notice from the definition in~\eqref{e.def:H_-1} that, in order to estimate the~$\|\Delta_n\ell_e\|_{\underline{H}^{-1}(\X_n)}$ norm, we would need to take a supremum over test functions~$g.$ We do not do this directly from the previous estimate and, instead, rely on a technical tool, referred to as the \emph{multi-scale Poincaré inequality} (see Proposition~\ref{l.msp}), which allows us to reduce the space of test functions defining the $\|\cdot \|_{\underline{H}^{-1}(\X_n)}$ semi-norm to a finite class of \textit{deterministic} rescaled indicator functions of suitable ``cubes". Additional non-trivial technical work is needed to adapt concentration bounds like \eqref{e.introestimate} to the setting where $g$ may be discontinuous and may have large $L^\infty$ norm. This additional work allows us to get similar concentration bounds that, when combined with the full expression in the multiscale Poincaré inequality, implies tighter concentration bounds for $\lVert \Delta_n \ell_\e \rVert_{\underline{H}^{-1}(\X_n)}$ than what one can get for $\Delta_{n}\ell_\e(x)$. For the reader that is expert in the theory of homogenization, let us remark that the preceding computation demonstrates that when~$n \veps^d \gg 1,$ as is true above the connectivity regime, the above discussion implies that the \emph{homogenization correctors are small,} and so, a homogenization-based averaging reduces to local averaging. 
    
    \medskip 
    {To further elucidate this connection with homogenization, one can think of the graph as a random resister network, with the edge weights of the graph as being conductances (reciprocal of resistance). In this case, capturing the effective or homogenized conductance is arguably the oldest problem in homogenization, studied even by Lord Rayleigh \cite{Rayleigh}. It is then easy to imagine (and is standard to show rigorously) that, in a suitable limit of infinitely many resistors, the {effective conductance} is bounded between the ideal situations when the resistors are connected in series (so that the effective conductivity is the harmonic mean of the individual conductances) on the one extreme, and when the resistors are connected in parallel (so that the effective conductivity is the arithmetic mean of the individual conductances). In our present setting, the rate at which the number of points in a typical ball of size~$\e_n,$ given by~$n\e_n^d$, goes to infinity as~$n \to \infty,$ controls the contrast between conductivities. In the present paper, we consider the limit~$n \e^{d+4} \gg \ln n;$ this then corresponds to a low contrast setting of conductivities, and therefore the arithmetic and harmonic means are very close. It is in this sense that homogenization reduces to local averaging. In contrast, the papers~\cite{armstrong2017optimal,AV2,armstrong2023optimal} deal with the high contrast setting in which~$n \e^d > C$ for a critical constant~$C$, and therefore the effective conductivity is given by a nonlinear, nonlocal average. }

    The above discussion illustrates the improved \textit{variance} estimates for the $\underline H^{-1}(\X_n)$ semi-norm of $\Delta_n \ell_e$, but, as it turns out, the \textit{bias} estimates for $\frac{1}{n} \sum_{i=1}^n g(x_i) \Delta_n \ell_e(x_i) $ also improve substantially when $\rho$ is non-uniform and it is only assumed to have bounded second derivatives (or slightly more regularity). Indeed, for non-uniform $\rho$ that is only assumed to have bounded second derivatives, at best we can hope to get
    \[   \E[ \Delta_n \ell_e(x)] - \frac{\sigma_\eta}{2}\Delta_\rho \ell_e(x)  = \mathcal{O}(\e_n),  \]
whereas for a smooth $g$ as described before one gets
\[     \E[ \frac{1}{n} \sum_{i=1}^n g(x_i) \Delta_n \ell_e(x_i)  ] - \int_\M g(x) \frac{\sigma_\eta}{2}\Delta_\rho \ell_e(x)  \rho(x) \dx = \mathcal{O}(\e_n^2),  \] 
as we will show in Proposition \ref{prop:bias}. It is the combination of the better variance \textit{and} bias estimates for the $\underline{H}^{-1}(\X_n)$ semi-norm of $\Delta_n \overline{u} - \frac{\sigma_\eta}{2}\Delta_\rho \overline{u}$ for a smooth function $\overline{u}$ that ultimately allows us to obtain the rates of convergence in our main theorems.

	\smallskip 
	Returning to the discussion surrounding how to proceed from~\eqref{eq:EigenvalueEstimate}, and observing that with high probability the denominator is at least~$\frac12$, we estimate the numerator using the~$\underline{H}^1(\X_n):\underline{H}^{-1}(\X_n)$ duality discussed above. The~$\underline{H}^1(\X_n)$ norm of~$\phi_{n,l}$, with high probability, is close to that of~$f_l,$ which is~$\sqrt{\lambda_{l}}$. On the other hand, the~$\underline{H}^{-1}(\X_n)$ norm of~$\mathcal{L}_{\e_n,n } f_l - \Delta_\rho f_l$ is estimated by adapting (and making rigurous) the discussion that we presented for the function $\ell_e$. Indeed, by taking a suitable Taylor expansion of the function $f_l$ one can show that the leading order term, the linear one, is the most involved (since the higher order derivative terms come with more powers of~$\veps$).

\medskip

	

	\subsubsection*{Key ideas in the proof of Theorem~\ref{thm:Data dependent construction}} The proof of this theorem proceeds in two steps. In the first, we prove that, with very high probability, for all $u : \X_n \rightarrow \R$  the $H^1(\M)$ semi-norm of $\Lambda_{\e_n/2} u$ is dominated from above by the $\underline{H}^1(\X_n)$ semi-norm. Applying this bound to the choice $u:= \phi_{n, l}  - f_l $, we will be able to control $\lVert \Lambda_{\e_n/2} \phi_{n,l} - \Lambda_{\e_n /2} f_l \rVert_{H^1(\M)}$ using the bounds from Theorem \ref{t.upperbound}. In the second step, which can be thought of as the analysis of the ``bias term" (note, however, that the operator$\Lambda_r$ is random) we obtain probabilistic bounds on the difference between $\nabla \Lambda_{\e_n/2} f_l$ and $\nabla f_l$. The error rates of the bias term are shown to be no worse than those derived in Theorem \ref{t.upperbound}. Combining these two steps Theorem \ref{thm:Data dependent construction} will follow.

	\subsection{Literature Review}\label{sec:literature}


	\textbf{Statistical Minimax Estimation.}
	%
	%
	As briefly discussed earlier, the eigenpair estimation problem studied in this paper is related to density estimation, a problem with a rich literature that has received attention since at least the 1960s. For example, statistical lower bounds for the fixed-point or pointwise density estimation problem were obtained in \cite{farrell1972best,ibragimov2013statistical,stone1980optimal}. Minimax rates for global density estimation using $L^q$ ($1\le q\le \infty$) norms were explored by \cite{Cencov:62SM,cencov2000statistical,khas1979lower,bretagnolle1979estimation}. In particular, assuming that the target density function is in a $\beta$-H\"older class (or Sobolev or Nikol’ski classes), those works show that the optimal density estimation rates are $n^{-\frac{\beta}{2\beta+d}}$ and $(n/ \log n)^{-\frac{\beta}{2\beta+d}}$, the former for the $L^q$ norm with $q <\infty$ and the latter for the $L^\infty$ norm. Lower bounds for estimating the $L^2$-norm of the \textit{gradient} of the density function were established in \cite{birge1995estimation} under the assumption that the target density belongs to a suitable class of distributions; a matching upper bound for the same problem was presented in \cite{bickel1988estimating}. For a more comprehensive discussion on some of these results and other related topics, we refer the reader to \cite{wainwright2019high,tsybakov2009nonparametric}. In this paper, we prove that learning eigenpairs of weighted Laplace-Beltrami operators is as difficult as solving a closely related density estimation problem, at least when we assume that the target model belongs to the class of densities with bounded second derivatives (which corresponds, roughly, to $\beta=2$ in the above results).

	In recent years, many interesting works have studied estimation problems motivated by manifold learning tasks through the lens of minimax statistical theory. Some examples of these tasks and their corresponding papers include dimension estimation \cite{kim2019minimax}, estimation of manifold support \cite{genovese2012minimax,kim2015tight}, maximum subspace estimation \cite{vu2013minimax}, and density estimation on unknown manifolds \cite{divol2022measure} under the Wasserstein loss. We remark that our results can be directly applied to the setting in \cite{divol2022measure} to construct an eigenfunction basis for the target density function. Regarding estimation problems that in practice are usually tackled through the use of graph Laplacians, we would like to highlight the work \cite{green2023minimax}, which establishes the minimax optimality of a method referred to as Principal Components Regression with Laplacian-Eigenmaps used for \textit{supervised} learning (regression) on unknown manifolds. Indeed, by assuming an additive model with Gaussian noise for the observations, as well as certain population level spectral series condition for the underlying regression function, the Principal Components Regression with Laplacian Eigenmaps is shown to achieve the optimal estimation rate of $\mathcal{O}(n^{-\frac{2}{d+4}})$ and shown to induce a goodness-of-fit testing rate of $\mathcal{O}(n^{-\frac{4}{d+4}})$; the paper \cite{PolyLapl} also considers a similar regression problem and obtains close to optimal rates of estimation for graph-based regressors that are solutions to graph-PDEs involving powers of the graph-Laplacian. Unlike the regression problem in \cite{green2023minimax}, in our paper we study an unsupervised learning problem for which, to the best of our knowledge, there is no related literature. We believe that some of the ideas presented in this paper can be used to derive lower bounds for similar estimation problems involving other differential operators such as those arising as scaling limits of other widely used graph-Laplacians (e.g., random walk Laplacians), as well as elliptic operators of interest in physics and other sciences such as Schr\"odinger operators.

	\medskip

	\textbf{Graph Laplacians and Laplace-Beltrami Operators.} To justify consistency, stability, and regularity of graph Laplacian-based data analysis methodologies, multiple works (e.g., \cite{trillos2020bayesian,hoffmann2020consistency,trillos2023large,SpectralNN,neuman2023graph}) have studied the relation between graph Laplacians built from random data points and differential operators over manifolds such as weighted Laplace-Beltrami operators. In particular, for different modes of convergence, and under different norms and assumptions, those works investigate the convergence of discrete graph Laplacians toward weighted variants of continuum Laplace-Beltrami operators in the large data limit. In what follows we provide some overview of the literature exploring this problem.  
	
	Early work on consistency of graph Laplacians focused on \textit{pointwise consistency}. Pointwise consistency results are about the convergence, as $n\to\infty$ and the connectivity parameter $\e_n\to 0$ at a sufficiently slow rate, of the sequence of graph Laplacians applied to a \textit{fixed} smooth function $f$ toward a weighted Laplace Beltrami operator applied to $f$. For data analysis, \textit{spectral consistency} of graph Laplacians is more relevant than pointwise consistency, since, as discussed earlier, many methodologies in machine learning are based on computing eigenvalues and eigenvectors of the graph Laplacian. For this reason, in the past decades, researchers have mainly focused on studying this type of consistency, introducing in the process a variety of analytical and probabilistic tools for carrying out their analysis. 
	
	
	\medskip 
	
	When the data points used to build the graph Laplacian are sampled from a distribution supported on a smooth and compact $d$-dimensional manifold \textit{without} boundary (assumptions that we also make in this paper):
	
	
	

	\begin{itemize}
		\item The authors in~\cite{hein2007graph} obtain a pointwise convergence rate of $\mathcal{O}((\frac{\log n}{n})^{\frac{1}{d+4}})$.
		\item The author in~\cite{singer2006graph} improves upon the previous paper and for the so called random walk Laplacian (a particular normalization of the standard Laplacian) establishes a pointwise convergence rate of $\mathcal{O}((\frac{1}{n})^{\frac{2}{d+6}})$.
		\item The authors in~\cite{BIK} study the spectral consistency of certain (not necessarily random) graph discretizations of Laplace-Beltrami operators.
		\item Using come constructions based on optimal transport, the authors in~\cite{trillos2019error} deduce a spectral convergence rate of $\mathcal{O}((\frac{\log n}{ n})^{\frac{1}{2d}})$ when the graph connectivity parameter $\e_n$ is chosen appropriately. They actually derive spectral convergence rates for all values of connectivity parameter $\e_n$ down to the ``connectivity threshold'' $\left(\frac{\log n}{n} \right)^{1/d}$, at least for $d\geq 3$. 
		\item For higher intensity graphs (i.e., higher connectivity), the authors in \cite{calder2019improved} obtain convergence rates of $\mathcal{O}((\frac{\log n}{n})^{\frac{1}{d+4}}),$ where the rates of convergence for eigenvectors are with respect to an~$L^2$-type norm. These results rely on pointwise convergence rates of graph Laplacians measured in an $L^2$-sense, which, as discussed in Section \ref{sec:MainResultsDiscussion}, leads to suboptimal spectral convergence rates. 
		\item Under the same assumptions as in the previous paper, the authors in~\cite{calder2022lipschitz} obtain similar convergence rates for eigenfunctions in the stronger $L^\infty$-norm and in an even stronger almost $C^{0,1}$-sense. Regularity estimates for graph Laplacian eigenvectors are also obtained. The analysis in that paper relies on regularity estimates, via coupling methods, for solutions of PDEs with a non-local continuum Laplacian that can be thought of as the average of the graph Laplacian. In \cite{SpectralNN}, these results were refined and combined with results from neural network approximation theory to analyze a method called \textit{spectral neural networks} (see, e.g., \cite{haochen2021provable}), a framework for learning from spectral geometric information in data that is based on the suitable training of neural networks.

		\item The authors in \cite{DunsonWuWu} obtain similar spectral convergence rates for eigenfunctions in an $L^\infty$-sense under the additional assumption that the graph weights are constructed using the heat kernel.
		\item For the specific case of data sampled from the uniform distribution over the flat torus, the authors in \cite{WormellReich} study a Sinkhorn-based weighted graph Laplacian and prove a spectral convergence rate of $\mathcal{O}(n^{-\frac{2}{d+8}+o(1)})$ in the $L^\infty$-norm.
		\item In the recent paper~\cite{cheng2022eigen}, the authors obtain a convergence rate for eigenvalues of $(\frac{\log n}{n})^{\frac{2}{d+4}}$ and a rate for eigenvectors, in an $L^2$-sense, of $(\frac{\log n}{n})^{\frac{1}{d+4}}$ for certain choice of $\e_n$. They also obtain an eigenvalue convergence rate of $(\frac{\log n}{n})^{\frac{2}{d+6}}$ and an eigenfunction convergence rate in $L^2$ of $(\frac{\log n}{n})^{\frac{2}{d+6}}$ for $\e_n$ tuned differently. These results hold when data points are sampled from the uniform distribution over a smooth manifold, while slower rates are obtained for more general distributions.


		\item In~\cite{wahl2024kernel}, the authors deduce a rate of convergence of $(\frac{(\log n)^3}{n})^{\frac{2}{d+6}}$ for both eigenvalues and eigenfunctions, the latter in an $L^2$-sense, for the graph Laplacian with weights built using the heat kernel. Their proof relies on a careful perturbation theory analysis. They also assume data to be uniformly distributed.


		\item In \cite{tan2024improved}, the authors deduce the same rate of convergence as in \cite{cheng2022eigen} but for $k$-NN graph Laplacians, another popular graph construction in machine learning different from the $\veps$ proximity graph setting analyzed in our work. On the other hand, the authors of the work \cite{FermatBasedLaplacian} study Laplacians on $\e_n$ proximity graphs relative to data-driven Fermat distances and show that their graph Laplacians converge spectrally to different versions of weighted Laplace-Beltrami operators. In this operators, the way the data density appears facilitates the detection of elongated clusters that are separated by narrow low density regions. 
	\end{itemize}

	There is also a recent line of works that study graph Laplacians, as well as other constructions used in unsupervised learning, built using data sampled from manifolds \textit{with} boundaries, a setting where the presence of boundary layers significantly complicates the analysis. Here we summarize some of the existing literature.
	
	\begin{itemize}
		\item The authors in the paper \cite{trillos2018variational} prove spectral consistency of graph Laplacians toward a suitable weighted Laplace-Beltrami operator with Neumann boundary conditions provided $\e_n$ stays above the connectivity threshold. They do not provide a rate. 
		\item  The authors in \cite{wu2023locally} study the pointwise consistency of local linear embeddings, another popular unsupervised learning technique, when the embeddings are built using data sampled from manifolds with boundary.
		\item \cite{peoples2021spectral} derives a spectral convergence result for the truncated graph Laplacian with \textit{Dirichlet} boundary condition and obtain the rate $\mathcal{O}((\frac{\log n}{n})^{\frac{1}{2d+6}})$ and $\mathcal{O}((\frac{\log n}{n})^{\frac{1}{4d+10}})$ for the convergence of eigenvalues and eigenvectors of these operators, respectively.
	\end{itemize}

	There are other learning tasks of interest that have motivated the analysis of graph Laplacians under different assumptions. One example is 
	\textit{multi-manifold clustering}, where data points are assumed to be sampled from a distribution supported on a union of multiple smooth manifolds that may intersect with each other transversely. In that setting, a natural goal is to cluster the data according to the manifold from which they were sampled. The work \cite{trillos2023large}, for example, uses a path-based algorithm to construct a variant of the graph Laplacian whose eigenapairs reveal, with high probability, the underlying multi-manifold structure of the data. In particular, they prove that the eigenpairs of their graph Laplacian converge toward the eigenpairs of a tensorized Laplace-Beltrami operator at the rate $\mathcal{O}((\frac{\log n}{n})^{\frac{1}{3d-1}})$, where $d$ is the dimension of the manifolds (assuming they all have the same dimension); other results are available in case the manifolds are assumed to have different dimensions. The work \cite{Lu2019GraphAT}, on the other hand, assumes the multi-manifold structure of the data, but focuses on studying the Laplacian of an $\e_n$-proximity graph built using a reflected geodesic distance, a sort of canonical metric in the multi-manifold setting. However, the operators analyzed in that paper do not tensorize in the large data limit and it is thus unclear whether they can be used to identify the underlying multi-manifold structure in the data. Finally, we mention the recent work \cite{BungertPoisson}, which studies continuum limits of solutions to graph Poisson equations (i.e., equations of the form $\Delta_n u = g$) for very degenerate right-hand sides. The motivation for their work comes from semi-supervised learning, where the goal is to propagate, in a principled and non-trivial manner, the information contained in the scarce labeled data to all other available data points. In this paper, we also study graph-Poisson equations but for right-hand sides that are regular. In this smoother setting, obtaining sharp rates of convergence in the $H^1$-sense of solutions to the graph-Poisson equation toward their continuum counterparts is an important step toward the proof of our Theorem \ref{t.upperbound}.

	
	\smallskip 
	As discussed in section~\ref{ss.ideas}, the second half of our paper, which pertains to graph Laplacian based estimators, is influenced by the recent papers~\cite{armstrong2023optimal,AV2}, which introduce tools from quantitative stochastic homogenization to problems in graph-based learning. Briefly: 
	\begin{itemize}
		\item These papers concern the large-scale behavior of solutions to the graph Poisson equation/eigenvalues and eigenfunctions on ``relatively \emph{sparse} graphs''. To be precise, they consider a Poisson point process with intensity above the critical intensity to guarantee the existence of a unique unbounded percolation cluster. This represents a model of a random geometric graph, say on a manifold, but focusing on a localized length-scale that is small enough so that the manifold appears flat, yet large enough so there is room for averaging. 
		\item Using tools from quantitative stochastic homogenization the authors of those papers adapt the results in~\cite{AD} to show that, above a random minimal scale with stretched exponential moments, one has convergence rates of order~$\O((\frac1{n})^{\frac1{d}}),$ \footnote{with an extra~$\sqrt{\log}$ factor in dimensions~$d=2$ that is intrinsic.} for the solutions to graph Poisson equation/eigenfunctions \footnote{associated with Dirichlet boundary conditions on a bounded,~$C^{1,1}$ or convex domain.} in~$\underline{L}^2, \underline{H}^1$ and~$C^{0,1}-$based norms. 
		\item \emph{At the level of sparsity of the graphs considered in~\cite{armstrong2023optimal,AV2}}, the rate that they obtain is optimal, as it matches celebrated convergence rates in the classical case of periodic homogenization~\cite{MR3082248}. We emphasize that here optimality is not meant in the statistical sense of minimaxity that we explore in this paper. 
		
		\item Compared to~\cite{armstrong2023optimal}, our present paper works in a ``high-intensity setting'' and therefore it is possible to (and we do) obtain higher order convergence rates. However, as explained in Section~\ref{ss.ideas}, this is possible because, in these settings, the ``nonlocal, nonlinear'' averaging in homogenization reduces to ``local'' averaging that can be captured by linear concentration inequalities. 
		
	\end{itemize}

	\nc

	


	Finally, we believe that our approach for obtaining convergence rates for problems associated with graph Laplacians is quite general and adapts to other related applications. The overall proof strategy is very robust and not sensitive to the exact details of the graph construction, provided that the graph is connected with high probability (and, as described earlier, even this connectivity assumption can be substantially relaxed to go down to the percolation threshold, provided one is willing to use tools from stochastic homogenization). Particular examples of such generalizations to which we expect our overall approach to be applied (and likely even yield essentially optimal convergence rates) include normalized and random walk graph Laplacians,~$k-$nearest neighbor graphs, etc; see, e.g.,  \cite{von2007tutorial,trillos2019error,calder2022lipschitz}. 
	

	

	\subsection{Outline}
	The rest of the paper is organized as follows. In section \ref{sec:lower bound}, we prove the lower bound for eigenpair estimation that we stated in Theorem \ref{thm:lower bound}. First, we present some background on Fano's method (section \ref{sec:minimax lower bound introduction}) and then proceed to apply this method to our setting (section \ref{sec:proof of lower bounds}). In section \ref{s.upperbound}, we prove the results on graph Laplacian based estimators enunciated in Theorems \ref{t.upperbound} and \ref{thm:Data dependent construction}. First, we revisit some existing approximation results in the literature that we use as a priori bounds in our main proofs. In section \ref{ss.graphpoincare}, we discuss the functional inequalities on graphs that are at the core of our main proofs. In section \ref{sec:ConcentrationBounds}, we present our main probabilistic bound, a concentration inequality for the inner product between a fixed test function $g$ and the difference $\mathcal{L}_{\e_n,n } \overline{u} - \Delta_\rho \overline{u}$ for a sufficiently regular $\overline{u}$. In section \ref{ss.poisson}, we present some results on estimating solutions to Poisson equation. These results prepare the ground for the proofs of our main theorems on graph Laplacian based estimators, which we present in sections \ref{ss.final} and \ref{sec:proof of upperboundcontinuum}.

In Appendix \ref{App:GeoBack}, we provide some brief background on the notions and tools from Riemannian geometry that we use in the remainder of the paper. Appendix \ref{A.proofs} contains the proofs of some technical lemmas. In Appendix \ref{sec:upper bound kde} we discuss how we can use some bounds from perturbation theory to conclude that density estimators can be turned into eigenpair estimators that achieve the lower bound from Theorem \ref{thm:lower bound}, at least when $\M$ is known. Finally, in Appendix \ref{app:Concentration} we collect some standard concentration bounds for sums of i.i.d. random variables and U-statistics.


	\textbf{Additional Notation.} We use $|\mathbf{v}|$ for the Euclidean norm of a vector $\mathbf{v}$ in $\R^d$. For $a,b\in\R^d$ such that $a^i<b^i$ for every $i=1, \dots, d$, we let $[a,b]$ be the set $[a,b]=\{x\in\R^d:a^i\le x^i\le b^i\}$, where $x^i$ represents the $i^\mathrm{th}$ coordinate of $x$. We use $[m]$ to denote the set $\{1,2,\dots,m\}$. $\mathbf{1}_d$ denotes the length-$d$ vector of all ones and $\0_d$ the length-$d$ vector of all zeros. For a set $I$, we use $|I|$ to denote its cardinality, $\mathbf{1}_I$ its indicator function, and use the notation $\vol_n(I)$ for $\frac{1}{n}|I \cap \X_n| $, the normalized discrete volume of a subset of $\X_n$. We use the notation $\avsum$ to denote averages over fixed finite sets. Precisely, 
	\[ \avsum_{y \in A}  f(y) := \frac{1}{|A|} \sum_{y \in A} f(y).\]

	For a given manifold $\M$, we denote by $L^2(\M)$ the space of (equivalence classes of) measurable functions on $\M$ endowed with the $L^2$-inner product with respect to $\M$'s volume form. Likewise, given a density function $\rho: \M \rightarrow \R$, we use $L^2(\rho)$ for the $L^2$ space of functions endowed with the weighted by $\rho$ inner product. With the notation $L^2(\rho)$ we thus obviate mentioning $\M$ when no confusion arises from doing so. For a given measurable subset $A$ of $\M$, we use the notation $\rho(A)= \int_A \rho(x) \dx $.
	
	Given $x,y\in\M$, we will write $|x-y|$ for the Euclidean distance between $x$ and $y$, and write $d_\M(x,y)$ (or simply $d(x,y)$ when no confusion arises from doing so) for their geodesic distance. We denote by $B_\M(x,r)$ the geodesic ball centered at $x\in \M$ with radius $r$, and by $B_r(0) \subseteq T_x \M$ the standard Euclidean ball of radius $r$ centered at the origin of the tangent plane at $x \in \M$, which we denote by $T_x \M$. In section 3, we often abbreviate $\frac{1}{\e_n^d}\eta(\frac{|x-y|}{\eps_n})$ by $\eta_{\e_n}(|x-y|)$.

	Finally, we use $C>1,0<c<1$ to denote constants that only depend on the dimension $d$, the kernel $\eta$ used to construct the graph, and the parameters in the definitions of $\mathcal{P}_{\M,l}$ and $\mathbf{M}$. We use $C_l$ for constants that, in addition, may also depend on $l \in \N$, the target eigenmode.



	\section*{Acknowledgments} 
	NGT was supported by NSF-DMS grant 2236447. C.L. gratefully acknowledges support from the IFDS at UW-Madison through NSF TRIPODS grant 2023239. C.L. thanks R.V. for an invitation to visit NYU during the summer of 2024, and also thanks Jitian Zhao, who provided support for accommodations during that visit. Part of the work was completed during this visit. R.V. thanks Dallas Albritton and Laurel Ohm for an invitation to UW-Madison, where this work was initiated. R.V. thanks Scott Armstrong for helpful conversations. Finally, R.V. acknowledges support from the National Science Foundation through NSF-DMS: 2407592.

	\section{Lower Bounds}\label{sec:lower bound}

	\subsection{Preliminaries}
	\label{sec:minimax lower bound introduction}
	
	In this section, we briefly review how to establish statistical lower bounds for general estimation problems using Fano's method. An introduction to this topic can be found in \cite{wainwright2019high}, where, in addition to Fano's method, other methods for obtaining statistical lower bounds such as LeCam's and Assouad's are discussed.

	Let $\M \in \MM$. For a class of distributions $\mathcal{P}$ over the data space $\M$, we use $\theta$ in this discussion to denote a function (parameter) of interest from the family $\mathcal{P}$ to a finite dimensional space or, more generally, to an arbitrary metric space. After observing i.i.d. samples $\mathcal{X}_n=\{x_1,\dots ,x_n\}$ from an unknown distribution $\rho$ in $\mathcal{P}$, a general statistical task is to estimate the unknown $\theta(\rho)$ from the observations. In the specific setting of this paper, we consider $\mathcal{P}$ to be a class of sufficiently regular distributions over $\M$, and $\theta(\rho)$ is the $l$-th eigenpair $(\lambda_l,f_l) \in \R \times H^1(\M)$ of the weighted Laplace-Beltrami operator $\Delta_\rho$ over $\M$, recalling our clarification on the use of the expression ``the" eigenpair. Recall, also, that we focus on the case $l \geq 2$ since the case $l=1$ is a trivial estimation problem.
	
	For an estimator $\hat{\theta}$ of $\theta$, i.e., a measurable function from $\M^n$ into the codomain of $\theta$, we use a metric $d(\hat{\theta}(\X_n),\theta(\rho))$ to evaluate the quality of the estimation of the true parameter $\theta(\rho)$. The worst case risk associated to an estimator $\hat{\theta}$ relative to a metric $d$ over the codomain of $\theta$ is defined as 
	\[\sup_{\rho \in \mathcal{P}}\E_{\X_n\sim\rho} \left[d(\hat{\theta}(\X_n),\theta(\rho))\right]. \] 
	This worst case risk is used to evaluate the performance of $\hat{\theta}$ over the entire family $\mathcal{P}$ and not just for a single model $\rho$. In our paper, we will take $d$ to be
	\begin{equation} \label{e.metric}
	d(\hat{\theta}(\X_n), \theta(\rho)) \defeq |\hat{\lambda}_l-\lambda_l|+\sqrt{\int_\M (f_l-\hat{f}_l)^2\dx}+\sqrt{\int_\M |\nabla f_l-\nabla \hat{f}_l|^2\dx},
	\end{equation}
	where we think of $\hat{\theta}$ as the pair $\hat{\theta}=(\hat{\lambda}_l, \hat{f}_l) \in \R \times H^1(\M)$. Another metric of potential interest for the estimation problem studied in this paper is
	\begin{equation} \label{e.metric.L2}
	d_{L^2}(\hat{\theta}(\X_n), \theta(\rho)) \defeq |\hat{\lambda}_l-\lambda_l|+\sqrt{\int_\M (f_l-\hat{f}_l)^2\dx},
	\end{equation}
	which omits the error of approximation of eigenfunction gradients. Our focus in this paper will be the metric \eqref{e.metric}.

	The \textit{minimax risk} associated to the estimation of an arbitrary $\theta(\rho)$ over the class $\mathcal{P}$ relative to the metric $d$ takes the form
	\begin{equation*}
	\mathfrak{M}_n(\theta(\mathcal{P});d)\defeq \inf_{\hat{\theta}} \sup_{\rho\in\mathcal{P}} \E_{\X_n\sim\rho}\left[  d(\hat{\theta}(X_n),\theta(\rho)) \right],
	\end{equation*}
	where the infimum ranges over \textit{all} measurable functions from $\M^n$ into the codomain of $\theta$; we abbreviate $\mathfrak{M}_n(\theta(\mathcal{P});d)$ by $\mathfrak{M}_n$ whenever no confusion arises from doing so. Note that $\mathfrak{M}_n$ depends on the number of data points $n$ and it is of interest to characterize how it behaves as $n$ grows. 
	
	As mentioned at the beginning of this section, Fano's method is a systematic approach for obtaining lower bounds for $\mathfrak{M}_n$ by reducing a given estimation problem to a testing or multiclasss classification problem. The idea is as follows. Suppose that $\{\rho_{1},\rho_{2},\dots, \rho_{M}\}\subseteq \mathcal{P}$ is a $2\delta$-separated set in the sense that 
	\begin{equation}
	d(\theta(\rho_j) ,  \theta(\rho_k)) \geq 2 \delta, \quad \forall j \not = k \in [M].
	\label{def:2deltaSeparated}
	\end{equation}
	We consider the pair of random variables $(Z,J)$ where $J $ is uniformly distributed over the set $[M]$ and $Z| J=j \sim \rho_j^n$; here and in what follows we use $\rho^n$ to represent the product measure of a distribution $\rho$ over $\M$ with itself $n$ times. Note that the marginal distribution of $Z$ is the mixture model $\mathbb{Q}_Z \defeq\frac{1}{M} \sum_{j=1}^M \rho^n_{j}$. Having introduced the variables $(Z, J)$, the classification problem of interest is to estimate the unobserved $J$ from the observed $Z$. Intuitively, the difficulty of this problem depends on the mutual information between $Z$ and $J$. That is, the higher the mutual information between $J$ and $Z$, the easier the identification of $J$  from observing $Z$ should be. In what follows we recall the precise definition of mutual information between two random variables and discuss some identities useful to find upper bounds for it.  
	
	Recall that the Kullback-Leibler (KL) divergence between two probability distributions $\mathbb{P}$ and $\tilde{\mathbb{P}}$ defined over the same space is given by
	\begin{align*}
	\KL(\mathbb{P} \| \tilde{\mathbb{P}} )\defeq \int  \log \left(\frac{d \mathbb{P}}{d\tilde{\mathbb{P}}}(y)\right) d\mathbb{P}(y)
	\end{align*}
	for $\mathbb{P}$ an absolutely continuous measure with respect to $\tilde{\mathbb{P}}$. The \textit{mutual information} between the random variables $Z$ and $J$ is defined in terms of $\mathrm{KL}$-divergence as
	\[
	\Information(Z;J)\defeq\KL\left(\mathbb{Q}_{Z, J} \| \mathbb{Q}_Z \mathbb{Q}_J\right),
	\]
	where $\mathbb{Q}_{Z, J}$ represents the joint distribution of $(Z,J)$, and $\mathbb{Q}_J$, $\mathbb{Q}_Z$ are the marginal distributions of $J$ and $Z$, respectively; notice that $\Information(Z;J) \geq 0$ and $\Information(Z;J)=0$ if and only if $Z$ and $J$ are independent. For the random variables $Z$ and $J$ introduced previously, their mutual information can be written as 
	\begin{align*}
	\Information(Z;J)=\frac{1}{M} \sum_{j=1}^M \KL\left(\rho^n_j \| \mathbb{Q}_Z\right),
	\end{align*}
	following the discussion from \cite[Section 15.3.1]{wainwright2019high}. Moreover, thanks to the above formula and a direct computation, the mutual information $\Information(Z;J)$ can be upper bounded by
	\begin{align}\label{eq:Information is bounded by KL divergence}
	\Information(Z;J) \leq \frac{1}{M^2} \sum_{j, k=1}^M \KL\left(\rho^n_j \| \rho^n_k\right),
	\end{align}
	as shown in Equation 15.34 in that same reference. From this formula it follows that if we can provide a uniform upper bound for $\KL\left(\rho^n_j \| \rho^n_k\right)$ over $j\not= k \in [M]$, we will then be able to directly obtain an upper bound for $\Information(Z;J)$. In turn, since
	\begin{align}\label{eq:KL product measure}
	\KL\left(\rho^n_j \| \rho^n_k\right)=n\KL(\rho_j\|\rho_k),
	\end{align}
	it will suffice to upper bound the $\mathrm{KL}$ divergence between elements in our $2\delta$-separated set $\{\rho_1, \dots, \rho_M\} \subseteq \mathcal{P}$. 
	
	In terms of $\delta, M$, and $I(Z;J)$, Fano's method provides a lower bound for $\mathfrak{M}_n$.

	\begin{proposition}[Fano's Method, Proposition 15.12 in \cite{wainwright2019high}]\label{prop:fano} 
		Let $\left\{\rho_1, \ldots, \rho_M\right\}$ be a fixed $2 \delta$-separated subset of $\mathcal{P}$ as defined in \eqref{def:2deltaSeparated}, where $d$ satisfies the axioms of a distance function. Suppose that $J$ is uniformly distributed over the index set $\{1, \ldots, M\}$, and $Z \mid$ $J=j \sim \rho_j^n$. Then the minimax risk $\mathfrak{M}_n(\theta(\mathcal{P}) ; d) $ is lower bounded by
		\begin{equation}
		\label{e.fano}
		\mathfrak{M}_n(\theta(\mathcal{P}) ; d ) \geq \delta\left\{1-\frac{\Information(Z;J)+\log 2}{\log M}\right\},
		\end{equation}
		where $\Information(Z;J)$ is the mutual information between $Z$ and $J$. In particular, thanks to \eqref{eq:Information is bounded by KL divergence} and \eqref{eq:KL product measure}, we have 
		\begin{equation}
		\label{e.fano.2}
		\mathfrak{M}_n(\theta(\mathcal{P}) ; d) \geq \delta\left\{1-\frac{n \max_{j ,k} \mathrm{KL}(\rho_j \|\rho_k) +\log 2}{\log M}\right\}.
		\end{equation}
	\end{proposition}
	

	\subsection{Proof of Theorem \ref{thm:lower bound}}\label{sec:proof of lower bounds}

	We start by observing that the sup in the expression on the right hand side of \eqref{eq:minimaxTheorem} runs over all manifolds $\M \in \MM$. To prove Theorem \ref{thm:lower bound}, it will thus suffice to find a lower bound for 
	\begin{equation}
	\inf_{\hat{f}_l,\hat{\lambda}_l}\sup_{\rho\in\mathcal{P}_{\M,l}}\E_{\X_n\sim\rho}\Biggl[ |\lambda_l - \hat{\lambda}_l | +  \lVert f_l - \hat{f}_l \rVert_{H^1(\M)}  \Biggr]
	\label{eq:minimaxFixedM}
	\end{equation}
	for a \textit{fixed} manifold $\M \in \MM$ that we will choose conveniently. Note that \eqref{eq:minimaxFixedM} is the minimax risk associated to the same eigenpair estimation problem that we have discussed throughout the paper but where we implicitly assume that the manifold $\M$ is known. To obtain a lower bound for \eqref{eq:minimaxFixedM}, we apply Fano's method. We focus on the case where $\M$ is the $d$-dimensional flat torus $\mathbb{T}^d$, since, as we discuss through the proof, this choice simplifies the analysis. We remark that lower bounds for \eqref{eq:minimaxFixedM} when $\M$ is more general can be obtained in a rather analogous way to the torus case under the assumptions discussed in Remark \ref{rem:OnExtendingTorus} below. 
	
	Recall that we use the notation $(\lambda_l,f_l)$ to denote an eigenpair of the weighted Laplace-Beltrami operator defined in~\eqref{eq:eigenpair}, where we assume that $f_l$ is normalized according to
	\begin{equation*}
	\int_{\M}  f_l^2\rho \dx = 1\,. 
	\end{equation*}
	As an eigenpair, $(\lambda_l, f_l)$ solves equation \eqref{eq:eigenform}, which here we rewrite for convenience as
	\begin{equation}
	\label{e.eigenvalue}
	-\mathrm{div} \bigl( \rho^2 \nabla f_l\bigr) = \lambda_l f_l \rho, \quad  \mbox{ in } \mathcal{M}\,.
	\end{equation}
	Naturally, both the eigenvalue $\lambda_l$ and eigenfunction $f_l$ depend on the density~$\rho,$ but we will often suppress this dependence when no confusion arises. Otherwise, we will write $\lambda_l(\rho), f_l(\rho)$. From the discussion in Section \ref{sec:minimax lower bound introduction}, we know that the key ingredient to obtain a lower bound for \eqref{eq:minimaxFixedM} is the construction of a family of ``sufficiently different'' densities~$\{\rho_1,\cdots, \rho_M\} \subseteq \mathcal{P}_{\M,l}$ that have the following properties: for a given~$\delta > 0$ to be chosen later, we want: 
	\begin{itemize}
		\item for any~$i \neq j,$ 
		\begin{equation*}
		\KL(\rho_i || \rho_j) \leqslant C\delta^2\,,
		\end{equation*}
		and 
		\item for any~$i\neq j,$ 
		\begin{equation*}
		|\lambda_l(\rho_i) - \lambda_l(\rho_j) | +  \lVert f_l(\rho_i) - f_l(\rho_j) \rVert_{H^1(\M)} \geqslant 2\delta\,. 
		\end{equation*}
	\end{itemize}
	We will execute this plan by constructing a local packing of distributions in~$\mathcal{P}_{\M, l}$ ``near'' the uniform distribution over~$\mathbb{T}^d$.  For that purpose, we begin by fixing a large enough positive integer $m \in \N$ such that
	\begin{equation} \label{e.choiceofm}
	m \gg C_l\,
	\end{equation}
	for a constant $C_l$ that may depend on $l$ and on $d$. In the torus setting that we will consider through the proof, this constant can be taken to be	\begin{equation}\label{e.choiceofm-Td}
	C_l = C(\sqrt{\lambda_l(\mathbbm{1})}+\frac{1}{\sqrt{\la_l(\mathbbm{1})}}),
	\end{equation}
	where $C$ is independent of $l$. Here and in the remainder of this section we use $\mathbbm{1}$ to denote the density function over $\M= \mathbb{T}^d$ that is identically equal to one. We also use the fact that $f_l(\mathbbm{1})$ satisfies the following regularity estimates:
	\begin{equation}\label{eq:L_infty norm for torus}
	\begin{split}
	\|D^2 f_l(\mathbbm{1})\|_{L_\infty(\M)}\le C\la_l,\\
	\|\nabla f_l(\mathbbm{1})\|_{L_\infty(\M)}\le C\sqrt{\la_l},\\
	\|f_l(\mathbbm{1})\|_{L_\infty(\M)}\le  C\,.
	\end{split}
	\end{equation}


	\begin{remark}
		\label{rem:RegGeneralManifold}
		When~$\M$ is not the torus or the density is not uniform, eigenfunctions are not necessarily uniformly bounded in~$l$ (see ~\cite{sogge2001riemannian}). In general, for a smooth manifold $\M$ and a smooth density $\rho$ with derivatives of all orders, elliptic estimates imply that 
		\begin{equation*} \label{e.derivativebound}
		\lVert D^k f_l\rVert_{L^\infty(\M)}   \leqslant C_{\M, \rho}\la_l^{\frac{d}{2} + \frac{k}{2}}, \quad k \in \mathbb{N}\,,
		\end{equation*}
        for a constant~$C_{\M, \rho}$ that is independent of~$l.$

	\end{remark}

	Next, we partition $\mathbb{T}^d$ into~$m^d$ pairwise disjoint cubes $\M_i$ with volume~$\vol(\M_i) = m^{-d}$ and consider the class of densities over $\mathbb{T}^d$ given by
	\begin{equation}\label{e.family}
	\mathcal{F} \defeq \biggl\{\rho_{\mathbf{c}}=\mathbbm{1}+\frac{1}{m^2}\sum_{i=1}^{m^d} c_i a_i(x) : c_i \in \{\pm 1\} \biggr\}.
	\end{equation}
	Here, the $a_i$ are fixed functions defined according to
	\[a_i \defeq \phi(m(x-b_i)), \quad x \in \mathbb{T}^d,\]  
	for $b_i$ the center of the cube $\M_i$. The function $\phi: \R^d \rightarrow \R$ used in the definition of the $a_i$ is a fixed template function given by  
	\begin{equation}
	\phi(x) = \varphi(|x- u_+|)  - \varphi(|x- u_-|),
	\end{equation}
	where $\varphi:[0,\infty)  \rightarrow [0,\infty)$ is the smooth scalar mollifier 
	\[ \varphi(t) = \begin{cases} C \exp \left( \frac{1}{64t^2 -1} \right) & \text{ if } t < 1/8, \\ 0 & \text{ else}, \end{cases} \]
	and
	\[ u_+:= (1/4, \dots, 1/4), \quad u_- := (-1/4, \dots, -1/4);  \]
	see an illustration of the above construction in Figure \ref{fig:lower bound}. 
	\begin{figure}[htbp]
		\centering
		\centering    	        \includegraphics[width=6cm]{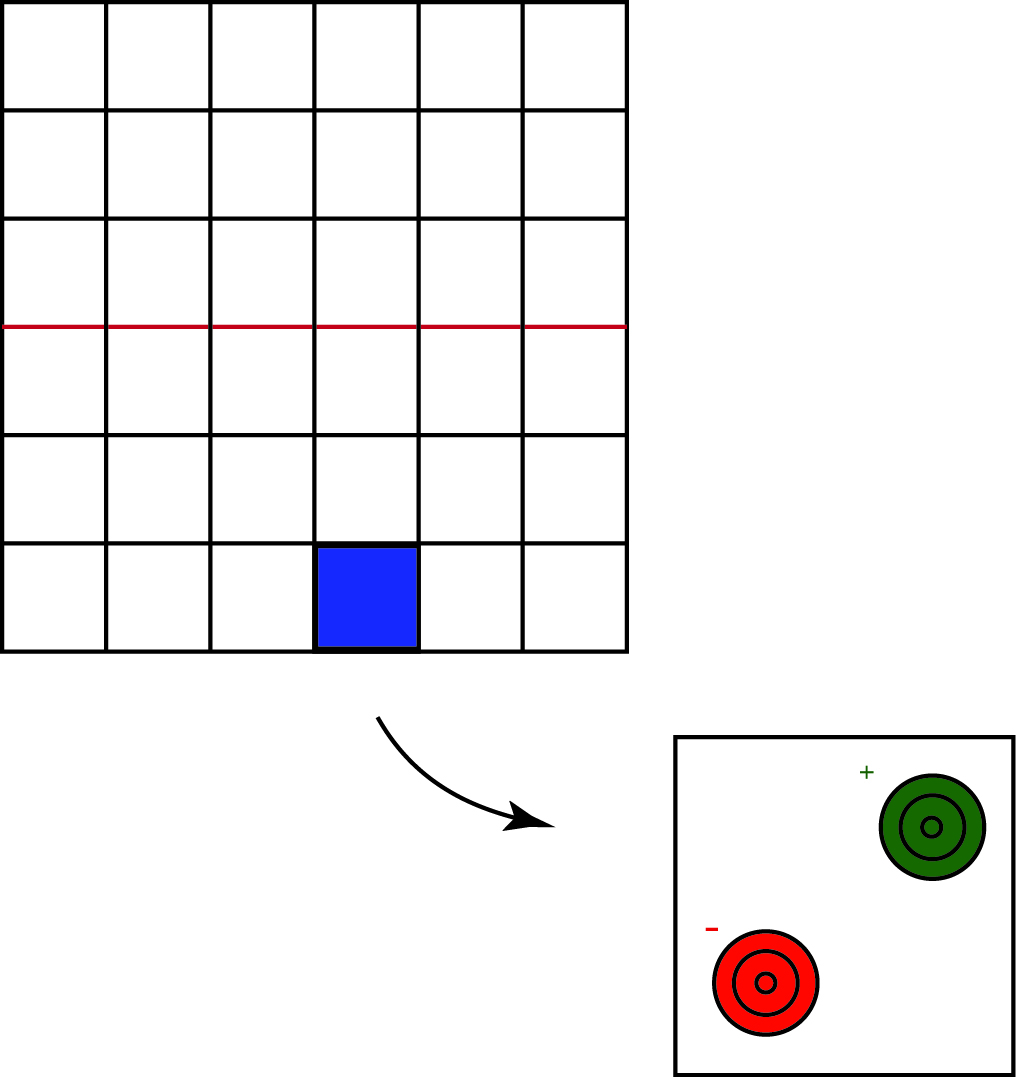}
		\put(5, 0){$\M_j$}
		\put(-55, 122){$\{ x \in \M \text{ s.t. }\nabla f_l(x) = 0 \}$}
		\caption{An illustration of a cube $\M_j$ with $j \in I$. Points in this cube are sufficiently far away from the set of points where the gradient of the eigenfunction $f_l$ vanishes. The function $a_j$ is also illustrated in this figure. The red concentric circles denote the negative level sets of $a_j$, i.e., the points around the point $b_j + \frac{1}{m}u_-$, while the green concentric circles represent the region in $\M_j$ where $a_j$ is positive.}
		\label{fig:lower bound}
	\end{figure}
	
	One can directly deduce the following.
	\begin{proposition}
		The following properties hold:
		\begin{enumerate}
			\item Each of the functions $a_i$ is $C^\infty$ and its support is contained in $\M_i$.
			\item $\lVert a_i \rVert_{L^\infty(\M)} \leq C$, $\lVert \nabla a_i \rVert_{L^\infty(\M)} \leq Cm$, $\lVert D^2 a_i \rVert_{L^\infty (\M)} \leq Cm^2$.
			\item $\int_\M a_i(x) dx =0$.
			\item The family $\F$ introduced in \eqref{e.family} is contained in the family $\mathcal{P}_\M$ from Definition \ref{def:DensityClass}.
		\end{enumerate}
	\end{proposition}

	We will now show that from the family $\F$ we can extract a large enough sub-family that satisfies the conditions needed to apply Fano's method. We need to state and prove a few auxiliary results first.

	\begin{lemma}\label{lem:region B}
		Let $m \in \mathbb{N}$ be fixed as in~\eqref{e.choiceofm-Td}. There exists constants~$C > 0,$ and~$0 < c < 1,$  depending only on~$d$, and an index set $I \subseteq [m^d]$ with cardinality at least $c m^d$, such that for every $i \in I$ there are subcubes $\mathcal{N}_i^+, \mathcal{N}_i^-$ of $\M_i$ satisfying:
		
		\begin{enumerate}
			\item The side length of~$\mathcal{N}_i^\pm$ is at least~$\frac{c}{m}.$
			\item  For every $x\in \mathcal{N}_i^\pm$ we have:
			\[ \pm \nabla a_i(x) \cdot \nabla f_l(\mathbbm{1})(x) \geq c \sqrt{\lambda_l(\mathbbm{1})} \nc m. \]
		\end{enumerate}
	\end{lemma}
	\begin{proof}
		For simplicity, in this proof we abbreviate $f_l(\mathbbm{1})$ as $f_l$ and similarly $\lambda_l(\mathbbm{1})$ as $\lambda_l$. 
		We first observe that the equality $\int_\M |\nabla f_l|^2\dx=\lambda_l$ implies
		\begin{align*}
		\la_l=\int_\M |\nabla f_l|^2\dx & \le \|\nabla f_l(x)\|^2_{L_\infty(\M) }\vol\left(\left\{x \in \M \text{ s.t. }|\nabla f_l(x)|>c_I\right\}\right) 
		\\& \quad + c_I^2 \vol\left(\left\{x\in \M \text{ s.t. }|\nabla f_l(x)|\le c_I\right\}\right)\,,
		\end{align*}
		for any positive constant $c_I$. From this it follows
		\begin{equation}\label{e.condition1}
		\vol\left(\left\{x:|\nabla f_l(x)|\ge  c_I\right\}\right)\ge \frac{\la_l-c_I^2}{1+\|\nabla f_l(x)\|^2_{L_\infty(\M) }}\,.
		\end{equation}    
		Choosing $c_I=\frac{\sqrt{\la_l}}{2}$, we see from \eqref{eq:L_infty norm for torus} and \eqref{e.condition1} that
		\begin{equation}\label{e.condition2}
		\vol\left(\left\{x \in \M \text{ s.t. } |\nabla f_l(x)|\ge  \frac{\sqrt{\la_l}}{2}\right\}\right)\ge c\,.
		\end{equation}

		Let us now consider the set
		\[ I:= \left\{ i \in [m^d] \text{ s.t. } \lVert \nabla f_l(x) \rVert_{L^\infty(\M_i)} > \frac{\sqrt{\lambda_l}}{2} \right\}. \]
		From the fact that
		\[ \bigcup_{i \not \in I} \M_i \subseteq \left\{x \in \M \text{ s.t. }|\nabla f_l(x)|\leq  \frac{\sqrt{\la_l}}{2}\right\}, \]
		it follows from \eqref{e.condition2} that
		\[ \frac{m^d - |I|}{m^d} \leq 1 - \vol\left(\left\{x \in \M \text{ s.t. } |\nabla f_l(x)|>  \frac{\sqrt{\la_l}}{2}\right\}\right) \leq 1-c.  \]
		Rearranging the above inequality, we obtain
		\[ cm^d \leq |I|. \]
		
		Now, let $i \in I$. By definition, there is $x_i\in \M_i$ such that $|\nabla f_l(x_i)| \geq \frac{\sqrt{\lambda_l}}{2}$. For any other $x \in \M_i$, it follows from~\eqref{eq:L_infty norm for torus} that 
		\begin{equation}\label{eq:nabla f_l(x)-nabla f_l(x_i)}
		|\nabla f_l(x)-\nabla f_l(x_i)|\le C \lambda_l|x-x_i|\le \frac{C\lambda_l}{m} \,. 
		\end{equation}    
		Hence, for every $x \in \M_i$ we have
		\begin{align}
		\label{eq:LowerboundNablaf_l}
		\begin{split}
		|\nabla f_l(x)| &\geq |\nabla f_l(x_i)| - |\nabla f_l(x_i) - \nabla f_l(x)|
		\\& \geq \frac{\sqrt{\lambda_l}}{2}- \frac{C \lambda_l}{m}
		\\ & \geq \frac{\sqrt{\lambda_l}}{4},
		\end{split}
		\end{align}
		where the last inequality follows from our assumption that $m$ is sufficiently larger than $\sqrt{\lambda_l}$; see \eqref{e.choiceofm}. To complete the proof, observe that, thanks to the radial symmetry of the function $\varphi(|\cdot|)$ (function used in the definition of the template function $\phi$) and \eqref{eq:nabla f_l(x)-nabla f_l(x_i)}, we can find a subcube $\mathcal{N}_i^+$ of $\M_i$ of side length at least $\frac{c}{m^d}$ such that
		\[ \frac{\nabla a_i (x)}{|\nabla a_i(x)|} \cdot \frac{\nabla f_l(x)}{|\nabla f_l(x)|} \geq 1/4, \quad \text{ and } \quad  |\nabla a_i(x)| \geq c m,\]
		for all $x \in \mathcal{N}_i^+$. Likewise, we can find a subcube $\mathcal{N}_i^-$ of $\M_i$ of side length at least $\frac{c}{m^d}$ such that
		\[ -\frac{\nabla a_i (x)}{|\nabla a_i(x)|} \cdot \frac{\nabla f_l(x)}{|\nabla f_l(x)|} \geq 1/4, \quad \text{ and } \quad |\nabla a_i(x)| \geq c m,\]
		for all $x \in \mathcal{N}_i^-$. From the above and \eqref{eq:LowerboundNablaf_l} we obtain
		\[ \pm \nabla a_i(x) \cdot \nabla f_l(x) \geq c\sqrt{\lambda_l} m, \quad \forall x \in \mathcal{N}_i^{\pm}. \]

	\end{proof}

	\begin{lemma}\label{lem:Laplacian to grad}
		Let $D\defeq [0,1/m]^d $, $g \in C^2(D)$, and suppose that~$u \in L^1(D)$ is such that 
		\[  \alpha \nc  g(x) + \Delta g (x) +\frac{1}{m}u(x) \geq \frac{C_0}{m}, \quad \mbox{ for almost every } x \in D\,,   \]
		for some~$C_0 > 0$ and some  $\alpha>0$\nc . Then 
		\[   \frac{\alpha}{2m} \nc \int_{D}  | g(x) | d x  +   \int_{D} |\nabla g(x)| \dx  \geq \frac{ C_0}{ 2^{d+2}}(1/m)^{d+2}-\frac{1}{2m^2}\int_{D}|u(x)|\dx.  \]
	\end{lemma}
	
	\begin{proof}
		
		Let $D_1 \defeq [1/4m , 3/4m ]^d$. Consider a smooth function $\zeta: D \rightarrow [0,1]$ satisfying:
		\begin{enumerate}
			\item $\zeta(x)= 0$ for all $x$ in $\partial D$.
			\item $ \zeta(x)\geq 1/2 $ for every $x$ in $D_1$.
			\item $|\nabla \zeta (x) |\leq 2m $ for all $x \in D$.
		\end{enumerate}
		
		Then
		\begin{align*}
		\begin{split}
		(1/2)^{d+1} (1/m)^d C_0 \frac{1}{m}  & \leq \int _{ D_1}  \zeta(x) \Bigl(  \alpha \nc g(x) + \Delta g(x) +\frac{1}{m}|u(x)| \Bigr) \dx  
		\\& \leq \int_{D}  \zeta(x) \Bigl(  \alpha \nc g(x) + \Delta g(x) +\frac{1}{m}|u(x)| \Bigr) \dx
		\\ & \leq  \alpha \nc \int_{D}  |g(x)| \dx  +  \int_ D  \zeta(x) \Delta g (x) \dx +  \frac{1}{m}\int_{D}  |u(x)| \dx
		\\& =   \alpha \nc \int_{D}  |g(x)| \dx  -  \int_ D  \nabla \zeta(x) \cdot \nabla g (x) \dx  + \int_{\partial D} \zeta(x) \vec{n}(x) \cdot \nabla g(x) \dx\\
		&\quad  + \frac{1}{m}\int_{D}  |u(x)| \dx
		\\&  =   \alpha \nc \int_{D}  |g(x)| \dx  -  \int_ D  \nabla \zeta(x) \cdot \nabla g (x) \dx + \frac{1}{m}\int_{D}  |u(x)| \dx
		\\&  \leq   \alpha \nc \int_{D}  |g(x)| \dx  +  2 m  \int_ D  |\nabla g(x)| \dx + \frac{1}{m}\int_{D}  |u(x)| \dx
		\\&    \leq   2m\Bigl( \frac{\alpha}{2m} \nc \int_{D}  |g(x)| \dx  +  \int_ D  |\nabla g(x)| \dx\Bigr)+\frac{1}{m}\int_{D}  |u(x)| \dx. 
		\end{split}
		\end{align*}
		After rearranging the above inequality, the result follows. 
	\end{proof}

	Next, we find lower bounds for the difference between the eigenpairs of two density functions $\rho_{\vc^1}, \rho_{\vc^2} \in \mathcal{F}$ for $\vc^1$ sufficiently different from $\vc^2\in \{\pm 1\}^{m^d}$. For our discussion, it will be useful to introduce the following definition.
	\begin{definition} \label{d.sd}
		Given~$\vc^1,\vc^2 \in \{\pm 1\}^{m^d}$, we say that~$\vc^1$ is \emph{sufficiently different} from~$\vc^2$ if there exist at least $\frac{|I|}{4}$ indices $i\in I$ such that $\vc^1_i\not= \vc^2_i$. $I$ is the index set introduced in Lemma \ref{lem:region B}.
	\end{definition}

	Densities built out of sufficiently different~$\vc^1$ and~$\vc^2$ induce eigenpairs that differ substantially. We quantify this precisely in the next lemma, which, in turn, we use to carry out Fano's method; see Proposition~\ref{prop:fano}. For convenience, in what follows we write $\rho_{\vc^1} \defeq \mathbbm{1}+\frac{1}{m^2} a^1$ and $\rho_{\vc^2} \defeq \mathbbm{1}+\frac{1}{m^2} a^2$, where $a^j(x)=\sum_{i=1}^{m^d} \vc^j_i a_i(x)$ for $j=1,2$.
	
	
	\begin{lemma}\label{lem:eigenpair-theta-difference}
		Let $m \in \mathbb{N}$ be fixed as in~\eqref{e.choiceofm-Td}, and consider densities~$\rho_{\vc^1} , \rho_{\vc^2} \in \mathcal{F},$ defined in~\eqref{e.family},  for sufficiently different~$\vc^1, \vc^2$, according to Definition \ref{d.sd}. Then we have either
		\begin{multline}\label{eq:eigenvector difference is at least delta}
		\frac{\lambda_{l}(\mathbbm{1})}{m}   \sqrt{\int_{\M} \left(f_l\left(\mathbbm{1}+\frac{1}{m^2} a^1\right)-f_l\left(\mathbbm{1}+\frac{1}{m^2} a^2\right)\right)^2\dx}\\
		+ \sqrt{\int_{\M} \left|\nabla f_l\left(\mathbbm{1}+\frac{1}{m^2} a^1\right)-\nabla f_l\left(\mathbbm{1}+\frac{1}{m^2} a^2\right)\right|^2\dx}\ge \frac{c \sqrt{\la_l(\mathbbm{1})\nc}}{ m^2}\,,
		\end{multline}
		or 
		\begin{equation}\label{eq:eigenvalue gap is at least delta}
		\left|\lambda_l\left(\mathbbm{1}+\frac{1}{m^2} a^1\right)-\lambda_l\left(\mathbbm{1}+\frac{1}{m^2} a^2\right)\right|\ge \frac{c \sqrt{\la_l(\mathbbm{1})} \nc }{m}\,. 
		\end{equation}
	\end{lemma}
	
	\begin{proof}
		Through the course of the proof, we will use $\la_{l,1},\la_{l,2},f_{l,1},f_{l,2}$ to represent, respectively, $\lambda_l\left(\mathbbm{1}+\frac{1}{m^2} a^1\right)$ ,$\lambda_l\left(\mathbbm{1}+\frac{1}{m^2} a^2\right)$, $f_l\left(\mathbbm{1}+\frac{1}{m^2} a^1\right)$, $f_l\left(\mathbbm{1}+\frac{1}{m^2} a^2\right)$. Writing the PDE~\eqref{e.eigenvalue} for each of the two densities, we obtain
		\begin{align}
		\la_{l,1} f_{l,1} \rho_{\vc^1}  &=-\div(\rho_{\vc^1}^2\nabla f_{l,1}) \,, \label{eq:PDE1} \\
		\la_{l,2} f_{l,2} \rho_{\vc^2} &=-\div(\rho_{\vc^2}^2\nabla f_{l,2}) \,. \label{eq:PDE2}
		\end{align}
		Recalling the definition of the index set $I$ from Lemma \ref{lem:region B}, we focus on those cubes with indices $i\in I$ for which $\vc^1_i\neq \vc^2_i$. Since our estimates below will be obtained cube by cube, we assume without the loss of generality that $\vc^1_i=1$ and $\vc^2_i=-1$. 
		
		First, we subtract \eqref{eq:PDE2} from \eqref{eq:PDE1} to obtain
		\begin{equation}\label{eq:subtract two PDEs}
		\begin{split}
		&\la_{l,1} f_{l,1} \left(1+\frac{1}{m^2} a_i\right)-\la_{l,2} f_{l,2}\left(1-\frac{1}{m^2} a_i\right) \\
		&=\div\left(\left(1-\frac{1}{m^2} a_i\right)^2\nabla f_{l,2}\right)-\div\left(\left(1+\frac{1}{m^2} a_i\right)^2\nabla f_{l,1}\right), \quad \mbox{ on } \mathcal{M}_i\,.
		\end{split}
		\end{equation}  
		Developing each of the terms, we can rewrite \eqref{eq:subtract two PDEs} as
		\begin{equation} \label{e.simplifiedeq1}
		\Delta \bigl( f_{l,1} - f_{l,2} \bigr) + (\lambda_{l,1} - \lambda_{l,2})f_{l,1} + \lambda_{l,2}\bigl( f_{l,1} - f_{l,2} \bigr)  +\frac{1}{m}\mathcal{R} = - \frac{4}{m^2} \nabla a_i \cdot \nabla f_l(\mathbbm{1}) \quad \mbox{ on } \M_i\,,
		\end{equation}
		where the function $\mathcal{R}$ satisfies
		\begin{equation}\label{eq:u integration small}
		\int_{ \M } |\mathcal{R}(x)|\dx \le \frac{C \la_l(\mathbbm{1})}{m}\,.
		\end{equation}
		We will prove~\eqref{e.simplifiedeq1} and \eqref{eq:u integration small} later on, and for now we complete the proof of the lemma assuming that the above estimates have been proved.  
		
		First, let $\mathcal{N}_i$ be $\mathcal{N}_i^-$ from Lemma \ref{lem:region B} so that
		\begin{align}\label{eq:a f_l>Cm}
		-\nabla a_i(x)\cdot \nabla f_l(\mathbbm{1})(x) >c  \sqrt{\lambda_l(\mathbbm{1})} \nc m, \quad \forall x \in \mathcal{N}_i, \quad i \in I.
		\end{align}
		Inserting \eqref{eq:a f_l>Cm} in~\eqref{e.simplifiedeq1}, and using assumptions \eqref{e.choiceofm}, \eqref{e.choiceofm-Td}, we deduce
		\begin{equation} \label{e.simplifiedeq}
		\Delta \bigl( f_{l,1} - f_{l,2} \bigr) + (\lambda_{l,1} - \lambda_{l,2})f_{l,1} + \lambda_{l,2}\bigl( f_{l,1} - f_{l,2} \bigr) + \frac{1}{m}\mathcal{R} >\frac{c \sqrt{\la_l(\mathbbm{1})} \nc} {m} \,, \quad x \in \mathcal{N}_i\,. 
		\end{equation}
		
		We discuss different scenarios.
		\begin{enumerate}
			\item If $|\lambda_{l,1} - \lambda_{l,2}| \ge  c_1\nc \frac{ \sqrt{\la_l(\mathbbm{1})} \nc }{m}$ for some constant $c_1$ (that we choose later on), then the claim \eqref{eq:eigenvalue gap is at least delta} holds and the proof is already complete.
			\item Otherwise, we must have $|\lambda_{l,1} - \lambda_{l,2}|<  c_1\nc \frac{ \sqrt{\la_l(\mathbbm{1})} }{m}$ and in that case we get
			\begin{align*}
			\lambda_{l,2}\bigl( f_{l,1} - f_{l,2} \bigr)+\Delta \bigl( f_{l,1} - f_{l,2} \bigr) + \frac{1}{m}(\mathcal{R}+c_1\sqrt{\lambda_l({\mathbbm{1}})} |f_{l,1}|)\ge \frac{c \sqrt{\la_l(\mathbbm{1})}\nc}{m}, \quad x \in \mathcal{N}_i\,. 
			\end{align*}
			Applying Lemma \ref{lem:Laplacian to grad} with $D=\mathcal{N}_i$, $\alpha = \lambda_{l,2}$, $g=f_{l,1}-f_{l,2}$, and $u = \mathcal{R} + c_1 \sqrt{\la_l(\mathbbm{1})}|f_{l,1}|$, we obtain
			\begin{multline}\label{eq:control for f_l perturbation}
			\frac{\lambda_{l,2}}{m} \nc\int_{\mathcal{N}_i}|f_{l,1}-f_{l,2}|\dx + \int_{\mathcal{N}_i} |\nabla f_{l,1}-\nabla f_{l,2}|\dx\\
			\ge \frac{c  \sqrt{\la_l(\mathbbm{1})}\nc}{m^{d+2}}-\frac{1}{m^2}\int_{\mathcal{N}_i} |\mathcal{R}(x)|\dx-\frac{c_1\sqrt{\lambda_l(\mathbbm{1})}}{m^2}\int_{\mathcal{N}_i} |f_{l,1}(x)|\dx \,.
			\end{multline}
			Given that~$\vc^1$ and~$\vc^2$ are sufficiently different, by Lemma~\ref{lem:region B} the number of~$i \in I$ for which $\vc^1_i\not=\vc^2_i$ is at least $cm^d$. Therefore, summing over such~$i$, and recalling~\eqref{eq:u integration small}, \eqref{e.choiceofm}, and \eqref{e.choiceofm-Td}, we obtain
			\begin{align}\label{eq:norm of lambda_l,f_l,nabla f_l}
			\begin{split}
			\frac{\lambda_{l,2}}{m}  \nc\int_{\M}|f_{l,1}-f_{l,2}|\dx + \int_{\M} |\nabla f_{l,1}-\nabla f_{l,2}|\dx
			& \ge \frac{c \sqrt{\la_l(\mathbbm{1})}\nc}{m^2}- \frac{1}{m^2}\int_\M |\mathcal{R}|\dx 
			\\ &\quad - c_1 \frac{\sqrt{\lambda_l(\mathbbm{1})}}{m^2}\int_{\M}|f_{l,1}(x)|\dx
			\\ & 
			\ge \frac{c  \sqrt{\la_l(\mathbbm{1})}\nc}{ m^2} - 2 c_1 \frac{\sqrt{\la_l(\mathbbm{1})}}{m^2} \lVert f_{l,1} \rVert_{L^2(\rho_{\vc^1})} \,
			\\ & \geq \frac{c  \sqrt{\la_l(\mathbbm{1})}\nc}{ m^2},
			\end{split}
			\end{align}
			where the last line follows from the fact that $ \lVert f_{l,1} \rVert_{L^2(\rho_{\vc^1})} =1$ and by choosing the constant $c_1$ to be sufficiently small.

			%

			%

			Finally, since~$\vol_\M(\M)=1,$  Jensen's inequality implies
			\[   \frac{\lambda_{l}(\mathbbm{1})}{m}  \nc\sqrt{\int_\M|f_{l,1}-f_{l,2}|^2\dx}+\sqrt{\int_\M |\nabla f_{l,1}-\nabla f_{l,2}|^2  \dx}
			\ge \frac{c \sqrt{\la_l(\mathbbm{1})}\nc }{m^2}\,,\]
		\end{enumerate}
		where we have also used \eqref{e.lambdadot} to replace $\la_{l,2}$ with $\la_l(\mathbbm{1})$. This would complete the proof of the lemma, and thus it would remain to justify our claims in~\eqref{e.simplifiedeq1} and \eqref{eq:u integration small}.

		We start by reorganizing \eqref{eq:subtract two PDEs} as follows:
		\begin{equation*}
		\begin{split}
		&(\la_{l,1}-\la_{l,2})f_{l,1}
		+\la_{l,2}(f_{l,1}-f_{l,2})+\frac{1}{m^2} a_i\left(\lambda_{l,1} f_{l,1}+\lambda_{l,2} f_{l,2}\right)\\
		=& \Delta (f_{l,2}-f_{l,1})
		-\frac{2}{m^2} \div(a_i(\nabla f_{l,1}+\nabla f_{l,2}))
		+\frac{1}{m^4}\div (a_i^2(\nabla f_{l,2}-\nabla f_{l,1}))\\
		=&\Delta (f_{l,2}-f_{l,1})
		-\frac{2}{m^2} \nabla a_i \cdot (\nabla f_{l,1}+\nabla f_{l,2})
		-\frac{2}{m^2} a_i(\Delta f_{l,1}+\Delta f_{l,2})\\
		&+\frac{2}{m^4} a_i\nabla a_i \cdot (\nabla f_{l,2}-\nabla f_{l,1})
		+\frac{1}{m^4} a_i^2(\Delta f_{l,2}
		-\Delta f_{l,1})\\
		=&\Delta (f_{l,2}-f_{l,1})
		-\frac{2}{m^2} \nabla a_i \cdot (\nabla f_{l,1}-\nabla f_l(\mathbbm{1})+\nabla f_{l,2}-\nabla f_l(\mathbbm{1}))-\frac{4}{m^2} \nabla a_i \cdot \nabla f_l(\mathbbm{1})\\
		&-\frac{2}{m^2} a_i(\Delta f_{l,1}+\Delta f_{l,2})
		+\frac{2}{m^4} a_i\nabla a_i\cdot (\nabla f_{l,2}-\nabla f_{l,1})
		+\frac{1}{m^4} a_i^2(\Delta f_{l,2}
		-\Delta f_{l,1})\,,
		\end{split}
		\end{equation*}
		where in the last line we have added and subtracted the term $\nabla f_l(\mathbbm{1})$. By rearranging the above equation, we obtain
		\begin{equation}\label{eq:lower bound condition}
		\begin{split}
		&(\la_{l,1}-\la_{l,2})f_{l,1}
		+\la_{l,2}(f_{l,1}-f_{l,2})
		+\frac{1}{m^2} a_i\left(\lambda_{l,1} f_{l,1}+\lambda_{l,2} f_{l,2}\right)
		+\Delta (f_{l,1}-f_{l,2})\\
		&+\frac{2}{m^2} \nabla a_i \cdot (\nabla f_{l,1}-\nabla f_l(\mathbbm{1})+\nabla f_{l,2}-\nabla f_l(\mathbbm{1}))
		+\frac{2}{m^2} a_i(\Delta f_{l,1}+\Delta f_{l,2})\\
		&-\frac{2}{m^4} a_i\nabla a_i \cdot (\nabla f_{l,2}-\nabla f_{l,1})
		-\frac{1}{m^4} a_i^2(\Delta f_{l,2}-\Delta f_{l,1})=-\frac{4}{m^2} \nabla a_i \cdot \nabla f_l(\mathbbm{1})\,.
		\end{split}
		\end{equation}
		Let $\mathcal{R}$ be given by
		\begin{equation*}
		\begin{split}
		\frac{\mathcal{R}}{m}\defeq& \frac{2}{m} \nabla a_i \cdot (\nabla f_{l,1}-\nabla f_l(\mathbbm{1})+\nabla f_{l,2}-\nabla f_l(\mathbbm{1}))\\
		&+ \frac{1}{m^2} a_i\left(\lambda_{l,1} f_{l,1}+\lambda_{l,2} f_{l,2}\right)+\frac{2}{m^2} a_i(\Delta f_{l,1}+\Delta f_{l,2})\\
		&-\frac{2}{m^4} a_i\nabla a_i\cdot (\nabla f_{l,2}-\nabla f_{l,1})
		-\frac{1}{m^4} a_i^2(\Delta f_{l,2}-\Delta f_{l,1})\,.
		\end{split}
		\end{equation*}
		With this definition, \eqref{e.simplifiedeq1} holds, and now we must verify \eqref{eq:u integration small} by bounding each of the above terms. We observe 
		\begin{align*}
		\int_\M \left|\frac{1}{m^2} a_i\left(\lambda_{l,1} f_{l,1}+\lambda_{l,2} f_{l,2}\right)\right|\dx & \le \frac{C \lambda_{l}(\mathbbm{1})}{m^2}\int_\M |f_{l,1}|+|f_{l,2}|\dx
		\\ & \leq \frac{C \lambda_{l}(\mathbbm{1}) \nc}{m^2}(\|f_{l,1}\|_{L^2(\rho_{\vc^1})}+\|f_{l,2}\|_{L^2(\rho_{\vc^2})})= \frac{2C \lambda_{l}(\mathbbm{1})}{m^2}.
		\end{align*} 
		From \eqref{e.eigenvalue}, we have
		\begin{equation*}
		\la_{l,1}f_{l,1}\rho_{\vc^1} = -2\rho_{\vc^1} \nabla\rho_{\vc^1} \cdot\nabla f_{l,1} -\rho_{\vc^1}^2\Delta f_{l,1}.
		\end{equation*}
		Integrating the above equality and using Cauchy-Schwarz inequality, we obtain
		\begin{align*}
		\int_\M |\Delta f_{l,1}|\dx  & \leq C\la_{l,1} \int_\M |f_{l,1}|\dx +C \int_\M |\nabla f_{l,1}|\dx \\ &
		\leq  C\la_{l,1}\|f_{l,1}\|_{L^2(\rho_{\vc^1})} + C\|\nabla f_{l,1}\|_{L^2(\rho_{\vc^1})}\leq C\la_{l,1}.
		\end{align*}
		A similar inequality holds for $\int_\M |\Delta f_{l,2}|\dx .$
		This yields
		\begin{equation*}
		\int_\M \left|\frac{2}{m^2} a_i(\Delta f_{l,1}+\Delta f_{l,2})\right|\dx\leq  \frac{C }{m^2}\int_\M \left|\Delta f_{l,1}\right|+\left|\Delta f_{l,2}\right|\dx \leq  \frac{C (\la_{l,1}+\la_{l,2} )}{m^2} \leq \frac{C \la_l(\mathbbm{1})}{m^2},
		\end{equation*}
		using \eqref{e.lambdadot} in the last step. Similarly, we have
		\begin{equation*}
		\int_\M \left|\frac{1}{m^4} a_i^2(\Delta f_{l,2}-\Delta f_{l,1})\right|\dx \le \frac{C \la_{l}(\mathbbm{1})}{m^4}\,.
		\end{equation*}
		\nc
		From Cauchy–Schwartz inequality and a similar reasoning as above, we have
		\begin{equation*}
		\int_\M \left|\frac{2}{m^4} a_i\nabla a_i \cdot (\nabla f_{l,2}-\nabla f_{l,1})\right|\dx \le \frac{C}{m^4}\|\nabla a_i\|_{L_\infty(\M)} \|\nabla f_{l,2}-\nabla f_{l,1}\|_{L_1(\M)}\le \frac{C \sqrt{\la_{l}(\mathbbm{1})}}{m^3}.
		\end{equation*}
		
		Finally, we notice that $\|\nabla f_{l,1}-\nabla f_l(\mathbbm{1})\|_{L^2(\M)}\le \frac{C\la_l(\mathbbm{1})}{m^2}$ and $\|\nabla f_{l,2}-\nabla f_l(\mathbbm{1}) \|_{L^2(\M)} \le \frac{C\la_l(\mathbbm{1})}{m^2}$ by using \eqref{eq:gradient of eigenfunction is upper bounded by the density perturbation}. Therefore, the term
		\begin{align*}
		\int_\M  \left| \frac{2}{m} \nabla a_i  \cdot (\nabla f_{l,1}-\nabla f_l(\mathbbm{1})+\nabla f_{l,2}-\nabla f_l(\mathbbm{1})) \right| \dx	
		\end{align*}
		is smaller than or equal to
		\begin{align*}
		C\frac{2}{m} \|\nabla a_i\|_{L^2(\M)} &\|\nabla f_{l,1}-\nabla f_l(\mathbbm{1})+\nabla f_{l,2}-\nabla f_l(\mathbbm{1})\|_{L^2(\M)}
		\\ & \le C\left(\|\nabla f_{l,1}-\nabla f_l(\mathbbm{1})\|_{L^2(\M)}
		+\|\nabla f_{l,2}-\nabla f_l(\mathbbm{1})\|_{L^2(\M)}\right)
		\\&\le \frac{C\la_l(\mathbbm{1})}{m^2}\,.
		\end{align*}
		Putting together all the above estimates, we verify \eqref{eq:u integration small} and with it conclude the proof. 
	\end{proof}
	With Lemma \ref{lem:eigenpair-theta-difference} in hand, we are now ready to prove the lower bound in Theorem \ref{thm:lower bound}.

	\begin{proof}[Proof of Theorem \ref{thm:lower bound}]
		The proof is based on Fano's method applied to a sufficiently large subfamily $\hat{\mathcal{F}}$ of the family~$\mathcal{F}$ constructed in~\eqref{e.family}. 
		
		
		
		\textbf{Step 1.}  The first step in the proof is to construct $\hat{\mathcal{F}}$. For that purpose, 		
		let~$\mathbb{H} \defeq \{+1,-1\},$ and consider the binary hypercube~$\mathbb{H}^{|I|},$ where the index set~$I$ was defined in Lemma~\ref{lem:region B}. We turn this set into a metric space by endowing it with a rescaled Hamming distance defined according to 
		\begin{equation*}
		d_H(\vc^\gamma, \vc^\beta)=\frac{1}{|I|} \sum_{i\in I} \mathbbm{1}_{\{\vc^\gamma_{i} \neq \vc^\beta_{i}\}} \,,    
		\end{equation*}
		where the indicator function~$\mathbbm{1}_{\{\vc^\gamma_{i} \neq \vc^\beta_{i}\}} $ is~$1$ if~$\vc^\gamma_i \neq \vc^\beta_i$ and zero otherwise. 
		We recall that the~$\frac14$-packing number of the set~$\mathbb{H}^{|I|}$, with respect to the above metric, denoted by~$M_H(\frac14; \mathbb{H}^{|I|}),$  is defined as the maximal cardinality of a subset of $ \mathbb{H}^{|I|}$ such that for any two of its (distinct) elements $\vc^\gamma, \vc^\beta$ we have $d_H(\vc^\alpha,\vc^\beta)\ge \frac{1}{4}$. It is proved in~\cite[Example 5.3]{wainwright2019high} that this packing number can be bounded from below as
		\begin{equation} \label{e.sizeofM}
		\log M_H\left(\frac{1}{4} ; \mathbb{H}^{|I|}\right) \geq |I| \cdot \KL\left(\frac{1}{4} \| \frac{1}{2}\right) \geq \frac{|I|}{10}\ge cm^d\,,
		\end{equation}
		since, by Lemma~\ref{lem:region B},~$|I| \geqslant c m^d$ for a $c > 0$ depending only on~$d$; in the above, $\KL\left(\frac{1}{4} \| \frac{1}{2}\right)$ denotes the $\KL$ divergence between the Bernoulli distributions with parameters $1/4$ and $1/2$.  
		
		Let~$\mathbf{T} \subseteq\mathbb{H}^{|I|}$ be a maximal $\frac{1}{4}$-packing of~$\mathbb{H}^{|I|}$, which, by~\eqref{e.sizeofM}, is such that $M:= |\mathbf{T}|$ is at least $e^{cm^d}.$ Our desired sub-family of~$\mathcal{F}$ consisting of densities that are sufficiently different is defined in terms of $\mathbf{T}$ as follows 
		\begin{equation*}
		\hat{\mathcal{F}} \defeq	\left\{\rho_{\vc} \text{ s.t. } \vc \in \mathbf{T}\right\} \subseteq \mathcal{F}\,. 
		\end{equation*}
		At this stage, it only remains to complete the following two steps:
		\begin{enumerate}
			\item Compute an upper bound for the KL divergence between distributions in~$\hat{\mathcal{F}}$.
			\item Obtain a lower bound for~$d(\theta(\rho_{\vc^\gamma}),\theta(\rho_{\vc^\beta}))$ for any two~$\gamma \neq \beta$ with~$\rho_{\vc^\gamma},\rho_{\vc^\beta} \in \hat{\mathcal{F}}.$ 
		\end{enumerate}
		\textbf{Step 2: Upper bound for KL divergence of densities in~$\hat{\mathcal{F}}$. }
		For any $\vc^\beta\not= \vc^\gamma$, since~$m$ is large,  the densities $\rho_{\vc^\beta}$ and $\rho_{\vc^\gamma}$ are bounded from below by $\frac{1}{2}.$ It follows that
		\begin{equation*}\label{eq:Hellinger distance control}
		\KL( \rho_{\vc^\beta}\|\rho_{\vc^\gamma})
		= \int_{\M} \!\!\rho_{\vc^\beta} \log \frac{\rho_{\vc^\beta}}{\rho_{\vc^\gamma}} \dx
		\stackrel{\text{Lemma \ref{lem:KL divergence upper bound}}}{\le} \int_\M\frac{\left(\rho_{\vc^\beta}{-}\rho_{\vc^\gamma}\right)^2}{\rho_{\vc^\gamma}} \dx
		\le  2\int_{\M}\left(\rho_{\vc^\beta}{-}\rho_{\vc^\gamma}\right)^2 \dx \,.
		\end{equation*}
		We notice that
		\begin{equation*}\label{eq:hellinger distance pairwise 1}
		\int_{\M}\left(\rho_{\vc^\beta}-\rho_{\vc^\gamma}\right)^2\dx  \leq \frac{4}{m^4}\int_{\M} \phi^2 \dx \leqslant \frac{C}{m^4} \,,
		\end{equation*}
		thanks to the definition of $\phi$. The above two inequalities imply the upper bound
		\[ \KL( \rho_{\vc^\beta}\|\rho_{\vc^\gamma}) \leq \frac{C}{m^4}.\]
		\textbf{Step 3. Lower bound for~$d(\theta(\rho_{\vc^\gamma}),\theta(\rho_{\vc^\beta}))$ for~$\gamma \neq \beta$ with~$\rho_{\vc^\gamma},\rho_{\vc^\beta} \in \hat{\mathcal{F}}$.}
		Indeed, for any $\gamma,\beta\in\mathbf{T}$ such that $\gamma\not=\beta$, we have $ d_H(\vc^\gamma,\vc^\beta)\ge \frac{1}{4}$, and so $\rho_{\vc^\gamma}$ and $\rho_{\vc^\beta}$ are sufficiently different according to Definition \ref{d.sd}. Lemma \ref{lem:eigenpair-theta-difference} and \eqref{e.choiceofm} (with $C_l$ as in \eqref{e.choiceofm-Td}) imply
		\begin{equation} \label{e.MH}
		d(\theta(\rho_{\vc^\gamma}),\theta(\rho_{\vc^\beta}))\ge \frac{c }{m^2}\,. 
		\end{equation}
		\textbf{Step 4. Conclusion.} The lower bound in Proposition~\ref{prop:fano} requires first an upper bound on the mutual information between~$Z$ and~$J$; we recall that $\Information$ was introduced in section \ref{sec:minimax lower bound introduction}, $J$ is a uniform distribution over $[M]$, and $Z$ is the equal-weights mixture distribution with components in $\hat{\F}$. Using \eqref{eq:Information is bounded by KL divergence}, \eqref{eq:KL product measure}, and Step 2, we obtain
		\begin{equation}\label{eq:upper bound of Information}
		\Information(Z;J) \leq \frac{1}{M^2} \sum_{\beta=1}^M \sum_{\gamma=1}^M \KL\left(\rho^n_{\vc^{\beta}} \| \rho^n_{\vc^{\gamma}}\right)\leq C\frac{n}{m^4}\,. 
		\end{equation}

		Inserting~\eqref{e.MH} and~~\eqref{eq:upper bound of Information} in~\eqref{e.fano}, and choosing $\delta \defeq  \frac{c}{m^2}$, we deduce
		\begin{equation*}\label{eq:R lower bound in proof}
		\mathfrak{M}_n( (\lambda_l , f_l) ;   d) 
		\geq \delta\left\{1-\frac{\Information(Z;J)+\log 2}{\log M}\right\}
		\geq \delta\left\{1-\frac{\frac{n}{m^4}+\log 2}{Cm^d}\right\} \geq \frac{c}{m^2}{}\left\{1-\frac{\frac{n}{m^4}+\log 2}{Cm^d}\right\}  \,.
		\end{equation*}
		Choosing
		\begin{equation*}
		 m= \lfloor Cn^{\frac{1}{d+4}} \rfloor
		\end{equation*}
		in the above inequality, we obtain the lower bound $c n^{-\frac{2}{d+4}}$ for \eqref{eq:minimaxFixedM} (with $\M = \mathbb{T}^d$) for all $n$ with $n^{1/(d+4)} \geq C \sqrt{\lambda(\mathbbm{1})}$, and from this the desired lower bound \eqref{eq:minimaxTheorem}.  
	\end{proof}
	
	
	\begin{remark}
		Our approach for obtaining the lower bound in Theorem \ref{thm:lower bound} can be easily adapted to deduce lower bounds for closely related eigenpair estimation problems for different families of weighted Laplace-Beltrami operators. Some examples of interest include the operators discussed in \cite{HOFFMANN2022189}, which can be thought of as scaling limits of different normalizations of graph Laplacians that are used in data clustering and other machine learning applications.     
	\end{remark}
	
	\begin{remark}\label{remark:decoupled distance}
		While here we have focused on deriving lower bounds for the minimax risk for our estimation problem relative to the distance function \eqref{e.metric}, it would be of interest to obtain analogous lower bounds for the metric \eqref{e.metric.L2} (which does not take into account the error of approximation of eigenfunction gradients) or for a metric that only measures the error of approximation of eigenvalues. We believe that the $n$ dependence of the minimax risk for estimating eigenvalues alone may be strictly smaller than the lower bound obtained in this paper, but proving or disproving this claim is left as an open problem that would be worth exploring in the future.

	\end{remark}

	\begin{comment}
	
	\begin{theorem}\label{thm:arbitrary manifold}
	For any~$l \in \mathbb{N}$ and $\M$ is an arbitrary manifold in $\MM$ and $\rho$ is a density supported on $\M$, there exists a constant~$C,N_0>0$ depending on~$d,$ and the parameters describing~$\mathcal{P}$ and $\MM$, such that for all~$n \geqslant N_0,$ we have 
	\begin{equation}\label{eq:minimax lower bound M fixed}
	\inf_{\hat{f}_l,\hat{\lambda}_l}\sup_{\rho\in\mathcal{P}}\E_{\X_n\sim \rho} \sqrt{\int_\M (f_l-\hat{f}_l)^2\dx} +\sqrt{\int_\M |\nabla f_l-\nabla \hat{f}_l|^2\dx}+|\hat{\lambda}_l-\lambda_l| \ge C\la_l(\rho) n^{-\frac{2}{d+4}}\,.
	\end{equation}
	\end{theorem}
	\begin{proof}
	$b_i$ in the proof of Theorem \ref{thm:lower bound} should be selected as a maximal $\frac{1}{m}$-separated set $y_i$, and $\M_i$ should be selected as the corresponding Voronoi set of $y_i$. See the properties of these concepts in \cite{faraco2024homogenization}.
	\end{proof}
	\end{comment}
	
	\begin{remark}
		\label{rem:OnExtendingTorus}
		In the proof of Theorem \ref{thm:lower bound} we obtained lower bounds for the minimax risk \eqref{eq:minimaxFixedM} for $\M$ the $d$-dimensional torus $\mathbb{T}^d$. We focused on the flat torus case to simplify our analysis, but we remark that it is possible to directly analyze \eqref{eq:minimaxFixedM} for more general $\M$ if we make some additional assumptions and adjust some of our constructions. For example, we would need to guarantee that $\M$ is such that the set where $\nabla f_l(\mathbbm{1})$ vanishes is sufficiently regular; by $\mathbbm{1}$ here we mean the uniform measure over $\M$. We would also need to introduce some cumbersome (but completely analogous) differential geometric constructions to be able to generalize the definitions of the template functions $a_i$ and to carry out some computations in a curved manifold setting. Some of the dependence of the lower bounds on $l$ could change due to Remark \ref{rem:RegGeneralManifold}. With these straightforward modifications in mind, our analysis should continue to imply that, for a generic manifold $\M \in \MM$, the minimax risk  \eqref{eq:minimaxFixedM} is lower bounded, up to a constant, by $n^{-\frac{2}{d+4}}$. 
		
	\end{remark}

	\nc
	\section{An Upper Bound Through Graph Laplacians}
	\label{s.upperbound} 
	In this section, we analyze the graph Laplacian based estimator for $(\lambda_l(\rho), f_l(\rho))$ that we discussed in the introduction. In particular, we prove Theorems \ref{t.upperbound} and \ref{thm:Data dependent construction}. The technical core of this section derives estimates for the graph Poisson equation and the relationship its solutions bear to solutions of a related Poisson equation for $\Delta_\rho$. The desired rates of convergence for the eigenvalues and eigenfunctions will be a consequence of these estimates. This section is organized as follows. In subsection~\ref{sec:introduction:Graph Laplacian}, we record some preliminary estimates that are used in the rest of the paper. In subsection~\ref{ss.graphpoincare}, we present some functional inequalities that are at the heart of the rest of this section. In subsection~\ref{sec:ConcentrationBounds}, we present our main concentration bounds. Then, in subsection~\ref{ss.poisson}, we present detailed estimates on solutions to the graph Poisson equation with a smooth right hand side. In subsection~\ref{ss.final}, we present the proof of Theorem~\ref{t.upperbound} and we finish with the proof of Theorem \ref{thm:Data dependent construction} in subsection \ref{sec:proof of upperboundcontinuum}.
	
	\textbf{In what follows, we assume that $\M \in \MM$ (for $\alpha >0$) and $\rho \in \mathcal{P}_{\M}^{2,\alpha}$ have been fixed, according to Assumption \ref{assump:MoreRegularity}.}

	\subsection{A Priori Bounds}\label{sec:introduction:Graph Laplacian}
	We begin by recording some results in the literature on approximation of eigenvalues and eigenfunctions of weighted Laplace-Beltrami operators using graph Laplacians. We use these a priori bounds in the sequel.

	\begin{proposition}\label{prop:f_n,f_0 angle}
		With probability at least $1-C n\exp(-cn\e_n^d)$, we have
		\begin{equation*}
		\| f_l - \phi_{n,l}\|_{\underline{L}^2(\X_n)}\le \frac{1}{2}.
		\end{equation*}
	\end{proposition}
	\begin{proof}
		This is adapted from the main results in \cite{trillos2019error}. See also \cite[Theorem 2.6]{calder2019improved}.
	\end{proof}
	
	\begin{proposition}\label{prop:la_n<2la_0}
		With probability at least $1-Cn\exp(-cn\e_n^d)$, we have
\begin{equation*}
		\frac12\lambda_{l} \leqslant \lambda_{n,l} \leqslant 2\lambda_{l}
		\end{equation*}
        
	\end{proposition}
	\begin{proof}
		This is also adapted from \cite{trillos2019error}.  
	\end{proof}

	

\begin{remark}
In the statement of Proposition \ref{prop:f_n,f_0 angle}, as well as in the statement of the main theorems in this paper, we have implicitly avoided the sign ambiguity of eigenfunctions by declaring $f_l$ to be the eigenfunction that is best aligned with a chosen $\phi_{n,l}$. Throughout the discussion in this section, we assume this sign convention. 
\end{remark}
    

	\subsection{Functional Inequalities on Random Geometric Graphs Above Connectivity Threshold} \label{ss.graphpoincare}
	
	Our first goal will be to derive various functional inequalities that quantify the assertion that when we are above the connectivity threshold (i.e., when the lower bound in~\eqref{eq:assumption:eps small} holds), then, on large scales, the discrete environment ``appears Euclidean". We can think of the estimates derived in this section as preparatory for the proof of Theorem \ref{t.upperbound}.

	We recall that for any~$g: \X_n \to \R$ its discrete $H^1$-semi-norm at length scale $\e_n$ is given by
	\begin{equation}\label{e.def-H_1.inside}
	\|g\|^2_{\underline{H}^1(\X_n)}  =  \frac1{\e_n^{d+2} }\avsum_{x,y\in \X_n} \eta\left(\frac{|x-y|}{\e_n}\right) (g(x) - g(y))^2 =  \sigma_\eta \langle \mathcal{L}_{\e_n, n} g , g\rangle_{\underline{L}^2(\X_n)}\,,
	\end{equation}
	and we also recall that the dual discrete~$\underline{H}^{-1}(\X_n)$ semi-norm is defined according to
	\begin{equation*}
	\|h\|_{\underline{H}^{-1}(\X_n)} := \sup_{g:\X_n\to\R \text{ s.t. } \avsum_{x \in \X_n} g(x) =0, \,\, \|g\|_{\underline{H}^1(\X_n)\leq 1}}\Biggl\{ \avsum_{x \in \X_n} h(x) g(x) \Biggr\}\,,
	\end{equation*}
	for $h: \X_n \to \R$.

	
	A basic inequality that follows from Proposition \ref{prop:la_n<2la_0} and that connects the $\underline{H}^1(\X_n)$ semi-norm with the $\underline{L}^2(\X_n)$ norm is the so called discrete (global) Poincar\'{e} inequality, which we state precisely in the next lemma.
	\begin{lemma}[Discrete Poincaré inequality]\label{l.poincaregraph}
		With probability at least $1-Cn\exp(-cn\e_n^d)$, for every $g:\X_n\to \R$ such that $\avsum_{\X_n} g=0$ we have
		\begin{equation*}
		\|g\|_{\underline{L}^2(\X_n)}^2\leq C \|g\|_{\underline{H}^1(\X_n)}^2.
		\end{equation*}
	\end{lemma}
	\begin{proof}
		By letting $\lambda_{n,2}$ be the second eigenvalue of $\Delta_n$, we have
		\begin{equation*}
		\|g\|_{\underline{L}^2(\X_n)}^2\leq \frac{\sigma_\eta}{\la_{n,2}} \|g\|_{\underline{H}^1(\X_n)}^2.
		\end{equation*}
		Using Proposition \ref{prop:la_n<2la_0}, we may find a deterministic positive lower bound for $\la_{n,2}$ with probability at least $1-Cn\exp(-cn\e_n^d)$. The desired inequality follows.
	\end{proof}

	A series of refinements of the discrete global Poincaré inequality are possible. Of particular relevance is the so-called \emph{multi-scale Poincar\'e inequality} (Proposition \ref{l.msp} below), which is at the core of our proof of Theorem \ref{t.upperbound}.  Thanks to this inequality, we will be able to find probabilistic bounds for the~$\underline{H}^{-1}(\X_n)$ semi-norm of a given (random) function of interest~$h:\X_n \to \RR$. While, \emph{a priori}, this semi-norm is defined as the supremum of inner products of $h$ with infinitely many test functions~$g,$ it is technically convenient to reduce its estimation to the computation of inner products against a suitable finite collection of test functions that capture the geometry of the random graph at all scales above $\e_n$. In order to state this result, we first need to introduce some notation and prove some auxiliary lemmas. 
	
	\medskip
	
	Let~$\nu>0$ be defined as
	\begin{equation}\label{eq:assumption:alpha small}
	4\nu := \min\{1,i_0,K^{-\frac{1}{2}},R/2\},
	\end{equation}
	where we recall $K$, $i_0$, and $R$ are bounds on geometric quantities of $\M$; see Definition \ref{def:ManifoldClass}. For this $\nu$, we let~$\{z_j\}_{j=1}^{N_\nu}$, with $N_\nu \in \mathbb{N}$, denote a maximal~$\nu-$separated net on~$\M $, and let~$\{V_j\}_{j=1}^{N_\nu}$ be the balls of radius $2\nu$ with centers~~$\{z_j\}_{j=1}^{N_\nu}$. By the choice of $\nu$, we know that, for each $j$, the \textit{logarithmic map} $\log_{z_j}: V_j \rightarrow T_{z_j}\mathcal{M}  \cong \R^d $ is a diffeomorphism onto its image. We use $\exp_{z_j}$ to denote the \textit{exponential map}, i.e., the inverse of $\log_{z_j}$, and define
	\begin{equation*}
	U_j  \defeq  \log_{z_j} \Bigl( V_j \Bigr) \subseteq \Rd\,.
	\end{equation*}
	We also consider~$\{\psi_j\}_{j=1}^{N_\nu} \subseteq C^\infty (\M)$ a smooth partition of unity subordinated to the covering~$\{V_j\}_{j=1}^{N_\nu}.$ In Appendix \ref{App:GeoBack} we review the notions of exponential and logarithmic maps.

	For a fixed length scale $\e_n$, we let~$m\in\N$ be such that 
	\begin{equation}
	3^m = \lceil  C \e_n^{-1} \rceil \,,
    \label{eqn:Choicem}
	\end{equation}
     for a constant $C$ chosen later,
	and for~$p \in \{1,\dots,m\}$ we let $\cu_p $ be the cube 
	\begin{equation}\label{eq-def:U}
	\cu_p  \defeq   \biggl[ -\frac{3^p}{2}, \frac{3^p}{2}\biggr)^d\,. 
	\end{equation} 
	Notice that, for any~$p \in \{1,\dots, m\},$ the family of cubes 
	\[ \upsilon + 3^{-m}\cu_p,  \quad \upsilon \in 3^{p-m}\mathbb{Z}^d \] 
	forms a partition for $\R^d$. Therefore, the sets~$U_j$ are contained in some finite union of these cubes. Whenever no confusion arises from doing so, we abbreviate $3^{-m}\cu_1$ (the cubes with side length of order $\e_n$) by $\cu$. For $j =1, \dots, N_\nu$, $1\le p\le m$, and $\upsilon \in 3^{p-m}\mathbb{Z}^d $, we let $  \cu_{p,j,\upsilon}^m$ be the set 
	\begin{equation*}
	\cu_{p,j,\upsilon}^m := \exp_{z_j}(\upsilon+ {3^{-m}}\cu_p),
	\end{equation*}
	for $\upsilon$ such that $\cu_{p,j,\upsilon}^m \cap V_j \not = \emptyset$, which, whenever there is no confusion on the root point $z_j$ and the center point $\upsilon$ being used, will be abbreviated by $\cu_{p}^m$. Note that, for $p=1$, the cube $3^{-m} \cu_p$ has side length of order $\e_n$, while for $p=m$ the cube $3^{-m} \cu_p$ has side length of order one. In particular, the partition of $\R^d$ into cubes when $p=1$ is the finest among all partitions, while the partition for $p=m$ is the coarsest.

	In what follows we present a localized version of the discrete Poincar\'e inequality. Our proof follows a similar structure as in \cite[Lemma 2.3]{armstrong2023optimal} but we use the fact that in our setting we have $\e_n\gg n^{-\frac{1}{d}}(\log n)^{\frac{1}{d}}$, which makes our graph well connected at all scales up to the length scale $\e_n$.

	\begin{lemma}[Discrete local Poincar\'{e}  inequality]
		\label{lem:local poincare inequality}
		For a fixed $j=1, \dots, N_\nu$, with probability at least $1-\e_n^{-d}\exp(-cn\e_n^d)$, there exists a constant $C>1$ such that
		$$
		n \e_n^d \sum_{x \in \cu^m_p \cap \X_n}\left|u(x)-(u)_{\cu^m_p\cap \X_n}\right|^2 \leqslant C 3^{2p}\sum_{y,y^{\prime} \in \cu^m_p \cap \X_n, \: \mathrm{ and } \: y \sim y'}\left|u(y)-u\left(y^{\prime}\right)\right|^2
		$$
		for every function $u:\cu^m_{p} \rightarrow \mathbb{R}$ (where $1\le p\le m$). In the above,
		\[(u)_{\cu^m_p\cap \X_n}  := \frac{1}{| \cu^m_p \cap \X_n| }  \sum_{x \in \cu^m_p \cap \X_n}  u(x),\]
		and we use $y\sim y'$ if $|y-y'|\leq \e_n/2$.

	\end{lemma}
	
	\begin{proof}
		Let $\mathfrak{P}$ be a partition of $ \log_{z_j}(\cu^m_p)  $ into cubes of side length $3^{-m}$ and let $w: \mathfrak{P} \rightarrow \R$ be the function
		given by
		\[  w(\cu)  \defeq  \frac{|  \exp_{z_j}(\cu)\cap \X_n  |}{   n\e_n^d  } , \]
		which, with probability at least $1-\e_n^{-d}\exp(-cn\e_n^d)$, satisfies 
		\begin{equation}
		\label{eq:EventDisPoincare}
		0<  c_1 \leq w(\cu) \leq c_2 , \quad \forall \cu \in \mathfrak{P}
		\end{equation}
		for some constants $c_1$ and $c_2$ that only depend on the constants in the definitions of $\MM$ and $\mathcal{P}_\M$; see Remark 3.8 in \cite{calder2022lipschitz}. In the remainder of this proof we will assume that the event \eqref{eq:EventDisPoincare} holds.
		
		Given a function $u : \cu^m_p \rightarrow \R$, we abuse notation slightly and define:
		\begin{equation}\label{eq-def:(g)}
		u({\cu})  \defeq   \frac{1}{|\exp_{z_j}(\cu)\cap \X_n|}  \sum_{y \in \exp_{z_j}(\cu)\cap \X_n} u(y), \quad  \cu \in \mathfrak{P} .
		\end{equation}
		Note that the average of $u$ over $\cu_p^m \cap \X_n$ can be written as
		\[  (u)_{\cu^m_p \cap \X_n}  = \frac{1}{| \cu^m_p \cap \X_n| }  \sum_{x \in \cu^m_p \cap \X_n}  u(x) =  \frac{1}{\sum_{\cu \in \mathfrak{P} }  w(\cu)   } \sum_{\cu \in \mathfrak{P} } w(\cu)  u(\cu)   =: \overline{u}_{\mathfrak{P},w} .\]
		In other words, the average of the function $u: \cu^m_p \rightarrow \R$ over $\cu^m_p \cap \X_n$ is equal to the weighted (by $w$) average of the function $u : \mathfrak{P} \rightarrow \R$.  
		
		Now, observe that
		\begin{align*}
		\begin{split}
		n \e_n^d \sum_{x \in \cu^m_p\cap \X_n }\left|u(x)-(u)_{\cu^m_p \cap \X_n}\right|^2 & = n\e_n^d  \sum_{ \cu \in \mathfrak{P}} \sum_{x \in \exp_{z_j}(\cu) \cap \X_n}\left|u(x)- \overline{u}_{\mathfrak{P}, w}\right|^2 
		\\& \leq  \mathcal{I} + \mathcal{II},
		\end{split} 
		\end{align*}
		where 
		\[ \mathcal{I}  \defeq  2 n\e_n^d   \sum_{ \cu \in \mathfrak{P}} \sum_{x \in \exp_{z_j}(\cu) \cap \X_n }\left|u(x)-  u (\cu) \right|^2, \quad \mathcal{II}  \defeq   2 n\e_n^d   \sum_{ \cu \in \mathfrak{P}} \sum_{x \in \exp_{z_j}(\cu) \cap \X_n }\left|u(\cu)-  \overline{u}_{\mathfrak{P}, w} \right|^2. \]
		For the first term, Jensen's inequality implies 
		\begin{align*}
		\begin{split}
		\mathcal{I} & = 2n\e_n^d \sum_{ \cu \in \mathfrak{P}} \sum_{x \in \exp_{z_j}(\cu) \cap \X_n}\left|   \frac{1}{|\exp_{z_j}(\cu)\cap\X_n|} \sum_{y \in \exp_{z_j}(\cu)\cap \X_n} (u(x) - u(y))  \right|^2
		\\&  \leq \frac{2}{c_1} \sum_{ \cu \in \mathfrak{P}} \sum_{x \in \exp_{z_j}(\cu) \cap \X_n} \sum _{y \in \exp_{z_j}(\cu)\cap \X_n} \left|u(x)-  u (y) \right|^2
		\\& \leq \frac{2}{c_1} \sum_{x,y \in \cu^m_p \cap \X_n, \: \mathrm{ and } \: x \sim y}\left|u(x)-u(y)\right|^2 ,
		\end{split}
		\end{align*}
		where the last inequality follows from the fact that we can choose the constant $C$ in \eqref{eqn:Choicem} to guarantee that if $\cu$ has side length equal to $3^{-m}$, then $x, y \in \exp_{z_j}(\cu) $ implies $x \sim y$.  
		
		To bound the term $\mathcal{II}$, it is convenient to define a relation on $\mathfrak{P}$ as follows: we write $\cu \sim_{\mathfrak{P}} \cu'$ if the cubes are either the same or share a face. In this way we can view $(\mathfrak{P}, \sim_\mathfrak{P})$ as a grid graph in $\R^d$ and use Lemma \ref{lem:PoincareGrid} (taking $\ell = 3^{p}$) in the Appendix to conclude that 
		\begin{align*}
		\begin{split}
		\mathcal{II}  & =  2 n\e_n^d   \sum_{ \cu \in \mathfrak{P}} \sum_{x \in \exp_{z_j}(\cu) \cap \X_n}\left|u(\cu)-  \overline{u}_{\mathfrak{P}, w} \right|^2
		\\ & \leq 2 c_2(n \e_n^d)^2 \sum_{ \cu \in \mathfrak{P}} | u(\cu )  - \overline{u}_{\mathfrak{P}, w}  |^2
		\\& \leq 2 c_2 C (n \e_n^d)^2  3^{2p} \sum_{\cu, \tilde \cu \in \mathfrak{P}, \: \cu \sim_\mathfrak{P} \tilde \cu  } |  u(\cu) - u(\tilde \cu)  |^2
		\\&  \leq 2 c_2 C (n \e_n^d)^2  3^{2p} \sum_{\cu, \tilde \cu \in \mathfrak{P}, \: \cu \sim_\mathfrak{P} \tilde \cu  } \frac{1}{|\exp_{z_j}(\cu)\cap\X_n||\exp_{z_j}(\tilde \cu)\cap\X_n|}\\
		& \qquad\qquad \qquad\sum_{x \in \exp_{z_j}(\cu)\cap\X_n} \sum_{ y \in \exp_{z_j}(\tilde \cu)\cap\X_n} ( u(x) - u(y))  |^2
		\\& \leq 2 c_2 C (c_1)^{-2}  3^{2p} \sum_{\cu, \tilde \cu \in \mathfrak{P}, \: \cu \sim_\mathfrak{P} \tilde \cu  }    \sum_{x \in \exp_{z_j}(\cu)\cap\X_n} \sum_{ y \in \exp_{z_j}(\tilde \cu)\cap\X_n} ( u(x) - u(y))^2
		\\& \leq 4 c_2 C (c_1)^{-2}  3^{2p}  \sum_{x,y \in \cu^m_p\cap \X_n, \: \mathrm{ and } \: x \sim y}\left|u(x)-u(y)\right|^2,
		\end{split}
		\end{align*}   
		where in the second to last inequality we used Jensen's inequality and in the last one we used the fact that any points $x , y$ with $x \in\exp_{z_j}( \cu)$ and $y \in \exp_{z_j}(\tilde \cu)$, for two cubes with $\cu \sim_{\mathfrak{P}} \tilde \cu$, are within distance $\e_n$ from each other (by choosing $C$ in \eqref{eqn:Choicem} conveniently).
	\end{proof}

	With the discrete local Poincaré inequality in hand, we can now state and prove the multiscale Poincaré inequality announced at the beginning of this section.

	

	\begin{proposition}[Multiscale Poincar\'e inequality] \label{l.msp}
		With probability at least~$1-\e_n^{-d}\exp(-cn\e_n^d)$, for any~$h: \X_n \to \R$ we have 
		\begin{equation*}
		\|h\|_{\underline{H}^{-1}(\X_n)} \leq  C\e_n \|h\|_{\underline{L}^2(\X_n)} 
		+   C\sum_{j=1}^{N_\nu}\sum_{p=1}^{m} 3^{p-m} \biggl( \avsum_{\upsilon \in \Upsilon_{p,j}^m} \biggl| \avsum_{ x \in \cu_{p,j,\upsilon}^m \cap \X_n}  h(x)\biggr|^2 \biggr)^{\sfrac12}\,, 
		\end{equation*}
		where we recall $m$ and $\e_n$ are related as in \eqref{eqn:Choicem} and where we use $\Upsilon_{p,j}^m$ to denote the set of points $\upsilon \in 3^{-m+p} \Zd $ such that $\cu_{p,j,\upsilon}^m \cap V_j \not = \emptyset$. \nc 
	\end{proposition}
	
	\begin{proof}
		Recall that
		\begin{equation*}
		\|h\|_{\underline{H}^{-1}(\X_n)} = \sup \biggl\{ \avsum_{\X_n} h(x) g(x): \avsum_{x \in \X_n} g(x) = 0, \| g\|_{\underline{H}^1(\X_n)}  \leqslant 1\biggr\}\,. 
		\end{equation*}
		Let~$\{\psi_j\}_{j=1}^{N_\nu} \subseteq C^\infty (\M)$ be the smooth partition of unity subordinated to the covering~$\{V_j\}_{j=1}^{N_\nu}$ defined at the beginning of Section \ref{ss.graphpoincare}. For a given test function $g: \X_n \rightarrow \R$ in the definition of the discrete $H^{-1}$ semi-norm, observe that 
		\begin{equation*}
		g=g \sum_j \psi_j=\sum_j  g \psi_j
		\end{equation*}
		and
		\begin{equation*}
		0=\avsum_{\X_n} g=\avsum_{\X_n} g\sum_{j} \psi_j=\sum_{j} \avsum_{\X_n} g\psi_j.
		\end{equation*}
		Thanks to this, we can rewrite $g$ as 
		\begin{align*}
		g=\sum_{j} g\psi_j-\sum_{j}\avsum_{\X_n} g\psi_j=:\sum_{j} g^j.
		\end{align*}
		We notice that for every $j$ the function $g^j= g\psi_j - \avsum_{\X_n}g \psi_j$ satisfies
		\begin{equation*}
		\avsum_{x\in \X_n} g^j(x)=0.
		\end{equation*}
		For every $p \leq m $, define
		\[ g^j_p (x) \defeq  g(\cu_p^m)  \defeq  \frac{1}{| \cu _p^m \cap \X_n|} \sum_{y \in  \cu_p^m \cap\X_n } g^j(y), \quad \text{ if } x \in \cu_p^m\cap \X_n, \quad  \cu_p^m \cap  V_j \not = \emptyset .\]
		For the local Dirichlet energy of $g^j$, we have
		\begin{equation}\label{eq:local dirichlet energy gj}
		\begin{split}
		&\frac{1}{n^2\e_n^{d+2}} \sum_{\cu_1^m \: : \: \cu_1^m \cap V_j \not = \emptyset} \sum_{y, y' \in \cu_1^m\cap \X_n , \: y \sim y'}  | g^j(y) - g^j(y')|^2\\
		&\quad =  \frac{1}{n^2\e_n^{d+2}} \sum_{\cu_1^m \: : \: \cu_1^m \cap V_j \not = \emptyset} \sum_{y, y' \in \cu_1^m\cap \X_n , \: y \sim y'}  | \psi_j(y)g(y) - \psi_j(y') g(y')|^2\\
		&\quad \leq\frac{2}{n^2\e_n^{d+2}} \sum_{\cu_1^m \: : \: \cu_1^m \cap V_j \not = \emptyset} \sum_{y, y' \in \cu_1^m \cap \X_n, \: y \sim y'} \left( | \psi_j(y)(g(y)-g(y')) |^2+| (\psi_j(y)-\psi_j(y'))g(y') |^2\right)
		\\ &\quad  \leq  C (\| g\|_{\underline{H}^1(\X_n)}^2 + \lVert g \rVert^2_{L^2(\X_n)}) \le C,
		\end{split}
		\end{equation}
	 where in the last inequality we used the fact that if $y \sim y'$ it follows that $\eta(\frac{|y-y'|}{\e_n}) \geq \eta(1/2)>0$. We also used the global discrete Poincaré inequality and the fact that $\lVert g \rVert_{\underline{H}^1(\X_n)} \leq 1$.

		Now, 
		\begin{multline}
		\sum_{x \in \X_n } h(x) g(x) = \sum_{x \in \X_n } \sum_{j=1}^{N_\nu}  h(x) g^j(x)  =\sum_{j=1}^{N_\nu}\sum_{x\in\X_n\cap V_j} h(x)( g^j(x) - g^j_1 (x))\\ 
		+ \sum_{j=1}^{N_\nu}\sum_{p=1}^{m-1} \sum_{x\in \X_n\cap V_j} h(x) ( g^j_p(x) - g^j_{p+1}(x))
		+\sum_{j=1}^{N_\nu}  \sum_{x\in\X_n\cap V_j} h(x) g^j_m(x) .
		\label{eqn:AuxMultiPoincare0}
		\end{multline}   
		Let us start by analyzing the term $\sum_{x \in \X_n} h(x)( g^j(x) - g^j_1 (x)) $. Indeed, observe that by a double application of Cauchy-Schwarz inequality we have
		\begin{align*}
		&\sum_{x\in\X_n \cap V_j} h(x)( g^j(x) - g^j_1 (x))  \\
		&\qquad \qquad = \sum _{\cu_1^m \: : \: \cu_1^m \cap V_j \not = \emptyset} \sum_{x \in \cu_1^m \cap \X_n }  h(x)( g^j(x) - g^j_1 (x)) 
		\\&\qquad \qquad \leq \sum_{\cu_1^m \: : \: \cu_1^m \cap V_j \not = \emptyset} \left( \sum_{x \in \cu_1^m  \cap \X_n} h(x) ^2 \right)^{1/2} \left( \sum_{x \in \cu_1^m  \cap \X_n} ( g^j(x) - g^j_1(x)) ^2 \right)^{1/2} 
		\\& \qquad \qquad\leq  \left( \sum_{\cu_1^m \: : \: \cu_1^m \cap V_j \not = \emptyset} \sum_{x \in \cu_1^m  \cap \X_n} h(x) ^2 \right)^{1/2} \left(   \sum_{\cu_1^m \: : \: \cu_1^m \cap V_j \not = \emptyset} \sum_{x \in \cu_1^m  \cap \X_n} ( g^j(x) - g^j_1(x)) ^2 \right)^{1/2} 
		\\& \qquad \qquad\leq \sqrt{n} \lVert  h \rVert_{\underline{L}^2(\X_n\cap V_j)} \left(   \sum_{\cu_1^m \: : \: \cu_1^m \cap V_j \not = \emptyset} \sum_{x \in \cu_1^m  \cap \X_n} ( g^j(x) - g^j(\cu_1^m)) ^2 \right)^{1/2} 
		\\&\qquad \qquad \leq \sqrt{n} \lVert  h \rVert_{\underline{L}^2(\X_n \cap V_j)}  \left( \frac{C 3^2}{n\e_n^d}  \sum_{\cu_1^m \: : \: \cu_1^m \cap V_j \not = \emptyset} \sum_{y, y' \in \cu_1^m \cap \X_n, \: y \sim y'}  | g^j(y) - g^j(y')|^2      \right)^{1/2}
		\\&\qquad \qquad \leq  n  \lVert  h \rVert_{\underline{L}^2(\X_n \cap V_j)}  \left( \e_n^2 \frac{C 3^2}{n^2\e_n^{d+2}}  \sum_{\cu_1^m \: : \: \cu_1^m \cap V_j \not = \emptyset} \sum_{y, y' \in \cu_1^m \cap \X_n, \: y \sim y'}  | g^j(y) - g^j(y')|^2      \right)^{1/2}
		\\&\qquad \qquad \leq C n  \lVert  h \rVert_{\underline{L}^2(\X_n \cap V_j)}  \e_n,
		\end{align*} 
		where in the third inequality we have used the local Poincar\'{e} inequality, i.e., Lemma \ref{lem:local poincare inequality}, when $p=1$; 
		in the last inequality, we have used \eqref{eq:local dirichlet energy gj}. Since $N_\nu$ is a constant, Cauchy-Schwarz inequality implies 
		\begin{eqnarray*}
			\sum_{j=1}^{N_\nu} \lVert  h \rVert_{\underline{L}^2(\X_n \cap V_j)}\leq C\lVert  h \rVert_{\underline{L}^2(\X_n)},
		\end{eqnarray*}
		from where it follows that
		\[\sum_{j=1}^{N_\nu}\sum_{x\in\X_n\cap V_j} h(x)( g^j(x) - g^j_1 (x)) \leq  C n \e_n \|h\|_{\underline{L}^2(\X_n)}. \]
		
		To bound the second term on the right hand side of \eqref{eqn:AuxMultiPoincare0}, observe that for any $p\le m-1$ we have
		\begin{align*}
		&\sum_{x\in\X_n\cap V_j} h(x)( g^j_p(x) - g^j_{p+1} (x))  \\
		& \qquad = \sum _{\cu_p^m \: : \: \cu_p^m \cap V_j \not = \emptyset} \sum_{x \in \cu_p^m\cap \X_n}  h(x)( g^j_p(x) - g^j_{p+1} (x)) 
		\\ &  \qquad =  \sum _{\cu_p^m \: : \: \cu_p^m \cap V_j \not = \emptyset}  |\cu_p^m\cap\X_n|  ( g^j(\cu_p^m) - g^j(\cu_{p+1}^m))   \avsum_{x \in \cu_p^m\cap \X_n} h(x)  
		\\ & \qquad =  \sum _{\cu_p^m \: : \: \cu_p^m \cap V_j \not = \emptyset}  |\cu_p^m\cap\X_n|   \avsum_{x \in \cu_p^m\cap \X_n} h(x)  \avsum_{x \in \cu_p^m\cap \X_n}( g^j(x) - g^j(\cu_{p+1}^m))
		\\& \qquad = \sum _{\cu_{p+1}^m \: : \: \cu_{p+1}^m \cap V_j \not = \emptyset}  \sum_{\cu_p^m \subseteq \cu_{p+1}^m} |\cu_p^m\cap\X_n|   \avsum_{x \in \cu_p^m\cap \X_n} h(x)  \avsum_{x \in \cu_p^m\cap \X_n}( g^j(x) - g^j(\cu_{p+1}^m)) 
		\\& \qquad \leq  \sum _{\cu_{p+1}^m \: : \: \cu_{p+1}^m \cap V_j \not = \emptyset}  \sum_{\cu_p^m \subseteq \cu_{p+1}^m} |\cu_p^m\cap\X_n| \lvert  \avsum_{x \in \cu_p^m\cap \X_n} h(x) \rvert \left(   \avsum_{x \in \cu_p^m\cap \X_n}( g^j(x) - g^j(\cu_{p+1}^m))^2 \right)^{1/2}
		\\& \qquad \leq    Cn 3^{(p-m)d}  \sum _{\cu_{p+1}^m \: : \: \cu_{p+1}^m \cap V_j \not = \emptyset}  \left(  \sum_{\cu_p^m \subseteq \cu_{p+1}^m}  \lvert \avsum_{x \in \cu_p^m\cap \X_n} h(x) \rvert^2 \right)^{1/2}\\
		& \qquad\qquad\qquad\qquad\qquad\qquad \times\left(  \sum_{\cu_p^m \subseteq \cu_{p+1}^m}    \avsum_{x \in \cu_p^m\cap \X_n}( g^j(x) - g^j(\cu_{p+1}^m))^2 \right)^{1/2} 
		\\& \qquad \leq   C\sqrt{n 3^{(p-m)d}}  \sum _{\cu_{p+1}^m \: : \: \cu_{p+1}^m \cap V_j \not = \emptyset}  \left(  \sum_{\cu_p^m \subseteq \cu_{p+1}^m}  \lvert \avsum_{x \in \cu_p^m\cap \X_n} h(x) \rvert^2 \right)^{1/2}\\
		& \qquad\qquad\qquad\qquad\qquad\qquad \times\left(  \sum_{\cu_p^m \subseteq \cu_{p+1}^m}    \sum_{x \in \cu_p^m\cap \X_n}( g^j(x) - g^j(\cu_{p+1}^m))^2 \right)^{1/2} 
		\\& \qquad =  C\sqrt{n 3^{(p-m)d}}  \sum _{\cu_{p+1}^m \: : \: \cu_{p+1}^m \cap V_j \not = \emptyset}  \left(  \sum_{\cu_p^m \subseteq \cu_{p+1}^m}  \lvert \avsum_{x \in \cu_p^m\cap \X_n} h(x) \rvert^2 \right)^{1/2} \\
		& \qquad\qquad\qquad\qquad\qquad\qquad \times \left(    \sum_{x \in \cu_{p+1}^m\cap \X_n}( g^j(x) - g^j(\cu_{p+1}^m))^2 \right)^{1/2},
		\end{align*} 
		where in the second equality we have used the fact that both $g^j_p$ and $g^j_{p+1}$ are constant within each $\cu_p^m$ (because the partitions are nested) and we have implicitly used $\cu_{p+1}^m$ to denote the set at scale $p+1$ that contains the set $\cu_{p}^m$; all inequalities follow from Cauchy-Schwarz inequality and counting. The latter term in the above inequality can be bounded, using the local Poincar\'{e}  inequality from Lemma \ref{lem:local poincare inequality} and \eqref{eq:local dirichlet energy gj}, by the term
		\begin{align*}
		&\leq C n \sqrt{ 3^{(p-m)d}}  \sum _{\cu_{p+1}^m \: : \: \cu_{p+1}^m \cap V_j \not = \emptyset}  \left(  \sum_{\cu_p^m \subseteq \cu_{p+1}^m}  \Big\lvert \avsum_{x \in \cu_p^m\cap \X_n} h(x) \Big\rvert^2 \right)^{1/2} \\
		&\qquad\qquad \times\left(   C3^{2(p+1 - m )}\frac{1}{n^2\e_n^{d+2}} \sum _{x, y \in \cu_{p+1}^m\cap\X_n, \: x \sim y} (g^j(x) - g^j(y))^2    \right)^{1/2} 
		\\& \leq Cn  \sqrt{ 3^{(p-m)d}}   \left( \sum _{\cu_{p+1}^m \: : \: \cu_{p+1}^m \cap V_j \not = \emptyset}  \sum_{\cu_p^m \subseteq \cu_{p+1}^m}  \Big\lvert \avsum_{x \in \cu_p^m\cap \X_n} h(x) \Big\rvert^2 \right)^{1/2} \\
		&\qquad\qquad \times\left(   C3^{2(p+1 - m )}\frac{1}{n^2\e_n^{d+2}} \sum _{\cu_{p+1}^m \: : \: \cu_{p+1}^m \cap V_j \not = \emptyset} \sum _{x, y \in \cu_{p+1}^m\cap\X_n, \: x \sim y} (g^j(x) - g^j(y))^2    \right)^{1/2} 
		\\& \leq Cn 3^{(p+1 -m)}  \sqrt{ 3^{(p-m)d}}   \left( \sum _{\cu_{p+1}^m \: : \: \cu_{p+1}^m \cap V_j \not = \emptyset}  \sum_{\cu_p^m \subseteq \cu_{p+1}^m}  \Big\lvert \avsum_{x \in \cu_p^m\cap \X_n} h(x) \Big\rvert^2 \right)^{1/2}
		\\& =  Cn 3^{(p+1 -m)}  \sqrt{ 3^{(p-m)d}}   \left(   \sum_{\cu_p^m \: : \: \cu_{p}^m \cap V_j \not = \emptyset}  \lvert \avsum_{x \in \cu_p^m\cap \X_n} h(x) \rvert^2 \right)^{1/2}
		\\& = C n 3^{p+1 -m}  \left(   \avsum_{\upsilon \in \Upsilon_{p,j}^m}  \Big\lvert \avsum_{x \in \cu_{p,j,\upsilon}^m\cap \X_n} h(x) \Big\rvert^2 \right)^{1/2}.
		\end{align*}
		
		For the third term on the right hand side of \eqref{eqn:AuxMultiPoincare0}, we first observe that the number of distinct $\cu^m_m$ in $\M$ is at most a constant $C>1$. By an application of Cauchy-Schwarz inequality, we obtain
		\begin{align*}
		\sum_{x\in\X_n \cap V_j} h(x)g^j_m(x) &= \sum _{\cu_m^m \: : \: \cu_m^m \cap V_j \not = \emptyset}  |\cu_m^m\cap\X_n|   \avsum_{x \in \cu_m^m\cap\X_n} h(x)  \avsum_{x \in \cu_m^m\cap\X_n} g^j(x) \\
		&\leq n  \sum _{\cu_m^m \: : \: \cu_m^m \cap V_j \not = \emptyset}  \left| \avsum_{x \in \cu_m^m\cap\X_n} h(x)  \avsum_{x \in \cu_m^m\cap\X_n} g^j(x)\right|\\
		&\leq C n \sqrt{\sum_{\cu_m^m \: : \: \cu_m^m \cap V_j \not = \emptyset} \left|\avsum_{x\in\cu_m^m\cap\X_n} h(x)\right|^2}\sqrt{\sum_{\cu_m^m \: : \: \cu_m^m \cap V_j \not = \emptyset} \left|\avsum_{x\in\cu_m^m\cap\X_n} g^j(x)\right|^2}\\
		&\leq C n \sqrt{\sum_{\cu_m^m \: : \: \cu_m^m \cap V_j \not = \emptyset} \left|\avsum_{x\in\cu_m^m\cap\X_n} h(x)\right|^2}.
		\end{align*} 
		
		
		Combining the above bounds we deduce the desired inequality.
	\end{proof}

	\subsection{Concentration Bounds on the Difference of Laplacians}	
	\label{sec:ConcentrationBounds}
	
	
	As discussed in section \ref{ss.graphpoincare}, in order to estimate the $\underline{H}^{-1}(\X_n)$ semi-norm of a function $h : \X_n \rightarrow \R$ it suffices to bound terms of the form $\avsum_{x \in \X_n} g(x) h(x) $ for a suitable collection of test functions $g$. For reasons that will become more apparent in sections \ref{ss.poisson} and \ref{ss.final}, but that were already hinted at in the discussion in section \ref{ss.ideas}, it will be important to develop these estimates for the choice $h:= \mathcal{L}_{\e_n, n} \overline{u} - \Delta_\rho \overline u $, where $\overline u: \M \rightarrow \R$ is sufficiently regular. Explicitly, we study concentration bounds for terms of the form 
	\begin{equation}
	\avsum_{x \in \X_n } g(x) (\Delta_\rho \overline{u}(x) -\mathcal{L}_{\e_n, n} \overline{u}(x)) ,
	\label{eqn:InnerProductDiffLapl}
	\end{equation}
	under different smoothness assumptions on $g$ and assuming that $\overline{u}$ has uniformly bounded derivatives of order three. 
    We present two results in this direction, Propositions \ref{Prop:ConcenSmoothg} and \ref{lem:pointwise convergence} below.  

    \begin{remark}
     In the computations that follow, by a priori assuming that $\overline{u}$ has bounded third order derivatives we will only require that $\M \in \MM$ (for some $\alpha \in [0,1]$) and that $\rho \in \mathcal{P}_\M$. The additional regularity on the data model required in Assumption \ref{assump:MoreRegularity} will only be needed in the next subsections, when justifying that certain choices of $\overline{u}$ of interest indeed have bounded third order derivatives.  
    \end{remark}
	
	In order to study \eqref{eqn:InnerProductDiffLapl}, we consider a second order Taylor expansion with exact remainder for a function $\overline{u}$ that has bounded third order derivatives:
	\begin{equation} \label{e.taylorexpansion}
	\begin{aligned}
	&\Delta_\rho \overline{u}(x) -\mathcal{L}_{\e_n,n} \overline{u}(x)   \\
	&\quad = \Delta_\rho \overline{u}(x) + \frac{2}{\sigma_\eta \e_n^2}\avsum_{y\in\X_n } \eta_{\e_n}(|y-x|) (\overline{u}(y)  - \overline{u}(x))  \\
	&\quad =  \Delta_\rho \overline{u}(x)+  \frac{2}{\sigma_\eta \e_n^2}\avsum_{y \in \X_n} \eta_{\e_n}(d(y,x)) (\overline{u}(y)  - \overline{u}(x))  \\
	&\qquad \qquad \qquad + \frac{2}{\sigma_\eta \e_n^2}\avsum_{y \in \X_n} (\eta_{\e_n}(|y-x|)-\eta_{\e_n}(d(y,x))) (\overline{u}(y)  - \overline{u}(x))  \\
	& \quad = \Delta_\rho \overline{u}(x)+ \frac{2}{\sigma_\eta \e_n^2} \avsum_{v:\exp_x(v)\in\X_n }\eta_{\e_n}(|v|)\biggl\{\nabla w(0)\cdot v + \frac12 \sum_{k,j=1}^d \partial^2_{kj} w(0) v_k v_j + \mathcal{R}(x,\exp_x(v))  \biggr\}  
	\\ & \qquad \qquad \qquad  +\mathcal{R}_1(x),
	\end{aligned}
	\end{equation}
	where we recall $d(\cdot, \cdot)$ denotes the geodesic distance on $\M$, $\exp_x$ the exponential map at $x$, and $w(v):= \overline{u}(\exp_x(v))$. Here, $\mathcal{R}(x, \exp_x(v))$ is the exact second order Taylor expansion remainder, which, according to Appendix \ref{app:Taylor}, can be written as
	\begin{equation}\label{eq-def:R}
	\mathcal{R}(x,\exp_x(v)) = \frac{\e_n^3}{2} \int_0^1 (1-t)^2   \sum_{i,j,k} \partial^3_{ijk} f(\exp_x(tv)) ((d\exp_x)_{tv}(v))_i ((d\exp_x)_{tv}(v))_j ((d\exp_x)_{tv}(v))_k \dt.
	\end{equation}
	The ``geometric'' remainder term~$\mathcal{R}_1$, on the other hand, is given by
	\begin{equation}\label{eq-def:R1}
	\mathcal{R}_1(x):=\frac{2}{\sigma_\eta \e_n^2}\avsum_{y \in \X_n} (\eta_{\e_n}(|y-x|)-\eta_{\e_n}(d(y,x))) (\overline{u}(y)  - \overline{u}(x))
	\end{equation}
	and measures the discrepancy between the graph Laplacians constructed with the Euclidean and geodesic distances.
	
	In our first result, we provide error bounds for the expectation of \eqref{eqn:InnerProductDiffLapl}.
	
	\begin{proposition}[Bias analysis]
		\label{prop:bias}
		 Let $g: \M \rightarrow \R$ and $\overline{u}: \M \rightarrow \R$ be two functions where $\int_\M|g(x)| \dx \leq C$, and $\overline{u}$ has bounded third order derivatives. Then
		\begin{equation}
		\label{eqn:BiasBound}
		\left|  \E\left[ \avsum_{x \in \X_n } g(x) (\Delta_\rho \overline{u}(x) - \mathcal{L}_{\e_n, n} \overline{u}(x))\right]  \right| \leq C( [g]_{1,\e_n} + C)B \e_n^2,
		\end{equation}
		where
		\begin{equation} \label{e.Bdef}
		\begin{split}
		B  &\defeq  \left(\int_\M \|\overline{u}\|^2_{L^\infty_{B_\M(x,\e_n)}}+ \|\nabla \overline{u}\|^2_{L^\infty_{B_\M(x,\e_n)}}+\|D^2 \overline{u}\|^2_{L^\infty_{B_\M(x,\e_n)}}+\|D^3 \overline{u}\|^2_{L^\infty_{B_\M(x,\e_n)}}\dx\right)^{\sfrac{1}{2}} \,,\\
		\end{split}
		\end{equation}
		and
		\begin{equation}
		[g]_{1, \e_n} :=  \sqrt{\int_{\M}\left(\int_{B_{1}(0)\subseteq \T_x\M} \eta(|v|) \sup_{t \in [0,1]} \frac{|g(x)-g(x+2t\e_n v)|}{\e_n} \dd v\right)^2\dx}.
		\label{eqn:DefNonLocalTV}
		\end{equation}
	\end{proposition}
	\begin{proof}        
		Since our goal is to bound the expectation of $ \avsum_{x \in \X_n } g(x) (\Delta_\rho \overline{u}(x) -\mathcal{L}_{\e_n, n} \overline{u}(x))$, it will suffice to compute the contribution of each of the terms in the last line of \eqref{e.taylorexpansion} to the overall expected value.

		\noindent \textbf{Step 1. Control of geometric remainder~$\mathcal{R}_1$.}
		Using \eqref{e.geodesic third derivative expansion} and a Taylor expansion of $\eta$ around $d(x,y)$, we deduce
		\begin{multline*}
		\avsum_{x\in\X_n} g(x) \mathcal{R}_1(x) =  \frac{2}{\sigma_\eta \e_n^2}\avsum_{x \in \X_n} \avsum_{y \in \X_n}  g(x) (\eta_{\e_n}(|y-x|)-\eta_{\e_n}(d(x,y)) (\overline{u}(y)  - \overline{u}(x))\\
		= \frac{2}{\sigma_\eta \e_n^2}\avsum_{x \in \X_n} \avsum_{y \in \X_n}  g(x) \eta'_{\e_n}(d(x,y)) \frac{s(x,y)d(x,y)^3}{\e_n} (\overline{u}(y)  - \overline{u}(x))+ \mathcal{O}(\e_n^2B),
		\end{multline*}
		where in the above and in what follows we use the notation $\eta'_{\e_n}(\cdot) = \frac{1}{\e^d_n} \eta'(\cdot/\varepsilon_n)$. 
		After changing to suitable normal coordinates, the expectation of the first term on the right hand side of the above expression can be written as 
		\begin{align}\label{eq:R11 term0}
		\begin{split}
		\E &\left[\frac{2}{\sigma_\eta \e_n^2}\avsum_{x \in \X_n} \avsum_{y \in \X_n}  g(x) \eta'_{\e_n}(d(x,y)) \frac{s(x,y)d(x,y)^3}{\e_n} (\overline{u}(y)  - \overline{u}(x))\right]\\
		& = \E\left[\frac{2}{ \sigma_\eta \e_n^3}  g(x) \eta'_{\e_n}(d(x,y)) s(x,y)d(x,y)^3 (\overline{u}(y)  - \overline{u}(x))\right]\\
		& = \frac{2}{\sigma_\eta \e_n^3}\int_\M\int_{B_{\e_n}(0)\subseteq \T_x\M}  \tilde{g}(0) \eta'_{\e_n}(|v|)\tilde{s}(0, \e_n v )|v|^3 (w(v)  - w(0)) \tilde{p}(0)\tilde{p}(v)J_x(v)\dd v\dd x\\
		&=\frac{2}{\sigma_\eta}\int_\M\int_{B_{1}(0)\subseteq \T_x\M}  \tilde{g}(0) \eta'(|v|) \tilde{s}(0, \e_n v )|v|^3 (w(\e_n v)  - w(0)) \tilde{p}(0)\tilde{p}(\e_n v)J_x(\e_n v)\dd v\dd x,
		\end{split}
		\end{align}
		where we use $\tilde{s}(0,v)=s(x,y)$ as the normal coordinate representation of $s$ around $x$, and $\tilde{p}(v)= \rho(\exp_x(v))$. By \eqref{e.Jactaylor}, we obtain
		\begin{multline}\label{eq:R11 term1}
		\left| \int_\M\int_{B_{1}(0)\subseteq \T_x\M}  \tilde{g}(0) \eta'(|v|) \tilde{s}(0, \e_n v )|v|^3 (w(\e_n v)  {-} w(0)) \tilde{p}(0)\tilde{p}(\e_n v)(J_x(\e_n v){-}1)\dd v\dd x \right| \\
		\leq CB\e_n^2,
		\end{multline}
		which allows us to focus on bounding the integral in the last line of \eqref{eq:R11 term0} but with the term $J_x(\e_n v)$ replaced by $1$. The resulting term can be decomposed into two further terms:
		\begin{multline}\label{eq:R11 term2}
		\int_\M\int_{B_{1}(0)\subseteq \T_x\M}  \tilde{g}(0) \eta'(|v|) \tilde{s}(0, \e_n v )|v|^3 (w(\e_n v)  - w(0)) \tilde{p}(0)\tilde{p}(\e_n v)\dd v\dd x \\
		= \int_\M\int_{B_{1}(0)\subseteq \T_x\M}  \tilde{g}(0) \eta'(|v|) \tilde{s}(0, \e_n v )|v|^3 \e_n \nabla w(0) \cdot v \tilde{p}(0)(\tilde{p}(0) + \e_n \nabla \tilde{p}(0) \cdot v )\dd v\dd x\\
		+\mathcal{R}_1^2,
		\end{multline}
		where $|\mathcal{R}_1^2|\leq CB\e_n^2$. The bound on $\mathcal{R}_1^2$ uses the estimate
		\begin{equation}\label{eq:taylor expansion of p}
		\Bigl| \tilde{p}(\e_n v) -  \tilde{p}(0) - \e_n \nabla \tilde{p}(0) \cdot v  \Bigr| \leqslant C \e_n^2 \,,
		\end{equation}
		which follows from a Taylor expansion on $\rho$ and the fact that second derivatives of $\rho$ are bounded, as well as a first order Taylor expansion of $\overline{u}$ around $x$. To bound the first term on the right-hand side of \eqref{eq:R11 term2}, we observe that due to symmetry over the integration variable $v$ (see \eqref{eqn:symmetrySecondFundForm}) we have 
		\begin{equation*}
		\int_\M\int_{B_{1}(0)\subseteq \T_x\M}  \tilde{g}(0) \eta'( |v|) \tilde{s}(0, \e_n v )|v|^3 \e_n \nabla w(0) \cdot v \tilde{p}(0)\tilde{p}(0) \dd v\dd x=0,
		\end{equation*}
		while the remaining terms in \eqref{eq:R11 term2} are, in absolute value, at most $CB\e_n^2$. From the above discussion, we conclude that 
		\begin{equation}
		\label{e.01}
		\left| \E \left[ \avsum_{x\in\X_n} g(x) \mathcal{R}_1(x) \right] \right| \leq CB\e_n^2.
		\end{equation}

		\noindent \textbf{Step 2: First order Taylor term.} The contribution of the first-order term in the Taylor expansion of $\overline{u}$ can be written as
		\begin{equation}
		\frac2{\sigma_\eta \e_n^2}\avsum_{x\in\X_n, y \in \X_n } \eta_{\e_n}(d(x,y))  g(x) \langle \nabla \overline{u}(x),  \log_x(y) \rangle.
		\label{eqn:FirstOrderTaylor}
		\end{equation}
		Note that
		\begin{align}
		\label{eq:first order taylor term main term}
		\begin{split}
		\E & \left[ \frac2{\sigma_\eta \e_n^2}\avsum_{x\in\X_n, y \in \X_n } \eta_{\e_n}(d(x,y))  g(x) \langle \nabla \overline{u}(x), \log_x(y) \rangle \right]
		=\E\left[ \frac2{\sigma_\eta\e_n^2}\eta_{\e_n}(d(x,y))  g(x) \langle \nabla \overline{u}(x) , \log_x(y) \rangle \right]\\
		& =  \int_\M\int_{B_{\e_n}(0)\subseteq\T_x\M} \frac2{\sigma_\eta \e_n^2}\eta_{\e_n}(|v|)  \tilde{g}(0)\nabla w(0) \cdot v\tilde{p}(0)\tilde{p}(v)J_x(v)\dd v \dx\\
		& =  \int_\M\int_{B_{1}(0)\subseteq\T_x\M} \frac2{ \sigma_\eta \e_n}\eta(|v|)  \tilde{g}(0)\nabla w(0) \cdot v\tilde{p}(0)\tilde{p}(\e_n v)J_x(\e_n v)\dd v \dx,
		\end{split}
		\end{align}
		where we have used the normal coordinates around a given $x\in \M$. By \eqref{e.Jactaylor-third derivative},
		\begin{align*}
		\Bigg| \int_\M\int_{B_{1}(0)\subseteq\T_x\M}  &\frac1{\e_n}\eta(|v|)   \tilde{g}(0)\nabla w(0) \cdot v\tilde{p}(0)\tilde{p}(\e_n v)(J_x(\e_n v)-1-q_x(\e_n v))\dd v \dx \Bigg| 
		\\  & \leq  C\e_n^2\int_\M\int_{B_{1}(0)\subseteq\T_x\M} \eta(|v|)  \left|\tilde{g}(0)\nabla w(0) \cdot v\tilde{p}(0)\tilde{p}(\e_n v)\right|\dd v \dx
		\\  & \leq CB\e_n^2.
		\end{align*}
		We can thus replace the last line in \eqref{eq:first order taylor term main term} with the term 
		\begin{equation}\label{eq:zero term in mathcalR2}
		\int_\M\int_{B_{1}(0)\subseteq\T_x\M} \frac2{\sigma_\eta \e_n}\eta(|v|)  \tilde{g}(0)\nabla w(0) \cdot v\tilde{p}(0)\tilde{p}(\e_n v)(1+q_x(\e_n v))\dd v \dx.
		\end{equation}
		To analyze this new term, we first study the contribution of the term $q_x(\e_n v)$. Using the Taylor expansion of $\tilde{p}$ in \eqref{eq:taylor expansion of p}, we deduce
		\begin{align}
		\begin{split}
		\label{eq:first term in mathcalR2}
		\int_\M\int_{B_{1}(0)\subseteq\T_x\M}& \frac1{\e_n}\eta(|v|)  \tilde{g}(0)\nabla w(0) \cdot v\tilde{p}(0)\tilde{p}(\e_n v)q_x(\e_n v)\dd v \dx\\
		&= \int_\M\int_{B_{1}(0)\subseteq\T_x\M} \frac1{\e_n}\eta(|v|)  \tilde{g}(0)\nabla w(0) \cdot v\tilde{p}^2(0)q_x(\e_n v)\dd v \dx\\
		&\quad +\int_\M\int_{B_{1}(0)\subseteq\T_x\M} \eta(|v|)  \tilde{g}(0)\nabla w(0) \cdot v\tilde{p}(0) (\nabla \tilde{p}(0)\cdot v)q_x(\e_n v)\dd v \dx+\mathcal{R}_2,
		\end{split}
		\end{align}
		where $\mathcal{R}_2$ satisfies
		\begin{equation*}
		\mathcal{R}_2\leq C\e_n^2\int_\M\int_{B_{1}(0)\subseteq\T_x\M} \eta(|v|)\left| \tilde{g}(0)\nabla w(0) \cdot v\tilde{p}(0) \right|p(\e_n v)\dd v \dx\leq CB\e_n^2.
		\end{equation*}
		Continuing the computation in \eqref{eq:first term in mathcalR2}, and recalling the symmetry of $q_x(\cdot)$ in \eqref{e.Jactaylor-third derivative}, we note that
		\begin{equation*}
		\int_\M\int_{B_{1}(0)\subseteq\T_x\M} \frac1{\e_n}\eta(|v|)  \tilde{g}(0)\nabla w(0) \cdot v\tilde{p}^2(0)q_x(\e_n v)\dd v \dx    = 0,
		\end{equation*}
		while the remaining term in \eqref{eq:first term in mathcalR2} is directly seen to satisfy
		\begin{equation*}
		\int_\M\int_{B_{1}(0)\subseteq\T_x\M} \eta(|v|)  \tilde{g}(0)\nabla w(0) \cdot v\tilde{p}(0) (\nabla \tilde{p}(0)\cdot v)q_x(\e_n v)\dd v \dx\leq CB\e_n^2.
		\end{equation*}

		For the other term in \eqref{eq:zero term in mathcalR2}, we consider the exact second order Taylor expansion for the function $t \mapsto \rho(\exp_{x}(t\e_n v))$,
		\begin{align*}
		\rho(\exp_x(\e_n v))-\rho(x) -\e_n\nabla \rho (x)\cdot v=\e_n^2\int_{0}^1 (1-t) \langle D^2\rho(\exp_x(t\e_nv )) d (\exp_x)_{t\e_n v }(v)   ,  d (\exp_x)_{t\e_n v }(v) \rangle \dt, 
		\end{align*}
		to obtain
		\begin{align}\label{eq:E left term}
		\begin{split}
		&\int_\M\int_{B_{1}(0)\subseteq\T_x\M}  \frac2{\sigma_\eta \e_n}\eta(|v|)  \tilde{g}(0)\nabla w(0) \cdot v\tilde{p}(0)\tilde{p}(\e_n v)\dd v \dx\\
		& =\int_\M\int_{B_{1}(0)\subseteq\T_x\M} \frac2{\sigma_\eta \e_n}\eta(|v|)  \tilde{g}(0)\nabla w(0) \cdot v\tilde{p}(0)\Bigl(\tilde{p}(0)+\e_n \nabla \tilde{p}(0)\cdot v \Bigr)\dd v \dx 
		\\ & \quad  +  \frac{2\e_n}{\sigma_\eta}\int_0^1(1-t) \mathcal{I}_t \dd t,
		\end{split}
		\end{align}
		where 
		\[ \mathcal{I}_t :=  \int_\M \int_{B_1(0) \subseteq T_x\M } \eta(|v|) g(x) \nabla \overline{u}(x) \cdot v  \rho(x) \langle D^2\rho(\exp_x(t\e_nv )) d (\exp_x)_{t\e_n v }(v)   ,  d (\exp_x)_{t\e_n v }(v) \rangle    \dd v \dx; \]
		it is important to highlight that here we need to use a first order Taylor expansion with exact remainder (and not, for example, a second order expansion with a remainder of order three) because in the class $\mathcal{P}_\M$ we have no control on derivatives of order more than two. 
		
		By symmetry over the variable $v$, 
		\begin{equation*}
		\begin{split}
		\int_\M\int_{B_{1}(0)\subseteq\T_x\M} \frac1{\e_n}\eta(|v|)  \tilde{g}(0)\nabla w(0) \cdot v\tilde{p}^2(0)\dd v \dx=0,\\
		\end{split}
		\end{equation*}
		while a direct computation using the definition of $\sigma_\eta$ reveals that
		\begin{equation*}
		\int_\M\int_{B_{1}(0)\subseteq\T_x\M} \eta(|v|)  \tilde{g}(0)\nabla w(0) \cdot v\tilde{p}(0)\nabla \tilde{p}(0)\cdot v\dd v \dx=\sigma_\eta\int_\M   g(x)\nabla \rho(x)\cdot \nabla \bar{u}(x) \rho(x)\dd x.
		\end{equation*}
		In order to bound the last term in \eqref{eq:E left term}, we bound $\mathcal{I}_t$ for all $t$. To do this, it is useful to introduce a tool from Riemannian geometry that will allow us to carry out a convenient symmetrization. Indeed, we use $\M'$s tangent bundle (see Appendix \ref{app:TangentBundle}), which we denote by $T\M$,  and endow it with its canonical metric and associated volume form. Then, for every $t \in [0,1]$, we introduce the change of variables $\Psi_t: T\M \rightarrow T \M $ given by
		\begin{equation}
		\label{eqn:ChangeVariables}
		(\tilde x , \tilde v ) = \Psi_t( x, v ):= (\exp_x(2t \e_n v), -d (\exp_x)_{2t\e _n v }(v) ). \footnote{In the flat setting, this corresponds to the change of variables $\tilde x = x + 2\e_n t v $, $\tilde v = -v$, and one can easily verify that this change of variables has Jacobian equal to one.}     
		\end{equation}
		The key feature of this map is that, up to the minus sign in the second coordinate, it is equal to the \textit{geodesic flow}, and hence, by Liouville's theorem (see, e.g., \cite[Chapter 3]{do1992riemannian}), it preserves the tangent bundle's volume form. In particular, the pushforward of $T\M$'s volume form by $\Psi_t$ is, again, $T\M's$ volume form. Because of this (see Appendix \ref{app:TangentBundle}), we have
		\begin{align*}
		&\int_\M \int_{B_1(0) \subseteq T_x\M } \eta (|v|) g(x) \rho(x) \nabla \overline{u} (x) \cdot v  \langle D^2\rho(\exp_x(t\e_nv )) d (\exp_x)_{t\e_n v }(v)   ,  d (\exp_x)_{t\e_n v }(v) \rangle  \dd v \dx
		\\ & =   \int_\M \int_{B_1(0) \subseteq T_x\M } \eta (|\tilde v|) g(\tilde x) \rho(\tilde x) \nabla \overline{u} (\tilde x) \cdot \tilde v  \langle D^2\rho(\exp_{\tilde x}(t\e_n\tilde v )) d (\exp_{\tilde x})_{t\e_n \tilde v }(\tilde v)   ,  d (\exp_{\tilde x})_{t\e_n \tilde v }(\tilde v) \rangle  \dd v \dx.
		\end{align*}
		On the other hand, the following identities are straightforward to verify from the definition of $(\tilde x, \tilde v)$:
		\[ |\tilde v | = |v|, \quad \exp_{\tilde x}(t \e_n \tilde v ) = \exp_{x}(t \e_n v ), \quad  d (\exp_x)_{t\e_n v}(v) = -  d (\exp_{\tilde x})_{t\e_n \tilde v}(\tilde v). \]
		Combining the above, we obtain
		\begin{align*}
		&\int_\M \int_{B_1(0) \subseteq T_x\M } \eta (|v|) g(x) \rho(x) \nabla \overline{u} (x) \cdot v  \langle D^2\rho(\exp_x(t\e_nv )) d (\exp_x)_{t\e_n v }(v)   ,  d (\exp_x)_{t\e_n v }(v) \rangle  \dd v \dx
		\\ & =   \int_\M \int_{B_1(0) \subseteq T_x\M } \eta (|v|) g(\tilde x) \rho(\tilde x) \nabla \overline{u} (\tilde x) \cdot \tilde v   \langle D^2\rho(\exp_x(t\e_nv )) d (\exp_x)_{t\e_n v }(v)   ,  d (\exp_x)_{t\e_n v }(v) \rangle  \dd v \dx.
		\end{align*}
		In addition, as in \eqref{eq:geodesic gamma computation}, we have
		\[ | \nabla \overline{u}(\tilde x) \cdot \tilde v   + \nabla \overline{u}( x) \cdot  v | \leq  C \lVert D^2 \overline{u} \rVert_{L^\infty(B_\M(x,\e_n))} \e_n ,   \]
		and, as a consequence, 
		\begin{equation*}
		\begin{split}
		&\Bigl| \int_\M \int_{B_1(0) \subseteq T_x\M } \eta (|v|) g(x) \rho(x) \nabla \overline{u} (x) \cdot v  \langle D^2\rho(\exp_x(t\e_nv )) d (\exp_x)_{t\e_n v }(v)   ,  d (\exp_x)_{t\e_n v }(v) \rangle  \dd v \dx \Bigr|\\
		&\qquad\leq \Bigl|\frac{1}{2}\int_\M \int_{B_{1}(0)\subseteq\T_{x}\M} \eta(|v|)  \langle D^2\rho(\exp_x(t\e_nv )) d (\exp_x)_{t\e_n v }(v)   ,  d (\exp_x)_{t\e_n v }(v) \rangle   \\
		&\qquad\qquad \qquad  \qquad \qquad \quad \nabla \overline{u}( x) \cdot  v  \Bigl(g(x) \rho(x)-g(\exp_x(2\e_nt v)) \rho(\exp_x(2\e_nt v))\Bigr) \dd v \dx\Bigr| + CB \e_n\\
		&\qquad\leq \frac{1}{2}\int_\M \int_{B_{1}(0)\subseteq\T_{x}\M} \eta(|v|) \left| \langle D^2\rho(\exp_x(t\e_nv )) d (\exp_x)_{t\e_n v }(v)   ,  d (\exp_x)_{t\e_n v }(v) \rangle \right|      \\
		& \qquad\qquad \qquad  \qquad \qquad \quad  \Bigl|\nabla \overline{u}(x) \cdot v \Bigr| \Bigl| g(x) \rho(x)-g(\exp_x(2\e_nt v)) \rho(\exp_x(2\e_nt v)) \Bigr| \dd v \dx + CB \e_n\\
        &\qquad\leq \frac{1}{2}\int_\M \sup_{v\in B_{1}(0)}\Bigl|\nabla \overline{u}(x) \cdot v \Bigr| \int_{B_{1}(0)\subseteq\T_{x}\M} \eta(|v|) \left| \langle D^2\rho(\exp_x(t\e_nv )) d (\exp_x)_{t\e_n v }(v)   ,  d (\exp_x)_{t\e_n v }(v) \rangle \right|      \\
		& \qquad\qquad \qquad  \qquad \qquad \quad   \Bigl| g(x) \rho(x)-g(\exp_x(2\e_nt v)) \rho(\exp_x(2\e_nt v)) \Bigr| \dd v \dx + CB \e_n\\
        &\qquad\leq C\sqrt{\int_\M\left( \sup_{v\in B_{1}(0)}\Bigl|\nabla \overline{u}(x) \cdot v \Bigr|\right)^2\dx }\sqrt{\int_\M \Bigl(\int_{B_{1}(0)\subseteq\T_{x}\M} \eta(|v|)  \Bigl| g(x) \rho(x)-g(\exp_x(2\e_nt v)) \Bigr| \dd v\Bigr)^2 \dx} + CB \e_n\\
		&\qquad\leq  C B \e_n \left(\sqrt{\int_{\M}\left(\int_{B_{1}(0)\subseteq \T_x\M} \eta(|v|) \frac{|g(x)-g(x+2t\e_n v)|}{\e_n} \dd v\right)^2\dx} + C\right)\leq C B ( [g ]_{1, \e_n} + C) \e_n .
		\end{split}
		\end{equation*}
        
		Putting together all the above computations, we obtain the following estimate:
		\begin{multline}\label{e.2}
		\Biggl|\E\left[ \frac2{\sigma_\eta \e_n^2}\avsum_{x\in\X_n, y \in \X_n } \eta_{\e_n}(d(x,y))  g(x)\nabla \overline{u}(x) \cdot \log_x(y) \right]\\
		-2\int_\M   g(x)\nabla \rho(x)\cdot \nabla \bar{u}(x) \rho(x)\dd x\Biggr|\leq CB(  [g]_{1,\e_n} +C)\e_n^2.
		\end{multline}
		In other words, up to an error of order $C([g]_{1,\e_n} +C)B\e_n^2$, the expectation of \eqref{eqn:FirstOrderTaylor} is equal to
		\[2\int_\M   g(x)\nabla \rho(x)\cdot \nabla \bar{u}(x) \rho(x)\dd x. \]
		
		\noindent \textbf{Step 3: Second order Taylor term.} Next, we consider the contribution of the second order Taylor term of $\overline{u}$, which can be written as:
		\begin{equation}
		\frac1{\sigma_\eta \e_n^2}\avsum_{x\in\X_n, y \in \X_n } \eta_{\e_n}(d(x,y))  g(x) \langle D^2\overline{u}(x) \log_x(y),  \log_x(y) \rangle
		\label{eqn:SecondOrderTaylor}
		\end{equation}
		and whose expectation takes the form
		\begin{align}
		\label{eq:second order taylor term main term}
		\begin{split}
		\E & \left[ \frac1{\sigma_\eta \e_n^2}\avsum_{x\in\X_n, y \in \X_n } \eta_{\e_n}(d(x,y))  g(x) \langle D^2\overline{u}(x) \log_x(y),  \log_x(y) \rangle \right]
		\\ & =  \int_\M\int_{B_{1}(0)\subseteq\T_x\M}  \frac1{\sigma_\eta}\eta(|v|)  \langle D^2\overline{u}(x) v,  v \rangle \tilde g(0) \tilde{p}(0)\tilde{p}(\e_n v)J_x(\e_n v)\dd v \dx.
		\end{split}
		\end{align}
		Now, thanks to \eqref{e.Jactaylor}, we can focus on estimating the above integral when we substitute the term $J_x(\e_n v)$ with $1$. In turn, after considering a simple first order Taylor expansion for $\rho$, we can conclude that, up to a term of order $CB\e_n^2$, the above expectation is equal to 
		\[   \int_\M\int_{B_{1}(0)\subseteq\T_x\M}  \frac1{\sigma_\eta}\eta(|v|)  \langle D^2\overline{u}(x) v,  v \rangle \tilde g(0) \tilde{p}(0)(\tilde p(0) + \nabla \rho(x) \cdot v  ) \dd v \dx. \]
		From a direct computation using the definition of $\sigma_\eta$, we conclude that
		\[   \int_\M\int_{B_{1}(0)\subseteq\T_x\M}  \frac1{\sigma_\eta}\eta(|v|)  \langle D^2\overline{u}(x) v,  v \rangle \tilde g(0) (\tilde{p}(0))^2\dd v \dx = \int_\M \Delta \overline{u}(x) g(x) \rho^2(x) \dx  ,\]
		where $\Delta$ denotes the standard Laplace-Beltrami operator on $\M$ (i.e., $\Delta = \divergence \nabla$), while 
		\[   \int_\M\int_{B_{1}(0)\subseteq\T_x\M}  \frac1{\sigma_\eta}\eta(|v|)  \langle D^2\overline{u}(x) v,  v \rangle \tilde g(0) \tilde{p}(0) \nabla \rho(x) \cdot v  \dd v \dx =0 , \]
		since each inner integral in the latter expression is equal to zero by symmetry.
		
		Putting together all the above computations, we obtain the following estimate:
		\begin{multline}\label{e.3}
		\Biggl|\E\left[  \frac1{\sigma_\eta \e_n^2}\avsum_{x\in\X_n, y \in \X_n } \eta_{\e_n}(d(x,y))  g(x) \langle D^2\overline{u}(x) \log_x(y),  \log_x(y) \rangle \right]\\
		- \int_\M \Delta \overline{u}(x) g(x) \rho^2(x) \dx \Biggr|\leq CB \e_n^2.
		\end{multline}

		\noindent \textbf{Step 4: Taylor expansion remainder.} Finally, we consider the contribution of the remainder term $\mathcal{R}$, which can be written as 
		\[ \frac{2}{\e_n^2\sigma_\eta}\avsum_{x,y \in \X_n} \eta_{\e_n}(d(x,y)) g(x) \mathcal{R}(x,y).\]
		By similar arguments as in the previous steps, to estimate the expectation of this term it would suffice to estimate the integral
		\begin{align*}
		\mathcal{II}_t:= \int_\M \int_{B_1(0) \subseteq T_x \M} \Bigl(\eta(|v|) g(x) & (\rho(x))^2  \sum_{i,j,k} \partial^3_{ijk} \overline{u}(\exp_x(tv)) 
		\\ & ((d\exp_x)_{tv}(v))_i ((d\exp_x)_{tv}(v))_j ((d\exp_x)_{tv}(v))_k \Bigr) \dd v \dx 
		\end{align*}
		for each $t\in [0,1]$ and show that it is of order $([g]_{1, \e_n} + C)B \e_n$. By following a similar computation as in the change of variables in \eqref{eqn:ChangeVariables} for the first order Taylor term, and symmetrizing the above integral, we obtain 
		\[  |\mathcal{II}_t| \leq C([g]_{1, \e_n} + C) B \e_n. \]
		From this we deduce
		\begin{equation}
		\left|\E\left[ \frac{2}{\e_n^2\sigma_\eta}\avsum_{x,y \in \X_n} \eta_{\e_n}(d(x,y)) g(x) \mathcal{R}(x,y)\right]\right| \leq C([g]_{1, \e_n} + C) B \e_n^2. 
		\label{e.4}
		\end{equation}
		
		\noindent \textbf{Conclusion:} Putting together the estimates in all the above steps, we conclude that
		\begin{align*}
		\Bigl| \E[ \avsum_{x \in \X_n } g(x) \mathcal{L}_{\e_n, n} \overline{u}(x)  ] + 2 \int_{\M} g(x) \nabla \rho(x) \cdot \nabla \overline{u}(x) \rho(x) \dx &+  \int_{\M} g(x) \rho^2(x) \Delta \overline{u}(x) \dx \Bigr|
		\\ & \leq C B \e_n^2 ,
		\end{align*}  
		whereas a direct computation gives
		\begin{align*}
		\E[ \avsum_{x \in \X_n } g(x) \Delta_\rho \overline{u}(x)] & = \int_\M g(x)  \Delta_\rho(x) \rho(x) \dx 
		\\& = -  2\int_\M g(x) \nabla \rho(x) \cdot \nabla \overline{u}(x) \rho(x)\dx - \int_\M g(x) \rho^2(x) \Delta \overline{u}(x) \dx  . 
		\end{align*}
		Inequality \eqref{eqn:BiasBound} follows from these two estimates.
	\end{proof}

	After carrying out the above bias analysis, we focus on obtaining concentration bounds for \eqref{eqn:InnerProductDiffLapl}. We consider two cases depending on the regularity of $g$. In our first result, we assume $g$ to be uniformly bounded by an order one constant and to have a bounded gradient.

	\begin{proposition}[Concentration bounds for regular test functions]
		\label{Prop:ConcenSmoothg}
		Under the same assumptions on $\overline{u}$ as in Proposition \ref{prop:bias}, for any $g: \M \rightarrow \R$ with~$\int_\M |\nabla g|^4 \dx \leq C$, and $\|g\|_{L^\infty(\M)}\leq C$, the following holds: for any~$t>0,$
		\begin{multline*}
		\P\Biggl[\left|\avsum_{x\in\X_n} g(x)\biggl(\Delta_\rho \overline{u}(x) - \mathcal{L}_{\e_n,n} \overline{u}(x) \biggr)\right| >t+C\e_n^2B\biggr]\\
		\le 2\exp \biggl(- \frac{cn\e_n^d t^2}{B^2+ \e_n t\|D^3\overline{u}\|_{L^\infty(\M)}+ t \|D^2\overline{u}\|_{L^\infty(\M)}+\frac{t\|D\overline{u}\|_{L^\infty(\M)}}{\e_n}} \biggr)\,,
		\end{multline*}
		with $B$ as in \eqref{e.Bdef}.
         Here $C>0$ and $0<c<1$ are constants that only depend on the parameters determining the families $\MM$ and $\mathcal{P}_{\M}$. 
	\end{proposition}
	\begin{proof}
		Under the stated additional assumptions on $g$, it is straightforward to show that $[g]_{1,\e_n} \leq C$. Thanks to this, Proposition \ref{prop:bias} implies
		\begin{equation*}
		\left|  \E\left[ \avsum_{x \in \X_n } g(x) (\Delta_\rho \overline{u}(x) - \mathcal{L}_{\e_n, n} \overline{u}(x))\right]  \right| \leq CB\e_n^2,
		\end{equation*} 
		and it will thus suffice to prove that
		\begin{multline}
		\P\Biggl[\left|\avsum_{x\in\X_n} g(x)\biggl(\Delta_\rho \overline{u}(x) - \mathcal{L}_{\e_n,n} \overline{u}(x) \biggr) - \E\left[  \avsum_{x\in\X_n} g(x)\biggl(\Delta_\rho \overline{u}(x) - \mathcal{L}_{\e_n,n} \overline{u}(x) \biggr) \right]\right| >t+CB\e_n^2\biggr]\\
		\le 2\exp \biggl(- \frac{cn\e_n^d t^2}{B^2+ \e_n t\|D^3\overline{u}\|_{L^\infty(\M)}+ t \|D^2\overline{u}\|_{L^\infty(\M)}+\frac{t\|D\overline{u}\|_{L^\infty(\M)}}{\e_n}} \biggr)\,.
		\label{eq:ConcentrationBounds}
		\end{multline}
		To prove this inequality, we use a similar decomposition as in the proof of Proposition \ref{prop:bias} and focus on obtaining concentration bounds for each of the terms. 
		
		\medskip
		
		\noindent \textbf{Step 1: Control of geometric remainder~$\mathcal{R}_1$.}
		Recall that from the first step in the proof of Proposition \ref{prop:bias} we have	
		\begin{multline*}
		\avsum_{x\in\X_n} g(x) \mathcal{R}_1(x) = 
		\frac{2}{\sigma_\eta \e_n^2}\avsum_{x \in \X_n} \avsum_{y \in \X_n}  g(x) \eta'_{\e_n}(d(x,y)) \frac{s(x,y)d(x,y)^3}{\e_n} (\overline{u}(y)  - \overline{u}(x)) + \mathcal{O}(B \e_n^2),
		\end{multline*}
		and so it will suffice to prove a concentration bound for the first term on the right hand side. We use the concentration bound for U-statistics in Lemma \ref{l.boundeddiff} with $-\eta'$ instead of $\eta$, and with 
		\[\mathcal{K}(x,y)= \frac{1}{\e_n^3} g(x) s(x,y) d(x,y)^3 (\overline{u}(y) - \overline{u}(x)).  \]
		We can easily verify that $\mathcal{B}_\mathcal{K} =  C\e_n^2 B^2,C_{\mathcal{K}} = \e_n \|\nabla \overline{u}\|_{L^\infty(\M)}$ and thus conclude that, with probability at least 
        $1-C\exp\left(-\frac{cn\e_n^dt^2}{\e_n^2 B^2+t\e_n \|\nabla \overline{u}\|_{L^\infty(\M)} }\right),$
        we have
		\begin{multline*}
		\Bigg|\frac{1}{\e_n^2}\avsum_{x \in \X_n} \avsum_{y \in \X_n}  g(x) \eta'_{\e_n}(d(x,y)) \frac{s(x,y)d(x,y)^3}{\e_n} (\overline{u}(y)  - \overline{u}(x))\\
		-\E\left[\frac{1}{\e_n^2}\avsum_{x \in \X_n} \avsum_{y \in \X_n}  g(x) \eta'_{\e_n}(d(x,y)) \frac{s(x,y)d(x,y)^3}{\e_n} (\overline{u}(y)  - \overline{u}(x))\right]\Bigg|\leq t.
		\end{multline*}

		\noindent \textbf{Step 2: Control of first order Taylor term.}
		Next, we derive concentration bounds for the difference between
		\begin{equation}
		\frac2{\sigma_\eta \e_n^2}\avsum_{x\in\X_n, y \in \X_n } \eta_{\e_n}(d(x,y))  g(x) \langle \nabla \overline{u}(x),  \log_x(y) \rangle
		\end{equation}		
		and its expectation. However, we do not apply Lemma \ref{l.boundeddiff} directly and instead (and quite importantly) first symmetrize this term to write it as 
		\begin{equation}\label{eq:symmetrization}
		\frac1{\sigma_\eta \e_n^2}\avsum_{x\in\X_n, y \in \X_n } \eta_{\e_n}(d(x,y)) ( g(x)\nabla \overline{u}(x) \cdot \log_x(y) +  g(y)\nabla \overline{u}(y) \cdot \log_y(x)).
		\end{equation}
		We then seek to apply the concentration inequality in Lemma~\ref{l.boundeddiff}, with the choice
		\begin{equation}\label{eq-def:K}
		\mathcal{K}(x,y)  \defeq  \frac{1}{\sigma_\eta\e_n^2}\left(g(x)\nabla \overline{u}(x) \cdot \log_x(y) +  g(y)\nabla \overline{u}(y) \cdot \log_y(x)\right)\,.  
		\end{equation}
		For this choice of kernel, and using the estimate in \eqref{eq:geodesic gamma computation}, we deduce
		\begin{equation*}
		\begin{aligned}
		& \int_\M\left(\int_{B_{1}(0)\subseteq T_x\M} \eta(|v|)\mathcal{K}(x,\exp_x(\e_n v))\dd v\right)^2\dx\\
		&\quad   \leq \frac{1}{\e_n^4}\int_\M \Biggl(\int_{B_{1}(0)\subseteq T_x\M} \eta_{\e_n}(|v|)\Bigl(g(x)\nabla \overline{u}(x) \cdot v \\
		&\qquad\qquad\qquad\qquad\qquad\qquad +  g(\exp_x(\e_n v))\nabla \overline{u}(\exp_x(\e_n v)) \cdot \log_{\exp_x(\e_n v)}(x)\Bigr)\dd v \Biggr)^2\dx\\
		&\quad   \leq \frac{1}{\e_n^4}\int_\M\left(\int_{B_{1}(0)\subseteq T_x\M} \eta_{\e_n}(|v|)\left(g(x)\nabla \overline{u}(x) \cdot v -  g(\exp_x(v))\nabla \overline{u}(x) \cdot v\right)\dd v\right)^2\dx +CB^2\\
		&\quad   \leq \frac{1}{\e_n^4}\int_\M\left(\int_{B_{1}(0)\subseteq T_x\M} \eta_{\e_n}(|v|)\left|g(x)-g(\exp_x(v))\right| \left|\nabla \overline{u}(x) \cdot v \right|\dd v\right)^2\dx+CB^2\\
		&\quad   \leq \frac{1}{\e_n^2}\int_\M\int_{B_{1}(0)\subseteq T_x\M} \eta^2(|v|)\left(g(x)-g(\exp_x(\e_n v))\right)^2 \left|\nabla \overline{u}(x) \cdot v \right|^2\dd v\dx+CB^2
		\leq CB^2,
		\end{aligned}
		\end{equation*}
		where the last inequality follows from Cauchy-Schwarz inequality and the fact that the $L^4(\M)$ norm of the gradient of $g$ is bounded by a constant. 
        Similar to the above computation for variance, we see that
		\begin{equation*}
		\sup_{x\in\M}\sup_{v\in B_{1}(0) \subseteq T_x\M } \left|\mathcal{K}(x,\exp_x(\e_n v))\right|
        \leq \frac{C\|\nabla \overline{u}\|_{L^\infty(\M)}}{\e_n}.
		\end{equation*}
		We conclude that
		\begin{align}
		\begin{split}
		\P\Biggl[ & \Biggl| \frac2{\sigma_\eta \e_n^2}\avsum_{x\in\X_n, y \in \X_n } \eta_{\e_n}(d(x,y))  g(x) \langle \nabla \overline{u}(x),  \log_x(y) \rangle \\
		&\qquad  -\E\left[ \frac2{\sigma_\eta \e_n^2}\avsum_{x\in\X_n, y \in \X_n } \eta_{\e_n}(d(x,y))  g(x) \langle \nabla \overline{u}(x),  \log_x(y) \rangle \right]\Biggr| \geqslant t \Biggr]	\\
        &\qquad \qquad \le  C \exp\biggl( -  \frac{cn\e_n^d t^2}{B^2+ \frac{Ct\|\nabla \overline{u}\|_{L^\infty(\M)}}{\e_n}}\biggr)\,.
		\end{split}
		\end{align}

		\medskip
		
		\noindent \textbf{Step 3: Second and third order Taylor terms.} For the second and third order Taylor terms, we observe that
		\[  \frac{2}{\e_n^2\sigma_\eta}\avsum_{x, y \X_n}g(x) \eta_{\e_n}(d(x,y)) \left(\frac{1}{2}\langle D^2\overline{u}(x) \log_x(y),  \log_x(y) \rangle    +  \mathcal{R}(x,y)\right)    \]
		can be written as a U-statistic as in \eqref{def:UStatistic} for a kernel $\mathcal{K}$ with 
		\[ B_\mathcal{K} \leq C B^2   , \quad C_{\mathcal{K}} \leq  C\|D^2\overline{u}\|_{L^\infty(\M)}+C\e_n\|D^3\overline{u}\|_{L^\infty(\M)}.  \]
		Hence, 
		\begin{align}\label{eq:1}
		\begin{split}
		\P\Bigl[  & \biggl| \frac{2}{\e_n^2\sigma_\eta}\avsum_{x, y \X_n}g(x) (\frac{1}{2}\langle D^2\overline{u}(x) \log_x(y),  \log_x(y) \rangle    +  \mathcal{R}(x,y))  \\
		&-   \E\left[ \frac{2}{\e_n^2\sigma_\eta}\avsum_{x, y \X_n}g(x) (\frac{1}{2}\langle D^2\overline{u}(x) \log_x(y),  \log_x(y) \rangle    +  \mathcal{R}(x,y))  \right] \biggr| \geqslant t \Bigr]	
		\\ & \qquad  \le  C \exp\biggl( -  \frac{cn\e_n^d t^2}{B^2+ \e_n t\|D^3\overline{u}\|_{L^\infty(\M)}+ t \|D^2\overline{u}\|_{L^\infty(\M)}}\biggr)\,.
		\end{split}
		\end{align}
		\medskip        
		
		\noindent \textbf{Conclusion:}
		Using Bernstein's inequality (i.e., Lemma \ref{lemm:Bernstein}), we conclude that
		\begin{align*}
		\P\Bigl[   \biggl| \avsum_{x\in\X_n} g(x) \Delta_\rho \overline{u}(x)   
		-   \E\left[ \avsum_{x\in\X_n} g(x) \Delta_\rho \overline{u}(x)  \right] \biggr| \geqslant t \Bigr] \\
        \leq   C \exp\biggl( -  \frac{cn\e_n^d t^2}{B^2+ \e_n t\|D^3\overline{u}\|_{L^\infty(\M)}+ t \|D^2\overline{u}\|_{L^\infty(\M)}+\frac{t\|D\overline{u}\|_{L^\infty(\M)}}{\e_n}}\biggr).
		\end{align*}
		Putting together the estimates in the previous steps and the above bound, 
		we deduce \eqref{eq:ConcentrationBounds} and with it we complete the proof.

	\end{proof}

	Next, we adjust the previous result to consider test functions $g$ that are rescaled versions of indicators of sets as in the multiscale Poincaré inequality in Proposition \ref{l.msp}. In contrast to the case considered previously, when $g$ is a rescaled version of an indicator function the term $[g]_{1, \e_n}$ cannot be bounded by an order one constant. This affects both the bias and variance of \eqref{eqn:InnerProductDiffLapl}. Our analysis will thus have to rely on a much more careful handling of the different error terms that appear in our calculation.

	\begin{proposition}[Concentration bounds for rescaled indicator functions of cells]
		\label{lem:pointwise convergence}
		Let $\cu= \cu_{p,j,\upsilon}^m$ be a cell as in section \ref{ss.graphpoincare} and let $\ell = 3^{p-m}$. Let $g= \frac{\indc_{\cu}}{\rho(\cu)}$. There exist constants~$C>0$ and $0<c<1$ depending on $\eta,$ $\MM$ and $\mathcal{P}_\M$ such that
		\begin{multline*}
		\P\Biggl[\left|\avsum_{x\in\X_n} g(x)\biggl(\Delta_\rho \overline{u}(x) - \mathcal{L}_{\e_n,n} \overline{u}(x) \biggr)\right| \geq C\frac{\e_n^2}{\ell}(B+1)\biggr]
		\\
        \leq Cn\exp \Biggl(- \frac{cn\e_n^{d+4}B^2}{B^2 +\frac{\e_n^2}{\e_n}\|\nabla \overline{u}\|_{L^\infty(\M)}+\e_n^2\|D^2\overline{u}\|_{L^\infty(\M)}+\e_n^3\|D^3\overline{u}\|_{L^\infty(\M)}  } \Biggr)\,,
		\end{multline*}
		with $B$ as in \eqref{e.Bdef}.
	\end{proposition}

	\begin{proof}
	
For $g$ of the form $\frac{\mathbf{1}_\cu}{\rho(\cu)}$ it is straightforward to see that
\[ [g]_{1,\e_n} \leq \frac{C}{\ell},  \]
and hence, thanks to Proposition \ref{prop:bias}, we have
\begin{equation*}
\left|  \E\left[ \avsum_{x \in \X_n } g(x) (\Delta_\rho \overline{u}(x) - \mathcal{L}_{\e_n, n} \overline{u}(x))\right]  \right| \leq CB\frac{\e_n^2}{\ell}.
\end{equation*} 
Thanks to the above, it will thus suffice to prove that
\begin{multline}
\P\Biggl[\left|\avsum_{x\in\X_n} g(x)\biggl(\Delta_\rho \overline{u}(x) - \mathcal{L}_{\e_n,n} \overline{u}(x) \biggr) - \E\left[  \avsum_{x\in\X_n} g(x)\biggl(\Delta_\rho \overline{u}(x) - \mathcal{L}_{\e_n,n} \overline{u}(x) \biggr) \right]\right|  >C(1+ B)\frac{\e_n^2}{\ell}  \biggr]\\
\le  Cn \exp \biggl( - \frac{cn\e_n^{d+4}B^2}{B^2 +\frac{\e_n^2}{\e_n}\|\nabla \overline{u}\|_{L^\infty(\M)}+\e_n^2\|D^2\overline{u}\|_{L^\infty(\M)}+\e_n^3\|D^3\overline{u}\|_{L^\infty(\M)}  }\biggr)\,.
\label{eq:claim-1}
\end{multline}	
Now, to prove this bound, it is actually convenient to consider the decomposition 
\begin{align*}
\avsum_{x\in\X_n} g(x)\left(\Delta_\rho \overline{u}(x) - \mathcal{L}_{\e_n,n} \overline{u}(x) \right) & =    \frac{2}{\sigma_\eta \e_n^2}\avsum_{x\in\X_n}  \avsum_{y \in \X_n} g(x) \eta_{\e_n}(d(x,y))  \nabla \overline{u}(x) \cdot \log_x(y) 
\\ & \quad + \avsum_{x \in \X_n} g(x)(\Delta_\rho \overline{u}(x) - \mathcal{A}_{\e_n, n} \overline{u}(x)), 
\end{align*}
where
\[ \mathcal{A}_{n, \e_n} \overline{u}(x) :=   -\frac{2}{\sigma_\eta \e_n^2} \avsum_{y\in\X_n }\eta_{\e_n}(d(x,y))\biggl\{\ \frac12 \langle  D^2 \overline{u}(x) \log_x(y), \log_x(y) \rangle+ \mathcal{R}(x,y)  \biggr\}   -\mathcal{R}_1(x);    \]
we recall that $\mathcal{R}$ and $\mathcal{R}_1$ were defined in \eqref{eq-def:R} and \eqref{eq-def:R1}, respectively.
In this decomposition, we isolate the contribution of the first order Taylor term in \eqref{e.taylorexpansion}, given that this term needs to be handled through a special symmetrization argument. Indeed, we will use the fact that
\begin{align*}
 \frac{2}{\sigma_\eta \e_n^2}\avsum_{x\in\X_n}   \avsum_{y \in \X_n} g(x) \eta_{\e_n}(d(x,y))  \nabla \overline{u}(x) \cdot \log_x(y)   =  \avsum_{x\in\X_n}  \mathcal{I}_{\e_n,n}(x),
\end{align*}
where
\begin{equation}
 \mathcal{I}_{\e_n, n}(x):=   \frac{1}{\sigma_\eta \e_n^2} \avsum_{y\in\X_n }\eta_{\e_n}(d(x,y))(g(x) \nabla \overline{u}(x)\cdot \log_x(y) + g(y) \nabla \overline{u}(y)\cdot \log_y(x))  .
\end{equation}
Inequality \eqref{eq:claim-1} will follow from 
\begin{multline}
\P\Bigl(  \left|  \avsum_{x\in\X_n}  \mathcal{I}_{\e_n,n}(x)   - \E\left[\avsum_{x\in\X_n}  \mathcal{I}_{\e_n,n}(x)\right] \right|  \geq C(1 + B) \frac{\e_n^2}{\ell} \Bigr) \\
\leq Cn\exp\left( - \frac{n \e_n^{d+4}B^2}{B^2 +\frac{\e_n^2}{\e_n}\|\nabla \overline{u}\|_{L^\infty(\M)}+\e_n^2\|D^2\overline{u}\|_{L^\infty(\M)}+\e_n^3\|D^3\overline{u}\|_{L^\infty(\M)}  }\right),
\label{eq:claim-1.1.}
\end{multline}
and
\begin{multline}
\P\Biggl[\left|\avsum_{x\in\X_n} g(x)\biggl(\Delta_\rho \overline{u}(x) - \mathcal{A}_{\e_n,n} \overline{u}(x) \biggr) - \E\left[  \avsum_{x\in\X_n} g(x)\biggl(\Delta_\rho \overline{u}(x) - \mathcal{A}_{\e_n,n} \overline{u}(x) \biggr) \right]\right|  >C(1+ B)\frac{\e_n^2}{\ell}  \biggr]\\
\le  Cn \exp \biggl( - \frac{cn\e_n^{d+4}B^2}{B^2 +\frac{\e_n^2}{\e_n}\|\nabla \overline{u}\|_{L^\infty(\M)}+\e_n^2\|D^2\overline{u}\|_{L^\infty(\M)}+\e_n^3\|D^3\overline{u}\|_{L^\infty(\M)}  }\biggr)\,.
\label{eq:claim-1.2.}
\end{multline}	

We thus focus on establishing \eqref{eq:claim-1.1.} and \eqref{eq:claim-1.2.}. We use the sets
\[ \partial \cu_{\e_n}: = \{  x\in \M \: \text{s.t.} \: \text{dist}(x,\partial \cu) \leq \e_n\}, \quad \cu_{\e_n}^\circ := \{  x\in \cu \: \text{s.t.} \: \text{dist}(x,\partial \cu) < \e_n\}, \]
where $\text{dist}$ denotes the (geodesic) distance from a point to a set. Note that if $x \in \M \setminus (\partial \cu_{\e_n} \cup \cu_{\e_n}^\circ)$, then $g(x)=0$ and for any $y \in \M$ with $d(x,y) \leq \e_n$ we have $ g(y)= g(x)=0$. On the other hand, if $x \in \cu_{\e_n}^\circ$, then $g(x)= \frac{1}{\rho(\cu)}$ and for any $y$ with $d(x,y)\leq \e_n$ we have $g(y)=g(x)$.

\medskip

\noindent \textbf{Proof of \eqref{eq:claim-1.1.}.} 
Observe that 
\begin{align*}
\avsum_{x\in \X_n} \mathcal{I}_{\e_n, n}(x) - \E\left[\avsum_{x\in \x_n} \mathcal{I}_{\e_n, n}(x)\right]  & =  \avsum_{x\in \x_n} (  \mathcal{I}_{\e_n, n}(x) - \E[\mathcal{I}_{\e_n, n}(x) \mid x ])
\\& \quad  + \avsum_{x\in \x_n} (   \E[\mathcal{I}_{\e_n, n}(x) \mid x ]  -  \E[\mathcal{I}_{\e_n, n}(x) ]   ).
\end{align*}
This can be interpreted as the decomposition of the U-statistic 
\[\avsum_{x\in \x_n} \mathcal{I}_{\e_n, n}(x) - \E\left[\avsum_{x\in \x_n} \mathcal{I}_{\e_n, n}(x)\right] \]
as a sum of a canonical U-statistic of order two and a sum of i.i.d. random variables (see, e.g., \cite{GineUStat}). We derive concentration bounds for each of these two terms.

For the sum of i.i.d.s, i.e., the second in the above decomposition, it suffices to apply Bernstein's inequality. We thus need to bound the random variable $ \E[\mathcal{I}_{\e_n, n}(x) \mid x ] $  and its variance. Now, 
\begin{multline*}
    \E[\mathcal{I}_{\e_n, n}(x) \mid x ]  \\
{=}\frac{1}{\sigma_\eta \e_n^2} \int_{B_1(0)\subseteq T_x \M} \eta(|v|) ( g(x) \nabla \overline u(x) \cdot (\e_nv) + g(\exp_x(\e_n v)) \nabla \overline u(\exp(\e_n v)) \cdot \log_{\exp(\e_n v )}(x)  ) J_x(\e_nv) \dd v,
\end{multline*}
from where it follows that
\[  | \E[\mathcal{I}_{\e_n, n}(x) \mid x ] | \leq \frac{C\|\nabla \overline{u}\|_{L^\infty(\M)}}{\e_n \rho(\cu)} \leq \frac{C\|\nabla \overline{u}\|_{L^\infty(\M)}}{\e_n \ell^d}. \]
On the other hand, 
\begin{align*}
&\E [  (\E[ \mathcal{I}_{\e_n, n} (x) |x  ] )^2 ]   \leq C  \int_{\M} \left(  \frac{1}{\e_n} \int_{B_1(0) \subseteq T_x \M}  (g(x) - g(\exp_x(\e_n v) )) \nabla \overline{u}(x) \cdot v   \dd v    \right)^2  \dx    
\\ &  + C \int_{\M} \left( \frac{1}{ \e_n^2} \int_{B_1(0)\subseteq T_x \M} g(\exp_x(\e_n v)) (  \nabla \overline u(x) \cdot (\e_nv) +  \nabla \overline u(\exp(\e_n v)) \cdot \log_{\exp(\e_n v )}(x)  )  \dd v \right)^2  \dx
\\& = C  \int_{\partial \cu_{\e_n} } \left(  \frac{1}{\e_n} \int_{B_1(0) \subseteq T_x \M}  (g(x) - g(\exp_x(\e_n v) )) \nabla \overline{u}(x) \cdot v   \dd v    \right)^2  \dx    
\\ &  + C \int_{\cu_{\e_n}^\circ \cup \partial \cu_{\e_n}} \left( \frac{1}{ \e_n^2} \int_{B_1(0)\subseteq T_x \M} g(\exp_x(\e_n v)) (  \nabla \overline u(x) \cdot (\e_nv) +  \nabla \overline u(\exp(\e_n v)) \cdot \log_{\exp(\e_n v )}(x)  )  \dd v \right)^2 \dx
\\ & \leq \frac{C B^2 \e_n \ell^{d-1} }{\e_n^2 \ell^{2d}} + \frac{C B^2  }{\ell^{d}}
\\& \leq \frac{C B^2 }{\e_n \ell^{d+1}},
\end{align*}
where in the second to last inequality we used \eqref{eq:geodesic gamma computation}. Bernstein's inequality then implies that
\[  \Prob\left[  \left|  \avsum_{x\in \x_n}(  \E[\mathcal{I}_{\e_n, n}(x) \mid x ]  -  \E[\mathcal{I}_{\e_n, n}(x) ] )  \right| \geq Ct \right] \leq 2\exp\left(- \frac{cn \e_n \ell^{d+1}t^2}{ CB^2 + t C\|\nabla \overline{u}\|_{L^\infty(\M)}  } \right), \]
which, taking $t =C \frac{\e_n^2}{\ell}B $, gives
\begin{multline*}
    \Prob\left[  \left|  \avsum_{x\in \x_n}(  \E[\mathcal{I}_{\e_n, n}(x) \mid x ]  -  \E[\mathcal{I}_{\e_n, n}(x) ] )  \right| \geq C \frac{\e_n^2}{\ell} B \right] \\
    \leq 2\exp\left(\frac{- cn \e_n^5 \ell^{d-1} B^2  }{B^2+ \frac{\e_n^2}{l} CB\|\nabla \overline{u}\|_{L^\infty(\M)} } \right) \leq 2\exp\left(\frac{- cn \ell^{d+4} B  }{B+ \frac{\e_n^2}{l} C\|\nabla \overline{u}\|_{L^\infty(\M)} } \right),
\end{multline*}
where the last inequality follows from the fact that $\ell \geq \e_n$.

Next, we analyze the canonical U-statistic of order two, which we rewrite as
\begin{align*}
\avsum_{x\in\X_n} &  (\mathcal{I}_{\e_n, n}(x) - \E[\mathcal{I}_{\e_n, n}(x) \mid x ]) \\
= & \frac{1}{n}\sum_{x\in\X_n\cap \partial \cu_{\e_n}  }  (\mathcal{I}_{\e_n, n}(x) - \E[\mathcal{I}_{\e_n, n}(x) \mid x ]) +\frac{1}{n}\sum_{x\in\X_n\cap \cu_{\e_n}^\circ}  (\mathcal{I}_{\e_n, n}(x) - \E[\mathcal{I}_{\e_n, n}(x) \mid x ])
\\ & = \frac{|\X_n \cap \partial \cu_{\e_n} |}{n \rho(\cu)} \avsum_{x\in\X_n\cap \partial \cu_{\e_n}  } \rho(\cu) (\mathcal{I}_{\e_n, n}(x) - \E[\mathcal{I}_{\e_n, n}(x) \mid x ])
\\ & \quad + \frac{|\X_n \cap  \cu_{\e_n}^\circ |}{n \rho(\cu)} \avsum_{x\in\X_n\cap  \cu_{\e_n}^\circ  }  \rho(\cu) (\mathcal{I}_{\e_n, n}(x) - \E[\mathcal{I}_{\e_n, n}(x) \mid x ]),
\end{align*}
and we seek to bound each of the latter terms. To obtain concentration bounds for the boundary term
\[ \frac{|\X_n \cap \partial \cu_{\e_n} |}{n \rho(\cu)} \avsum_{x\in\X_n\cap \partial \cu_{\e_n}  } \rho(\cu) (\mathcal{I}_{\e_n, n}(x) - \E[\mathcal{I}_{\e_n, n}(x) \mid x ]), \]
we start by noticing that, thanks to Bernstein's inequality, we have
\[ \frac{|\X_n \cap \partial \cu_{\e_n}|}{n \rho(\cu)} \leq C\frac{\e_n}{\ell}  \]
with probability at least $1- 2\exp(- cn \ell^{d-1} \e_n)$. On the other hand, a direct use of Bernstein's inequality and a union bound allow us to deduce that, with probability at least $1- Cn \exp(- cn \e_n^{d+4})$, 
\[\rho(\cu) \left| \mathcal{I}_{\e_n, n}(x) - \E[\mathcal{I}_{\e_n, n}(x) \mid x  ] \right| \leq C \e_n \sup_{y\in B_\M(x,\e_n)} |D\overline{u}(y)| \]
for all $x \in \X_n$. Putting together the above estimates, we deduce that
\[ \left| \frac{|\X_n \cap \partial \cu_{\e_n} |}{n \rho(\cu)} \avsum_{x\in\X_n\cap \partial \cu_{\e_n}  } \rho(\cu) \left(\mathcal{I}_{\e_n, n}(x) - \E[\mathcal{I}_{\e_n, n}(x) \mid x  ]\right)  \right| \leq C \frac{\e_n^2}{\ell}B,\]
with probability at least $1- Cn \exp\left(\frac{- cn \ell^{d+4} B  }{B+ \frac{\e_n^2}{l} \|\nabla \overline{u}\|_{L^\infty(\M)} } \right)$.

It remains to study the interior term 
\[ \frac{|\X_n \cap  \cu_{\e_n}^\circ |}{n \rho(\cu)} \avsum_{x\in\X_n\cap  \cu_{\e_n}^\circ  }  \rho(\cu) (\mathcal{I}_{\e_n, n}(x) - \E[\mathcal{I}_{\e_n, n}(x) \mid x ]).  \]
For this term, the key observation is that, for a given $x \in \cu_{\e_n}^\circ$, we have
\[ \rho(\cu) \mathcal{I}_{\e_n, n } (x) = \frac{1}{\sigma_\eta \e_n^2} \avsum_{y\in\X_n }\eta_{\e_n}(d(x,y))(\nabla \overline{u}(x)\cdot \log_x(y) + \nabla \overline{u}(y)\cdot \log_y(x)).   \]
In turn, \eqref{eq:geodesic gamma computation} implies that $\frac{1}{\e_n^2}(\nabla \overline{u}(x)\cdot \log_x(y) + \nabla \overline{u}(y)\cdot \log_y(x))$ is uniformly bounded by $C\|D^2\overline{u}\|_{L^\infty(\M)}$. A similar computation as above can see the $L^2$ norm of $\frac{1}{\e_n^2}(\nabla \overline{u}(x)\cdot \log_x(y) + \nabla \overline{u}(y)\cdot \log_y(x))$ is bounded by $CB^2.$ Due to this, Bernstein's inequality and a union bound imply that, with probability at least $1- Cn \exp(- \frac{n \e_n^d t^2}{B^2 + C\|D^2\overline{u}\|_{L^\infty(\M)}t } ) $,
\[ \left| \rho(\cu)( \mathcal{I}_{\e_n, n}(x)  -  \E[ \mathcal{I}_{\e_n, n}(x) \mid x ]) \right| \leq t,  \]
for all $\x \in \X_n \cap \cu_{\e_n}^\circ$; note that this is a much better estimate than for points in $\partial \cu_{\e_n}$, where, instead, we had the advantage of having fewer terms. Taking $t= C\frac{\e_n^2}{\ell} B$, the above implies
 \[  \left|  \frac{|\X_n \cap  \cu_{\e_n}^\circ |}{n \rho(\cu)} \avsum_{x\in\X_n\cap  \cu_{\e_n}^\circ  } \rho(\cu) (\mathcal{I}_{\e_n, n}(x) - \E[\mathcal{I}_{\e_n, n}(x) \mid x ]) \right| \leq C\frac{\e_n^2}{\ell} B, \]
with probability at least $1- Cn \exp\left(- \frac{n \e_n^{d+4} \ell^{-2}B^2}{B^2+C\e^2_n B\|D^2 \overline{u}\|_{L^\infty(\M)}} \right)\ge 1- Cn \exp\left(- \frac{n \e_n^{d+4} B}{B+C\e^2_n \|D^2 \overline{u}\|_{L^\infty(\M)}} \right)$, since $\e_n\leq \ell\leq c$.

Inequality \eqref{eq:claim-1.1.} is a consequence of the above estimates. 
\medskip 
	
\noindent \textbf{Proof of \eqref{eq:claim-1.2.}.} In order to prove \eqref{eq:claim-1.2.}, we consider a similar decomposition as before:
		\begin{align*}
		\avsum_{x\in\X_n} g(x) & \biggl(\Delta_\rho \overline{u}(x) - \mathcal{A}_{\e_n,n} \overline{u}(x) \biggr) - \E\left[  \avsum_{x\in\X_n} g(x)\biggl(\Delta_\rho \overline{u}(x) - \mathcal{A}_{\e_n,n} \overline{u}(x) \biggr) \right] 
		\\&   = \avsum_{x \in \X_n}  g(x) (\Delta_\rho \overline{u}(x) - \E[  \mathcal{A}_{\e_n, n} \overline{u}(x)  | x]  )  - \E[g(x)(\Delta_\rho \overline{u}(x) - \mathcal{A}_{\e_n, n} \overline{u}(x) )]  
		\\& \quad + \avsum_{x \in \X_n} g(x) ( \E[  \mathcal{A}_{\e_n, n} \overline{u}(x)  | x]  -  \mathcal{A}_{\e_n, n} \overline{u}(x)  ).
		\end{align*}
	 It will suffice to show that
		\begin{align}\label{eq:bias}
		\begin{split}
		\P\left[ \left| \avsum_{x\in\X_n} g(x)  \Bigl(\Delta_\rho \overline{u}(x) -  \E[\mathcal{A}_{\e_n,n} \overline{u}(x)\mid x]\Bigr)  - \E[g(x)(\Delta_\rho \overline{u}(x) - \mathcal{A}_{\e_n, n} \overline{u}(x) )] \right| >C\frac{\e_n^2}{\ell} B\right]
		\\  \le  2\exp \Biggl( \frac{- cn\ell^{d-2} \e_n^4B}{B+\e_n^2\|D^2\overline{u}\|_{L^\infty(\M)}+\e_n^3\|D^3\overline{u}\|_{L^\infty(\M)}}  \Biggr)\,,
		\end{split}
		\end{align}
		and
		\begin{multline}\label{eq:random}
		\P\left[ \left| \avsum_{x\in\X_n} g(x) \Bigl(\E[\mathcal{A}_{\e_n,n} \overline{u}(x)\mid x] - \mathcal{A}_{\e_n,n} \overline{u}(x) \Bigr) \right| >C\frac{\e_n^2}{\ell}B\right]\\
        \le  2n \exp \Biggl( \frac{- cn\e_n^{d+4}\ell^{-2} B}{B+\e_n^2\|D^2\overline{u}\|_{L^\infty(\M)}+\e_n^3\|D^3\overline{u}\|_{L^\infty(\M)}}  \Biggr)\,.
		\end{multline}

To prove \eqref{eq:bias}, we use Bernstein's inequality. Since in this case $|g(x)|\leq \frac{1}{\rho(\cu)}\leq \frac{C}{l^d}$, and $\Delta_{\rho} \overline{u}(x)$ is uniformly bounded, it suffices to find a uniform bound for $\E[\mathcal{A}_{\e_n,n} \overline{u}(x)\mid x]$. In turn, since
		\begin{equation} \label{e.taylorexpansion.b2}
		\begin{aligned}
		&\E[\mathcal{A}_{\e_n,n} \overline{u}(x)\mid x]   = -\E\Biggl[\frac{2}{\sigma_\eta \e_n^2} \avsum_{y \in \X_n }\eta_{\e_n}(d(x,y))\biggl\{ \frac12 \langle D^2 \overline{u}(x) \log_x(y), \log_x(y) \rangle  + \mathcal{R}(x,y)  \biggr\}\Bigg\lvert \, x\Biggr] 
        \\  & \qquad \qquad \qquad \qquad  + \E[\mathcal{R}_1(x)|x], \\
		\end{aligned}
		\end{equation}
		it is straightforward to show that $\E[\mathcal{A}_{\e_n,n} \overline{u}(x)\mid x]$ is uniformly bounded by $C\|D^2\overline{u}\|_{L^\infty(\M)}+C\e_n\|D^3\overline{u}\|_{L^\infty(\M)}$ and the $L^2$ norm is bounded by $CB^2$. From the above discussion, Bernstein's inequality implies
		\begin{align*}
		\begin{split}
		\P\left[ \left| \avsum_{x\in\X_n} g(x)  \Bigl(\Delta_\rho \overline{u}(x) -  \E[\mathcal{A}_{\e_n,n} \overline{u}(x)\mid x]\Bigr)  - \E[g(x)(\Delta_\rho \overline{u}(x) - \mathcal{A}_{\e_n, n} \overline{u}(x) )] \right| \geq t \right]
		\\  \le  2\exp \biggl( - \frac{cn\ell^d t^2 }{B^2+t \|D^2 \overline{u}\|_{L^\infty(\M)}}\biggr)\,.
		\end{split}
		\end{align*}
		Taking $t = C\frac{\e_n^2}{\ell}B$, we obtain \eqref{eq:bias}.

To conclude, we establish \eqref{eq:random}. Observe that
\[ \avsum_{x \in \X_n} g(x) \left( \E[ \mathcal{A}_{\e_n, n} \overline{u}(x)  \mid x ]  - \mathcal{A}_{\e_n, n} \overline{u}(x) \right) = \frac{|\cu|}{n \rho(\cu)} \avsum_{x \in \cu} \left( \E[ \mathcal{A}_{\e_n, n} \overline{u}(x)  \mid x ]  - \mathcal{A}_{\e_n, n} \overline{u}(x) \right).\]
 A standard concentration bound allows us to show that with probability at least $1- \exp(-cn\ell^d)$ we have 
 \[  \frac{|\cu|}{n \rho(\cu)} \leq C.  \]
On the other hand, using Bernstein's inequality and a union bound, it is straightforward to show that, with probability at least $1 - Cn \exp(- \frac{n \e_n^d t^2}{B^2+Ct\|D^2\overline{u}\|_{L^\infty(\M)}+Ct\e_n\|D^3\overline{u}\|_{L^\infty(\M)}})$, we have
\[| \E[ \mathcal{A}_{\e_n, n} \overline{u}(x)  \mid x ]  - \mathcal{A}_{\e_n, n} \overline{u}(x) | \leq t, \]
for all $x \in \X_n$. Taking $t = C \frac{\e_n^2}{\ell} B$, inequality \eqref{eq:random} now follows.

	\end{proof}

	\nc

	\subsection{Estimation of Solutions to Poisson Equation in the $ {\underline{H}^1(\X_n)}$ Semi-norm} \label{ss.poisson}

	\smallskip 	
	Let~$f \in C^{1,\alpha}(\M)$ (for some $\alpha>0$) be a fixed function such that 
	\begin{equation*}
	\int_{\M} f \rho\dx= 0\,. 
	\end{equation*}
	Our goal in this section is to compare the solution to the graph Poisson equation
	\begin{equation}
	\label{e.discrete}
	\mathcal{L}_{\e_n,n} u_{\e_n,n} = f - \avsum_{\X_n} f\,, \quad \mbox{ on } \X_n\,,
	\end{equation}
	with the solution of its continuum counterpart
	\begin{equation} \label{e.L0def}
	\Delta_\rho \overline{u}  =  	-\frac{1}{ \rho} \mathrm{div}\,(\rho^2 \nabla \overline{u}) = f\, \quad \mbox{ on } \M\,. 
	\end{equation}
	Solutions to these equations are only uniquely defined up to additive constants and so we normalize them to be mean zero.  Precisely, we require 
	\begin{equation*}
	\avsum_{\X_n} u_{\e_n,n} = 0, \quad \int_{\M} \overline{u}\rho \,dx  = 0\,. 
	\end{equation*}
   Elliptic regularity (see the formulation of $\Delta_\rho$ in normal coordinates in \eqref{eqn:LaplacianCoordinates}, Remark \ref{rem:RegCoeffciients}, and \cite[Corollary 2.29]{FernndezReal2022}) entails that~$\overline{u} \in C^{3,\alpha}(\M)$ and thus, in particular, $\overline{u}$ has bounded third order derivatives.

	To compare the solutions of \eqref{e.discrete} and \eqref{e.L0def}, we start by defining 
	\begin{equation} \label{e.wdef}
	w_{\e_n,n}  \defeq  u_{\e_n,n} - \overline{u}\,
	\end{equation}
	and observing that the function~$w_{\e_n,n}$ solves a discrete equation of the form 
	\begin{equation} \label{e.weqdef}
	\mathcal{L}_{\e_n,n} w_{\e_n,n} = h, \quad \mbox{ on } \X_n\,,
	\end{equation}
	for $h = \Delta_\rho \overline{u}(x) - \mathcal{L}_{\e_n,n} \overline{u}(x) - \avsum_{\X_n} f $. Our main result in this section is the following proposition. This result is not only useful because it requires the same main elements as for the proof of our main theorems in the next section, but it is also of interest in its own right. 
	
	\begin{proposition}
		\label{p.graphpoisson}
		Let~$u_{\e_n,n}$ and~$\overline{u}$ denote the unique mean zero solutions to~\eqref{e.discrete} and~\eqref{e.L0def}, respectively.  There exist constants~$C >1, 0<c<1 $ such that 
		\begin{multline}\label{eq:prop graph poisson}
		\P \biggl[ \biggl\{\frac1{\e_n^2}\avsum_{x,y \in \X_n} \eta_{\e_n}(|x-y|)\Bigl( w_{\e_n,n} (x) - w_{\e_n,n}(y) \Bigr)^2\biggr\}^{\frac12} \geqslant  CB \log(1/\e_n) \eps_n^2\biggr] 		\\
		\leqslant  	C n\e_n^{-d}\exp\left(\frac{-cn\e_n^{d+4}B}{B + \e_n \lVert \nabla f \rVert_{L^\infty(\M)} +\e_n^2\lVert D^2 f\rVert_{L^\infty(\M)}+\e_n^3\lVert D^3 f \rVert_{L^\infty(\M)}}\right)\,,
		\end{multline}
		where~$w_{\e_n,n}  \defeq  u_{\e_n,n} - \overline{u},$ and $B$ is as in \eqref{e.Bdef}.	
	\end{proposition}

	\begin{proof}
		In the proof, we derive an energy inequality that quantifies how close the solution to the continuum problem is to solving the discrete equation. 
		Multiplying both sides of \eqref{e.weqdef} by~$w_{\e_n,n}$, and summing over~$x \in \X_n,$ we obtain 
		\begin{multline}\label{eq:decomposition}
		\avsum_{\X_n} w_{\e_n,n} \mathcal{L}_{\e_n,n} w_{\e_n,n} 
		= \frac1{\sigma_\eta \e_n^2} \avsum_{x,y \in \X_n} \eta_{\e_n}(|x-y|) \bigl( w_{\e_n,n}(x) - w_{\e_n,n}(y)\bigr)^2\\
		= \avsum_{x\in \X_n} h(x) w_{\e_n,n}(x) = \avsum_{x \in \X_n} \biggl(\Delta_\rho \overline{u}(x) - \mathcal{L}_{\e_n,n} \overline{u}(x)\biggr)  w_{\e_n,n}(x)  - \Biggl(\avsum_{\X_n} f\Biggr)\Biggl( \avsum_{\X_n} w_{\e_n,n}
		\Biggr)
		=: \mathfrak{A} + \mathfrak{B}\,. 
		\end{multline}
		We estimate each of the terms~$\mathfrak{A}$ and~$\mathfrak{B}$ one after the other. 
		Let us begin with~$\mathfrak{B}.$ Since we chose~$u_{\e_n,n}$ to have mean zero, it follows that $ \avsum_{\X_n} w_{\e_n, n} = - \avsum_{\X_n} \overline{u}$.  Consequently, Bernstein's inequality implies that for any~$t > 0$ 
		\begin{multline} \label{e.averageueps}
		\P \biggl[ \Bigl| \avsum_{x\in\X_n} w_{\e_n,n} (x) \Bigr| > t\biggr] = \P \biggl[ \Bigl| \avsum_{x\in\X_n} \overline{u}(x)  \Bigr| > t\biggr] \leqslant 2 \exp \Bigl( - \frac{nt^2}{\|\overline{u}\|_{L^2(\M)}^2+t\|\overline {u}\|_{L^\infty(\M)}}\Bigr) \\
		\leqslant 2 \exp \Bigl( - \frac{cnt^2}{\|f\|_{L^2(\M)}^2+t\|\overline {u}\|_{L^\infty(\M)}}\Bigr) \,,
		\end{multline}
		where the last inequality follows from the fact that, by Poincaré inequality (i.e., Lemma \ref{lemma:poincare inequality}),
		\begin{equation} \label{e.basicenergy}
		\int_{\M} \overline{u}^2\dx \leqslant C\int_{\M} |\nabla \overline{u}|^2\dx  \leqslant C \int_{\M} f \overline{u}\dx \leqslant C\|f\|_{L^2(\M)}\|\overline{u}\|_{L^2(\M)}\,. 
		\end{equation}
		Using again Bernstein's inequality, we obtain
		\begin{equation} \label{e.discavgoff}
		\P \Bigl[ \Bigl|\avsum_{ x \in \X_n} f(x)\Bigr| > t\Bigr] \leqslant 2 \exp \Bigl(  - \frac{cnt^2}{\|f\|_{L^2(\M)}^2+t\|f\|_{L^\infty(\M)}}\Bigr)\,,\quad \forall t >0. 
		\end{equation}
		Combining \eqref{e.averageueps} and \eqref{e.discavgoff}, 
		\begin{multline}\label{eq:3-3}
		\P\left[\left|\Biggl(\avsum_{\X_n} f(x)\Biggr)\Biggl( \avsum_{\X_n} w_{\e_n,n}\Biggr)\right|>\left(t+C\e_n^2\right)^2\right]\le \P\left[\left|\avsum_{\X_n} f(x)\right|^2 +\left|  \avsum_{\X_n} w_{\e_n,n}\right|^2>2\left(t+C\e_n^2\right)^2\right]\\
		\le \P\left[\left|\avsum_{\X_n} f(x)\right|^2 >\left(t+C\e_n^2\right)^2\right]+\P\left[\left|  \avsum_{\X_n} w_{\e_n,n}\right|^2>\left(t+C\e_n^2\right)^2\right] \\
		\le \P\left[\left|\avsum_{\X_n} f(x)\right| >t+C\e_n^2\right]+\P\left[\left|  \avsum_{\X_n} w_{\e_n,n}\right|>t+C\e_n^2\right]\\
		\le C\exp\left(-\frac{cnt^2}{\|f\|_{L_2(\M)}^2+t\|\overline{u}\|_{L^\infty(\M)}+t\|f\|_{L^\infty(\M)}}\right).
		\end{multline}
		Choosing $t=CB\e_n^2$ concludes the estimate for~$\mathfrak{B}.$ 
		
		\medskip
		
		We turn to the term~$\mathfrak{A}.$ For this term, observe that
		\begin{align*}
			\avsum_{x\in\X_n} \Bigl( \Delta_\rho \overline{u}(x) - \mathcal{L}_{\e_n,n} \overline{u}(x)\Bigr) w_{\e_n,n}(x)
	 &	\leq \|\Delta_\rho \overline{u} - \mathcal{L}_{\e_n,n} \overline{u}\|_{\underline{H}^{-1}(\X_n)}\|w_{\e_n,n}\|_{\underline{H}^1(\X_n)}\,.
	 \\& \quad + 	\left(\avsum_{x\in\X_n} (\Delta_\rho \overline{u}(x) - \mathcal{L}_{\e_n,n} \overline{u}(x)) \right) \left( \avsum_{x \in \X_n}  w_{\e_n,n}(x) \right)
	 \\  &	= \|\Delta_\rho \overline{u} - \mathcal{L}_{\e_n,n} \overline{u}\|_{\underline{H}^{-1}(\X_n)}\|w_{\e_n,n}\|_{\underline{H}^1(\X_n)}\,
	 \\& \quad + 	\left(\avsum_{x\in\X_n}  f(x)  \right) \left( \avsum_{x \in \X_n}  w_{\e_n,n}(x) \right).
		\end{align*}
		Since the second term on the right hand side of the above expression is equal to $\mathfrak{B}$ (which we have already bounded), it suffices to prove that 
		\begin{multline}\label{eq:claim H-1 norm of L0-Ln}
		\P\left[ \|\Delta_\rho \overline{u} - \mathcal{L}_{\e_n,n} \overline{u}\|_{\underline{H}^{-1}(\X_n)}\geq C B \log(1/\e_n) \e_n^2\right]\\
        \leq C n\e_n^{-d}\exp\left(\frac{-cn\e_n^{d+4}B}{B + \e_n \lVert \nabla f \rVert_{L^\infty(\M)} +\e_n^2\lVert D^2 f\rVert_{L^\infty(\M)}+\e_n^3\lVert D^3 f \rVert_{L^\infty(\M)}}\right).
		\end{multline}
	At this stage, we seek to use the multiscale Poincaré inequality (Proposition \ref{l.msp}) and Proposition \ref{lem:pointwise convergence}. First, note that
		\[   \avsum_{ x \in \cu_{p,j,\upsilon}^m \cap \X_n} (\Delta_\rho \overline{u}(x) - \mathcal{L}_{\e_n,n} \overline{u}(x)) =   \frac{n \rho(\cu_{p,j,\upsilon}^m ) }{| \cu_{p,j,\upsilon}^m \cap \X_n|} \avsum_{ x \in \X_n}  \frac{\mathbf{1}_{\cu_{p,j,\upsilon}^m}}{\rho(\cu_{p,j,\upsilon}^m)} ( \Delta_\rho \overline{u}(x) - \mathcal{L}_{\e_n,n} \overline{u}(x) ). \]
		The factor $\frac{n \rho(\cu_{p,j,\upsilon}^m ) }{| \cu_{p,j,\upsilon}^m \cap \X_n|} $ can be controlled by a constant $C$ with probability at least $1 - C\e_n^{-d}\exp(-cn\e_n^d)$ by a standard application of Bernstein's inequality. On the other hand, from Proposition \ref{lem:pointwise convergence},    
		\[ \left| \avsum_{ x \in \X_n}  \frac{\mathbf{1}_{\cu_{p,j,\upsilon}^m}}{\rho(\cu_{p,j,\upsilon}^m)} ( \Delta_\rho \overline{u}(x) - \mathcal{L}_{\e_n,n} \overline{u}(x) ) \right| \leq C\frac{\e_n^2}{3^{p-m}}B ,\]
		with probability at least $1 - Cn\exp \Biggl(- \frac{cn\e_n^{d+4}B}{B +\frac{\e_n^2}{\e_n}\|\nabla \overline{u}\|_{L^\infty(\M)}+\e_n^2\|D^2\overline{u}\|_{L^\infty(\M)}+\e_n^3\|D^3\overline{u}\|_{L^\infty(\M)}  } \Biggr)$. Using a union bound, we conclude that
		\[ \lVert \Delta_\rho \overline{u}  - \mathcal{L}_{\e_n, n} \overline u \rVert_{\underline{H}^{-1}(\X_n)} \leq C \sum_{p=1}^m 3^{p-m} \left( \frac{\e_n^2}{3^{p-m}} B \right) = CB m \e_n^2 \leq  C B \log(1/\e_n) \e_n^2,   \]		
		with probability at least $1-Cn\sum_{p=1}^m3^{m-p} \exp \Biggl(- \frac{cn\e_n^{d+4}B}{B +\frac{\e_n^2}{\e_n}\|\nabla \overline{u}\|_{L^\infty(\M)}+\e_n^2\|D^2\overline{u}\|_{L^\infty(\M)}+\e_n^3\|D^3\overline{u}\|_{L^\infty(\M)}  } \Biggr) \geq 1-Cn\e_n^{-d} \exp \Biggl(- \frac{cn\e_n^{d+4}B}{B +\frac{\e_n^2}{\e_n}\|\nabla \overline{u}\|_{L^\infty(\M)}+\e_n^2\|D^2\overline{u}\|_{L^\infty(\M)}+\e_n^3\|D^3\overline{u}\|_{L^\infty(\M)}  } \Biggr)  $.
		This completes the proof of \eqref{eq:claim H-1 norm of L0-Ln}. Returning to \eqref{eq:decomposition}, we deduce that with very high probability
		\[\lVert w_{\e_n, n}  \rVert_{\underline{H}^1(\X_n)}^2  \leq a_1 \lVert w_{\e_n, n}  \rVert_{\underline{H}^1(\X_n)} + a_2,   \]  
		where $a_1 = C B \log(1/\e_n) \e_n^2$ and $a_2= CB \e_n^2$. From this we obtain the desired probabilistic bound on $\lVert w_{\e_n, n}  \rVert_{\underline{H}^1(\X_n)}$.
	\end{proof}

	\subsection{Proof of Theorem~\ref{t.upperbound}} \label{ss.final}

	We are ready to present the proof of Theorem \ref{t.upperbound}.

	\begin{proof}[Proof of Theorem \ref{t.upperbound}]
	The proof of~\eqref{e.quench} is a rather simple adaptation of Proposition~\ref{p.graphpoisson}, and various proofs of this are possible; see~\cite{trillos2019error,calder2019improved,armstrong2023optimal} for further details. At a high level, all of these proofs proceed by first obtaining a rate of convergence for the eigenvalues, and then using the equation to obtain rates of convergence for eigenfunctions. Here, we present a refinement of the proof strategy in~\cite{calder2019improved} using the estimates that we derived in earlier sections.
		
		\textbf{Step 1. Convergence rates for eigenvalues.} 
		The goal of this step is to show that
		\begin{equation}\label{e.eigenvaluerate}
		\P \Bigl[ \frac{|\lambda_{n,l} - \lambda_l|}{\lambda_l} > C\log(1/\e_n)\e_n^2\Bigr] \leqslant Cn \e_n^{-d} \exp \Biggl( - \frac{cn \e_n^{d+4}}{1+\e_n\sqrt{\lambda_l} \lambda_l^{\frac{d-1}{2}}}\nc\Biggr)\,. 
		\end{equation}
        if we show \eqref{e.eigenvaluerate}, then with the choice of $\e_n$ in \eqref{e.howhighup}, we obtain
        \begin{equation*}
		\P \Bigl[ \frac{|\lambda_{n,l} - \lambda_l|}{\lambda_l} > C\log(1/\e_n)\e_n^2\Bigr] \leqslant Cn \e_n^{-d} \exp \left( -cn \e_n^{d+4}\right)\,. 
		\end{equation*}
        \nc
		Recall that from \eqref{eq:EigenvalueEstimate} we have
		\begin{equation} \label{e.intermediatestep}
		|\lambda_{n,l} - \lambda_l|  \leqslant 
		\frac{1}{1 - \| f_l - \phi_{n,l}\|_{\underline{L}^2(\X_n)}} \biggl| \avsum_{x \in \X_n} \phi_{n,l}\bigl( \mathcal{L}_{\e_n,n} f_l - \Delta_\rho f_l \bigr) \biggr| \,.
		\end{equation}
		Appealing to~Proposition \ref{prop:f_n,f_0 angle}, for example, the first factor is no more than~$\frac12$ with very high probability. For the second term, we first note that
        \[ \biggl| \avsum_{x \in \X_n} \phi_{n,l}\bigl( \mathcal{L}_{\e_n,n} f_l - \Delta_\rho f_l \bigr) \biggr| \leq \lVert \phi_{n,l} \rVert_{\underline{H}^1(\X_n)} \lVert \mathcal{L}_{\e_n,n} f_l - \Delta_\rho f_l \rVert_{\underline{H}^{-1}(\X_n)}.   \]
       Now,
		\begin{equation*}
	\|\phi_{n,l}\|_{\underline{H}^1(\X_n)}\leq C \sqrt{\la_{n,l}} \leq C\sqrt{\lambda_l}\,,
		\end{equation*}
		with probability at least $1-Cn\exp(-cn\e_n^d)$ (see Proposition \ref{prop:la_n<2la_0}). For the other term, we can use \eqref{eq:claim H-1 norm of L0-Ln} (which, recall, uses the multiscale Poincaré inequality and Proposition \ref{lem:pointwise convergence}) applied to $f_l$ (which solves the Poisson equation $\Delta_\rho f_l = \lambda_l f_l$) to get 
        \begin{multline}\label{eq:claim H-1 norm fl}
		\P\left[ \|\Delta_\rho f_l - \mathcal{L}_{\e_n,n} f_l\|_{\underline{H}^{-1}(\X_n)}\geq C B \log(1/\e_n) \e_n^2\right]\\
        \leq Cn \e_n^{-d}\exp\left(-\frac{cn\e_n^{d+4}B}{B+\e_n \|D f_l\|_{L^\infty(\M)} + \e_n^2\| D^2 f_l\|_{L^\infty(\M)}+ \e_n^3\| D^3 f_l\|_{L^\infty(\M)}\nc}\right).
		\end{multline}
        It remains to estimate the constant~$B$ and $\| D f_l\|_{L^\infty(\M)}$ for an eigenfunction. The elliptic regularity for eigenfunctions implies immediately (for details, see~\cite{armstrong2023optimal}, Remark \ref{rem:RegCoeffciients} and \cite[Corollary 2.29]{FernndezReal2022} ) that 
		\begin{equation*}
	B \sim C\sqrt{\lambda_l}, \quad \| D f_l\|_{L^\infty(\M)}\leq \lambda_l^{\frac{d+1}{2}}, \quad  \| D^2 f_l\|_{L^\infty(\M)}\leq \lambda_l^{\frac{d+2}{2}},\quad  \| D^3 f_l\|_{L^\infty(\M)}\leq \lambda_l^{\frac{d+3}{2}}\nc \,.
		\end{equation*}
	Inequality \eqref{e.eigenvaluerate} follows. 
		
We note that the point of departure from the proof in~\cite{calder2019improved} is how we estimate the inner product in~\eqref{e.intermediatestep}. In that paper, one applies Cauchy-Schwarz to estimate each factor in~$\underline{L}^2(\X_n)$. This leads to losses in length scale (mainly from the linear term in the Taylor expansion of $\overline{u}$ in the analysis in section \ref{sec:ConcentrationBounds}) that are precluded when one instead measures the size of the right-hand side in~$\underline{H}^{-1}(\X_n),$ which is a more \emph{global} quantity. 
		
		\smallskip 
		\textbf{Step 2.} Having proven eigenvalue rates of convergence, to prove convergence rates for eigenfunctions we follow the proof of Proposition~\ref{p.graphpoisson} closely: one simply studies the discrete graph equation solved by the function~$w_{\e_n,n}  \defeq  \phi_{n,l} - f_l,$ and proceeds as in that argument. This function satisfies the equation
		\begin{equation*}
		\mathcal{L}_{\e_n,n} w_{\e_n,n} -\la_{n,l} w_{\e_n,n}  = - \bigl( \mathcal{L}_{\e_n,n} f_l- \Delta_\rho f_l \bigr)  - \bigl( \lambda_l - \la_{n,l}\bigr) f_l  =: h_l \,.
		\end{equation*}
		Multiplying this equation by~$w_{\e_n,n}$ and estimating the right-hand side in~$\underline{H}^{-1}(\X_n)$, one simply repeats the arguments in Lemma~\ref{lem:pointwise convergence} and Proposition~\ref{p.graphpoisson} to ultimately obtain
		\begin{multline*}
		\P\Bigl[\biggl(\frac1{\e_n^2} \avsum_{x,y \in \X_n} \eta_{\e_n}(|x-y|) (\phi_{n,l}(x) - \phi_{n,l}(y) - f_l(x) + f_l(y))^2\biggr)^{\sfrac12} \geqslant \frac{C \log(1/\eps_n)\eps_n^2\lambda_l}{\gamma_l}  \Bigr]		\\
		\leq \P\biggl[ \|h_l\|_{\underline{H}^{-1}(\X_n)}\geq C \log(1/\eps_n)\eps_n^2\sqrt{\lambda_l}\nc \biggr]
		+ \P \biggl[ \frac{|\la_{n,l} - \lambda_l|}{\lambda_l} \geqslant C \log(1/\eps_n)\eps_n^2 \biggr] \\
		\leqslant Cn \e_n^{-d} \exp \Biggl( - \frac{cn \e_n^{d+4}}{1+\e_n\sqrt{\lambda_l} \lambda_l^{\frac{d-1}{2}}\nc}\Biggr)\leq Cn \e_n^{-d} \exp \left( - cn \e_n^{d+4}\right), 
		\end{multline*}	
        where we use the parameter choice of $\e_n$ in terms of $\lambda_l$ in the final inequality.
        
		Finally, a similar bound on the~$\underline{L}^2(\X_n)$ norm of the difference~$\phi_{n,l} - f_l$ follows by an application of the (global) discrete Poincar\'e inequality (Proposition \ref{l.poincaregraph}). Combining with Step 1, this completes the proof of~\eqref{e.quench}. 
		
		
		\smallskip
		\textbf{Step 3.} Turning to the proof of~\eqref{e.anneal}, we decompose $\mathcal{E}_l(\X_n)$ into two parts as
		\begin{multline*}
		\E_{\X_n\sim\rho} \Bigl[ \mathcal{E}_l(\X_n) \Bigr] = \E_{\X_n\sim\rho} \Bigl[ \mathcal{E}_l(\X_n)\mathbbm{1}_{\mathcal{E}_l(\X_n)\leq C\log(1/\e_n)\e_n^2} \Bigr]\\
        +\E_{\X_n\sim\rho} \Bigl[ \mathcal{E}_l(\X_n)\mathbbm{1}_{\mathcal{E}_l(\X_n)> C\log(1/\e_n)\e_n^2} \Bigr] .
		\end{multline*}
		First,
		\begin{multline}\label{eq:El2>1 bound}
		\E_{\X_n\sim\rho} \Bigl[ \mathcal{E}_l(\X_n)\mathbbm{1}_{\mathcal{E}_l(\X_n)>C\log(1/\e_n)\e_n^2} \Bigr]  \leqslant \max_{\X_n} |\mathcal{E}_l(\X_n)|  \P(\mathcal{E}^2_l(\X_n)\geqslant C\log(1/\e_n)\e_n^2) \\
		\leqslant \frac{C}{\e_n^{d+2}}  \P\left(\mathcal{E}_l(\X_n)\geqslant C\log(1/\e_n)\e_n^2\right) 
		\leqslant  \frac{Cn\e_n^{-d}}{\e_n^{d+2}}\exp( - c n\e_n^{d+4} ) \,. 
		\end{multline}
		On the other hand,
		\begin{equation}\label{eq:El2 bound}
		\E_{\X_n\sim\rho} \Bigl[ \mathcal{E}_l(\X_n)\mathbbm{1}_{\mathcal{E}^2_l(\X_n)\leq C\log(1/\e_n)\e_n^2} \Bigr]\leq  C\log(1/\e_n)\e_n^2.
		\end{equation}
		By combining the above inequalities, we obtain
		\begin{equation}
		\E_{\X_n\sim\rho} \Bigl[ \mathcal{E}_l(\X_n) \Bigr]\leq C\log(1/\e_n)\e_n^2 + \frac{Cn}{\e_n^{2d+2}}\exp( - c n\e_n^{d+4} ) .
		\end{equation}
		Optimizing over~$\e_n >0$ with $\e_n\sim \left(\frac{n}{\log n}\right)^{-\frac{1}{d+4}}$ we deduce 
		\begin{eqnarray*}
			\E_{\X_n\sim\rho} \Bigl[ \mathcal{E}_l(\X_n) \Bigr]\leq C   \left(\frac{1}{n}\right)^{\frac{2}{d+4}}\frac{\log n}{\log\log n}.
		\end{eqnarray*}
		This completes the proof of the theorem.
	\end{proof}

	\subsection{Extension Results}\label{sec:proof of upperboundcontinuum}
	

	
	\subsubsection{Proof of Theorem \ref{thm:Data dependent construction}}

	We mimic the proof of \cite[Lemma 9]{trillos2019error}, observing that the extension operator $\Lambda_r$ introduced in \eqref{eq:ExtensionLambda} is defined as a convolution of the kernel $k_r$ with an empirical measure and not with the manifold's volume form as done in \cite{trillos2019error}. We decompose the proof into two steps. First, we show that with probability at least $1-C r^{-d}\exp(-cn r^d)$ we have
	\begin{equation}\label{eq:proof of nabla interpolation}
	\int_{\M} | \nabla \Lambda_ru|^2 \dx \leq  \frac{C}{r^{d+2}} \avsum_{x, y \in \X_n} \eta\left(\frac{|x-y|}{2r}\right) ( u(x) - u(y))^2 
	\end{equation}
	for every $u: \X_n \rightarrow \R$. 
	By taking $u=\phi_{n,l}-f_l$ and $r = \e_n/2$, we immediately deduce from the above inequality and the continuum Poincaré inequality (see Lemma \ref{lemma:poincare inequality}) that
	\begin{equation}
	\lVert  \Lambda_{\e_n/2} \phi_{n,l}- \Lambda_{\e_n/2} f_l   \rVert_{H^1(\M)} \leq  C \lVert \phi_{n,l} - f_l \rVert_{\underline{H}^1(\X_n)},
	\label{eq:AuxVarianceExtension}
	\end{equation}
	with probability at least $1-C \e_n^{-d}\exp(-cn \e_n^d)$. 
	
	In the second step, we ensure that, with probability at least $1- Cn\exp(- cr^d) - Cn\exp(-c r^{d} t^2)$,
	\begin{equation}
	\int_\M | \nabla \Lambda_r f_l(x) - \nabla f_l(x)   |^2 \dx\leq C (r^4+t^2).
	\label{eq:ExtensionBias}
	\end{equation}
	Combining \eqref{eq:AuxVarianceExtension}, \eqref{eq:ExtensionBias} with $r= \e_n/2$, and using the triangle inequality, we obtain, with probability at least $1-C\e_n^{-d}\exp(-cn\e_n^d)-Cn\exp(-c\e_n^{d} t^2)$, 
	\begin{equation*}
	\int_\M | \nabla \Lambda_{\e_n/2} \phi_{n,l}(x) - \nabla f_l(x)   |^2 \dx\leq C(\e_n^4+t^2) +C\|\phi_{n,l}-f_l\|_{\underline{H}^1(\X_n)}^2.
	\end{equation*}
	From the above and Poincaré inequality (to get bounds for the $L^2(\M)$-norm) we immediately deduce \eqref{e.mainestimatewithextension}.  It thus suffices to establish \eqref{eq:proof of nabla interpolation} and \eqref{eq:ExtensionBias}.

	\medskip

	Along the proofs of these two bounds, we use the following construction from \cite{trillos2019error}: with probability at least $1-C r^{-d}\exp(- cn r^d)$, we can find a density function ${\rho}_n$ and a map $T: \M \rightarrow \X_n$ such that
	\begin{enumerate}
		\item $c \rho(x) \leq {\rho}_n(x) \leq C \rho(x), \quad \forall x \in \M$.  
		\item The pushforward of $\rho_n $ by $T$ is equal to the empirical measure $ \frac{1}{n}\sum_{i=1}^n \delta_{x_i} $, and $\sup_{x\in\M} |x- T(x)| \leq r$.
	\end{enumerate}
	In the rest of the proof, we will implicitly assume that we are in the event where 1. and 2. hold.

	\textbf{Step 1:} In order to prove \eqref{eq:proof of nabla interpolation}, we start by introducing $\tilde{V}_1, \dots, \tilde{V}_n$, the partition of $\M$ defined by $\tilde{V}_i := T^{-1}(\{ x_i\})$, $i=1, \dots, n$. Relative to a given $u: \X_n \rightarrow \R$, we consider the function $\tilde{u}: \M \rightarrow \R$  defined as
	\[ \tilde{u}(x) := u(x_i) , \quad x \in \tilde{V}_i, \quad i=1, \dots, n.\]
	We note that $\tilde u$ is a piecewise constant function over $\M$. This function is only used in the analysis that follows and it is not an extension of $u$ that is useful in practice.
	
	From the definition of $k_r$ (see \eqref{eq:ExtensionKernel}) it follows that
	\begin{equation}\label{eq:nabla k_r}
	\nabla_x k_r(x, x_i ) = \frac{1}{r^{d+1}}\eta\left( \frac{|x-x_i|}{r} \right) \frac{P_x(x_i - x)}{|x-x_i|},
	\end{equation}
	where $\nabla_x$ is the gradient in $\M$ in the $x$ coordinate, and $P_x(x_i -x )$ represents the projection of the vector $x_i- x$ onto $T_x\M$, the tangent plane at $x$. We follow \cite[Lemma 9]{trillos2019error} and observe that, for any given $x \in \M$, we can write
	\[ \nabla \Lambda_r u(x) = \frac{1}{\theta(x)} A_1(x) + A_2(x),  \]
	where 
	\[ A_1(x) \defeq  \frac{1}{n} \sum_{i=1}^n ( u(x_i) - \tilde u (x))\nabla_x k_r(x,x_i)  \]
	and
	\[ A_2(x) \defeq  \nabla (\theta^{-1}(x)) \frac{1}{n} \sum_{i=1}^n ( u(x_i) - \tilde u (x)) k_r(x,x_i), \]
	as can be directly verified. In the above computation, we simply added and subtracted the terms involving $\tilde u (x)$. The specific details in the definition of $\tilde u (x)$ are, for the moment, irrelevant, and in particular we could have put any real number in place of $\tilde u (x)$ without changing the above identity. 
	
	Using Cauchy-Schwarz inequality we obtain
	\[ |A_1(x)|^2 \leq \left(   \frac{1}{n r^d} \sum_{i=1}^n \eta\left(\frac{|x- x_i|}{r}\right) \right) \left( \frac{1}{nr^{d+2}} \sum_{i=1}^n ( u(x_i) - \tilde u (x) )^2 \eta\left(\frac{|x- x_i|}{r}\right) \right).   \]
	By \eqref{e.comparegeodesic} and a standard concentration inequality (\cite[Corollary 3.7]{calder2022lipschitz}) we see that, with probability at least $1-C (tr)^{-d}\exp(-cnr^d t^2)$, 
	\begin{equation}\label{eq:uniform distribution control}
	\left|\frac{1}{n r^d} \sum_{i=1}^n \eta\left(\frac{|x- x_i|}{r}\right)-\rho(x)\right|\leq t+Cr^2, \text{ for all } x\in\M,
	\end{equation}
	where we recall $\eta$ was assumed to be normalized (i.e., \eqref{e.normalize} holds). Therefore, by choosing $t=C'>0$ for $C'$ small enough, we deduce that with probability at least $1-C r^{-d}\exp(-cn r^d)$ 
	\begin{eqnarray*}
		\frac{1}{n r^d} \sum_{i=1}^n \eta\left(\frac{|x- x_i|}{r}\right)\leq C, \text{ for all } x\in\M.
	\end{eqnarray*}
	Similarly, with probability at least $1-C r^{-d}\exp(-cn r^dt^2)$, for all $x \in \M$
	\begin{equation}\label{eq:theta}
	|\theta(x)- \tau_\psi \rho(x)|\leq C r^2+t,
	\end{equation}
	where the factor $\tau_\psi$ is given by
    \begin{equation}
       \tau_\psi:= \int \psi(|v|)\dd v. 
       \label{eqn:TauFactor}
    \end{equation}
    Indeed, note that $\theta$ is nothing but a kernel density estimator for $\rho$. In particular, with probability at least $1-C r^{-d}\exp(-cn r^d)$, 
	\begin{equation}\label{eq:theta2}
	C >\theta(x)> c, \text{ for all } x\in\M.
	\end{equation}
	We conclude from the above that, with probability at least $1-Cr^{-d}\exp(-cnr^d)$,
	\begin{equation}\label{A1}
	\left|\frac{1}{\theta(x)} A_1(x) \right|\leq \frac{C}{nr^{d+2}} \sum_{i=1}^n ( u(x_i) - \tilde u (x) )^2 \eta(|x- x_i|/r), \quad  \forall  x\in\M.
	\end{equation}

	Regarding $A_2(x)$, we have
	\begin{equation}\label{eq:theta union bound}
	| \nabla \theta (x)| \leq \frac{1}{r} \left( \frac{1}{n r^d} \sum_{i=1}^n\eta \left(\frac{|x_i-x|}{r}\right) \right)  \leq \frac{C}{r}, \text{ for all } x\in\M.
	\end{equation}
	From this, Cauchy-Schwartz inequality, and Assumption \ref{assump.Eta} it follows that 
	\begin{equation}\label{A2}
	|A_2 (x)|^2 \leq C \frac{1}{nr^{d+2}} \sum_{i=1}^n \eta(|x-x_i|/r) ( u(x) - \tilde u(x_i))^2, \text{ for all } x\in\M.
	\end{equation}
	Combining \eqref{A1} and \eqref{A2}, we conclude that
	\[ |\nabla \Lambda_r u(x)|^2 \leq C \frac{1}{n r^{d+2} } \sum_{i=1}^n \eta (|x- x_i|/{r}) ( \tilde u(x) - u(x_i))^2 , \quad \forall  x\in\M. \]
	
	Now, recall that for all $x \in \tilde{V}_j$ we have $\tilde u (x) = u(x_j)$. In addition, by the fact that $\eta$ was assumed to be Lipschitz it follows that for all $x \in \tilde{V_j}$ we have $\eta(| x-  x_i| /r ) \leq C \eta( |x_j - x_i| / 2r)$.  Hence
	\begin{align*}
	\int_{\tilde{V}_j} |\nabla \Lambda_r u(x)|^2  \rho(x ) \dx   &   \leq C  \int_{\tilde{V}_j} |\nabla \Lambda_r u(x)|^2  \rho_n(x ) \dx  
	\\ & \leq C \rho_n(\tilde{V_j} )  \frac{1}{n r^{d+2} } \sum_{i=1}^n ( u(x_i) - u(x_j))^2 \eta(|x_i - x_j|/2r)
	\\ & = C \frac{1}{ n^2 r^{d+2} } \sum_{i=1}^n ( u(x_j) - u(x_i))^2 \eta(|x_i - x_j|/2r).
	\end{align*}
	Summing over $j=1, \dots, n$, we get
	\[ \int_\M   |\nabla \Lambda_r u(x)|^2  \rho(x) \dx  \leq C  \frac{1}{n^2 r^{d+2}} \sum_{j=1}^n\sum_{i=1}^n  ( u(x_i) - u(x_j))^2 \eta(|x_i- x_j|/2r).  \]
	This finishes the proof of \eqref{eq:proof of nabla interpolation}.
	
	\textbf{Step 2:} To prove \eqref{eq:ExtensionBias}, we start by using a similar decomposition for $\nabla \Lambda_r f_l$ as in \textbf{Step 1} to obtain 
	\begin{align*}
	\begin{split}
	\int_\M  \Bigl| \nabla \Lambda_r f_l(x) - \nabla f_l(x)   |^2 &  \dx =\\
	&\int_\M \Bigl|\frac{1}{n\theta(x)}\sum_{j=1}^n (f_l(x_j)-f_l(x))\frac{1}{r^{d+1}}\eta\left(\frac{|x-x_j|}{r}\right)\frac{P_x(x_j-x)}{|x_j-x|}\\
	&\quad  \quad +\nabla (\theta^{-1}(x))\frac{1}{n}\sum_{j=1}^n (f_l(x_j)-f_l(x)) \frac{1}{r^d} \eta\left(\frac{|x-x_j|}{r}\right)-\nabla f_l(x)\Bigr|^2\dx.
	\end{split}
	\end{align*}
	In what follows, we estimate each of the terms
	\begin{equation}\label{term1 in extension}
	\left|\nabla (\theta^{-1}(x))\frac{1}{n}\sum_{j=1}^n (f_l(x_j)-f_l(x)) \frac{1}{r^d} \eta\left(\frac{|x-x_j|}{r}\right)\right|,
	\end{equation} 
	\begin{equation}\label{term2 in extension}
	\left|\frac{1}{n\theta(x)}\sum_{j=1}^n (f_l(x_j)-f_l(x))\frac{1}{r^{d+1}}\eta\left(\frac{|x-x_j|}{r}\right)\frac{P_x(x_j-x)}{|x_j-x|}-\nabla f_l(x)\right|,
	\end{equation}
	and show that they are small in the $L^2(\M)$-sense with very high probability.
	
	\textbf{Controlling \eqref{term1 in extension}.} First, observe that we can write
    \[ \left|\frac{1}{n}\sum_{j=1}^n (f_l(x_j)-f_l(x)) \frac{1}{r^d} \eta\left(\frac{|x-x_j|}{r}\right)\right| = r^2 | \mathcal{L}_{r, n} f_l(x)|.  \]
    From the pointwise consistency of graph Laplacians (see, for example, \cite[Theorem 3.3]{calder2019improved}), one can easily deduce that, with probability at least $1- Cn \exp(-cnr^{d+2})$,  
    \[ | \mathcal{L}_{r, n} f_l(x) | \leq C, \quad \forall x \in \M.\]
It thus remains to find a better bound for $|\nabla \theta^{-1}(x)|$ (e.g., an order one bound) than the one presented in \eqref{eq:theta union bound}. 
For this, recall that, thanks to \eqref{eq:nabla k_r}, we have
\[ \nabla \theta (x) = \frac{1}{n} \sum_{j=1}^n \frac{1}{r^{d+1}} \eta\left( \frac{|x-x_j|}{r} \right)\frac{P_x(x_j-x)}{|x-x_j|}. \]
Now, by \eqref{e.comparegeodesic} and \eqref{eqn:DifferenceGradients}, 
\[ \left| \nabla \theta (x) - \frac{1}{n} \sum_{j=1}^n \frac{1}{r^{d+1}} \eta\left( \frac{d(x,x_j)}{r} \right)\frac{\log_x(x_j)}{d(x,x_j)} \right| \leq C r, \]
in the same event where \eqref{eq:theta2} holds. In addition, for every fixed $x \in \M$ we have
\[ \E \left[ \frac{1}{n} \sum_{j=1}^n \frac{1}{r^{d+1}} \eta\left( \frac{d(x,x_j)}{r} \right)\frac{\log_x(x_j)}{d(x,x_j)}  \mid x \right] = \frac{1}{r} \int_{B_1(0) \subseteq T_x \M} \eta(|v|)J_x(rv)\rho(\exp_x(rv))\frac{v}{|v|} \dd v.   \]
The latter term, however, is of order one (i.e., a vector whose norm is bounded by a uniform constant) since, by symmetry, we have
\[ \int_{B_1(0) \subseteq T_x \M} \eta(|v|) \frac{v}{|v|}\dd v =0.   \]
On the other hand, Bernstein's inequality and a suitable union bound imply that, with probability at least $1- Cn \exp(- c n r^{d+2})$ we have
\[ \left|  \frac{1}{n} \sum_{j=1}^n \frac{1}{r^{d+1}} \eta\left( \frac{d(x,x_j)}{r} \right)\frac{\log_x(x_j)}{d(x,x_j)}         -\E \left[ \frac{1}{n} \sum_{j=1}^n \frac{1}{r^{d+1}} \eta\left( \frac{d(x,x_j)}{r} \right)\frac{\log_x(x_j)}{d(x,x_j)}  \mid x \right]  \right| \leq C  \]
for all $x \in \M$. 

Putting together all the above computations, and using also \eqref{eq:theta2}, we obtain that \eqref{term1 in extension} is uniformly bounded over all $x \in \M$ by $Cr^2$, with probability at least $1- Cn \exp(- cn r^{d+2})$.

	\medskip 
    
	\textbf{Controlling \eqref{term2 in extension}.} We start by noticing that, thanks to \eqref{e.comparegeodesic} and \eqref{eqn:DifferenceGradients}, for every $x \in \M$ we have 
	\begin{multline*}
	\Bigl|\frac{1}{n\theta(x)}\sum_{j=1}^n (f_l(x_j)-f_l(x))\frac{1}{r^{d+1}}\eta\left(\frac{|x-x_j|}{r}\right)\frac{P_x(x_j-x)}{|x_j-x|}\\
	-\frac{1}{n\theta(x)}\sum_{j=1}^n (f_l(x_j)-f_l(x))\frac{1}{r^{d+1}}\eta\left(\frac{d(x,x_j)}{r}\right)\frac{\log_x(x_j)}{d(x,x_j)}\Bigr|
	\leq Cr^2,
	\end{multline*}
	under the event where \eqref{eq:theta2} holds.
	
Now, 
\begin{align*}
   \E \Bigl[  \frac{1}{n}\sum_{j=1}^n & (f_l(x_j)-f_l(x))\frac{1}{r^{d+1}}  \eta\left(\frac{d(x,x_j)}{r}\right)\frac{\log_x(x_j)}{d(x,x_j)} \mid x \Bigr] 
   \\& = \int_{B_1(0) \subseteq T_x \M} \frac{f_l(\exp_x(rv))- f_l(x)}{r} \eta(|v|) \frac{v}{|v|} \rho(\exp_x(rv)) J_x(rv) \dd v.
\end{align*}
After a Taylor expansion of $\rho$ and $f_l$, the latter term is seen to be equal to
\[ \rho(x)\int \nabla f_l (x) \cdot v \eta(|v|)\frac{v}{|v|} \dd v  + O(r^2), \]
where by $O(r^2)$ we mean a vector whose norm is bounded by a constant times $r^2$. In the above, the order $r$ term is indeed zero given that, by symmetry, we have
\[ \int \nabla f_l (x) \cdot v \eta(|v|) \nabla \rho(x) \cdot v \frac{v}{|v|} \dd v =0  \]
and
\[ \int \langle D^2 f_l(x) v, v \rangle \eta(|v|) \frac{v}{|v|} \dd v =0. \]
Note that the $O(r^2)$ term is controlled by the derivatives up to order three of $f_l$ and derivatives up to order two of $\rho$. On the other hand, a direct computation reveals that  
\[  \rho(x)\int \nabla f_l (x) \cdot v \eta(|v|)\frac{v}{|v|} \dd v = \tau_\psi \rho(x)\nabla f_l(x), \]
where $\tau_\psi$ is the same constant as in \eqref{eqn:TauFactor}.

Finally, by Bernstein's inequality and a union bound, it is straightforward to see that
 \begin{align*}
  \Bigl| \frac{1}{n}\sum_{j=1}^n & (f_l(x_j)-f_l(x))\frac{1}{r^{d+1}}  \eta\left(\frac{d(x,x_j)}{r}\right)\frac{\log_x(x_j)}{d(x,x_j)} 
   \\& - \E \Bigl[  \frac{1}{n}\sum_{j=1}^n  (f_l(x_j)-f_l(x))\frac{1}{r^{d+1}}  \eta\left(\frac{d(x,x_j)}{r}\right)\frac{\log_x(x_j)}{d(x,x_j)} \mid x \Bigr]  \Bigr| \leq t
\end{align*}   
    for all $x \in \M$, with probability at least $1-Cn \exp(-cnr^{d} t^2)$. Putting all the above together, and using also \eqref{eq:theta}, we deduce that \eqref{term2 in extension} is uniformly bounded over all $x \in \M$ by $t + Cr^2$, with probability at least $1 - Cr^{-d} \exp(-cnr^d) - Cn \exp(-cnr^d t^2)$.

From our probabilistic bounds on \eqref{term1 in extension} and \eqref{term2 in extension} we deduce \eqref{eq:ExtensionBias} and with it we conclude the proof.

    \bibliographystyle{abbrv}
	\bibliography{references}

	\appendix
	

	\section{Background on Riemannian Geometry}
	\label{App:GeoBack}

	\subsection{Exponential Map and Normal Coordinates}
	\label{app:ExpMap}

	%

	Let $x \in \M$. The \textit{exponential map} $\exp_x$ at the point $x$ is the map $\exp_x : T_x \M \rightarrow \M$ with the property that, for every $v\in \T_x \M $, the curve $t \in \R_+ \mapsto \exp_x(t v) $ is the unique constant speed geodesic that starts at $x$ with initial velocity $v$.
	
	It turns out that for small enough $r>0$, the exponential map 
	\begin{equation} \label{expmap}
	{\rm exp}_{x}:  B_r(0) \subseteq T_{x}\mathcal{M} \to B_\M(x,r) \subseteq\mathcal{M}
	\end{equation}
	is a diffeomorphism between the~$d$-dimensional Euclidean ball $B_r(0)$ in the tangent space $T_{x}\mathcal{M}$ and the geodesic ball of radius $r$ centered at $x$. The \textit{injectivity radius} $r_0$ is the largest $r$ such that all the exponential maps $\{\exp_x\}_{x \in \M}$ are diffeomorphisms, as described above. For $r<r_0$, we can thus introduce the diffeomorphic inverse of $\exp_x$, the \textit{logarithmic map}
	\begin{equation} \label{logmap}
	{\rm log}_{x}:  B_\M(x,r)  \subseteq \M \to B_r(0)\subseteq T_{x}\mathcal{M}\,. 
	\end{equation}
	Given $y \in B_\M(x,r)$ (for $r < r_0$), $v=\log_x(y) \in T_x \M$ can be interpreted as the initial velocity of the minimizing geodesic that at time $t=0$ starts at $x$ and at time $1$ ends at $y$ —i.e., the curve $t\in [0,1] \mapsto \exp_x(tv)$. Moreover, we have the relation 
	\[ d(y, x ) = |v|, \]
	and ${\rm exp}_{x}(0) = x.$ 
	
	By \textit{normal coordinates} around a point $x \in \M$, we simply mean the parameterization of $B_\M(x,r)$ via the exponential map $\exp_x$. In the paper, we repeatedly consider integrals of functions $g$ supported on  $B_\M(x,r)$. In normal coordinates, these integrals can be written as
	\[ \int_{B_\M(x,r)} g(x) \dx =  \int_{ B_r(0) \subseteq T_x \M }  g(v) J_x(v) \dd v, \]
	where $J_x(\cdot)$ is the Jacobian of the exponential map, i.e.,
	\begin{equation*}
	J_x(v): = |\det D_v \bigl( \exp_x(v)\bigr)|\,.
	\end{equation*}
	When $|v|\leq \e_n$ for $\e_n$ satisfying \eqref{eq:assumption:eps small}, since~$D_v(\exp_x(0)) = I,$ it is well-known that the Jacobian admits a Taylor expansion about~$v = 0$ given by  
	\begin{equation} \label{e.Jactaylor}
	|J_x(v)-1|\leq C|v|^2,
	\end{equation}
	where $C$ only depends on \textit{scalar} curvature bounds on $\M$ (see \cite[Chapter 4]{do1992riemannian}), and is, in particular, uniform in~$x\in \M$; the latter fact follows from Rauch comparison theorem; see \cite[Section 2.2]{BIK}. 
	
	For the analysis in our proofs in section \ref{s.upperbound}, however, we need to impose some regularity assumptions on the function $v \mapsto D_v \exp_x(v)$ and also develop the Jacobian $J_x(v)$ to one degree higher. We will thus assume that the function $v \mapsto D_v\exp_x(v)$ satisfies
    \begin{equation}
       \lVert D_v \exp_x(\cdot ) \rVert_{C^{2,\alpha}} \leq L_I, 
       \label{eqn:RegularGeometry1}
    \end{equation}
    for all $x \in \M$. Here, $\lVert \cdot \rVert_{C^{2,\alpha}}$ is the canonical norm in the space of functions that are twice continuously differentiable and whose second derivatives are $\alpha$-H\"older continuous; the bound \eqref{eqn:RegularGeometry1} can be interpreted as an $\alpha$-H\"older continuity assumption on the rate of change of the intrinsic curvature (through sectional curvature) along the manifold. For the Jacobian, the bound \eqref{eqn:RegularGeometry1} implies 
	\begin{equation}\label{e.Jactaylor-third derivative}
	|J_x(v)- 1 + q_x(v)|\leq C(L_I)  |v|^3, 
	\end{equation}
	where $q_x(v)$ is a homogeneous polynomial of degree two that in particular is symmetric (i.e., $q_x(v)=q_x(-v)$). The constant $C(L_I)$ in the above inequality only depends on the constant $L_I$ from \eqref{eqn:RegularGeometry1}.

Finally, in normal coordinates the operator $\Delta_\rho$ can be written as
\begin{equation}
\label{eqn:LaplacianCoordinates}
    (\Delta_\rho f)(\exp_x(v)) = - \frac{1}{\rho \sqrt{\det(g)}} \sum_{i=1}^d \sum_{j=1}^d \frac{\partial }{\partial v_i} \left( \sqrt{\det(g)} \rho^2 g^{ij} \frac{\partial}{\partial v_j} f(\exp_x(v)) \right)  ,
\end{equation}
where $g^{ij}$ are the components of the inverse of the matrix $g(v):= D_v(\exp_x(v)) \cdot  D_v(\exp_x(v))^{{T}}$ (i.e., the metric tensor in normal coordinates.) 

\begin{remark}
For $\M$ and $\rho$ satisfying Assumption \ref{assump:MoreRegularity}, the coefficients $\{ a_{ij}\}_{ij}$ defined by
\[ a_{ij}:= \sqrt{\det(g)} \rho^2 g^{ij}  \]
fulfill the uniform ellipticity condition in \cite[Equation (2.32)]{FernndezReal2022} and are such that their $C^{2,\alpha}$ norm is controlled by a constant that only depends on the constants in Assumption \ref{assump:MoreRegularity}. 
\label{rem:RegCoeffciients}
\end{remark}
\nc


	\subsection{Euclidean Vs Geodesic Distances}

	While the bounds on the Jacobian of the exponential map can be interpreted as bounds on the manifold's ``intrinsic curvature" and its derivatives, for the purposes of our analysis, and given that our proximity graphs are based on the Euclidean distance and not on $\M$'s intrinsic geodesic distance, it will also be important to control the manifold's ``extrinsic curvature". Precisely, for $x,y\in\M$ such that $|x-y|\leq \e_n $ where $\e_n$ satisfies \eqref{eq:assumption:eps small}, it can be shown that 
	\begin{equation} \label{e.comparegeodesic}
	0 \le d(x,y)-|x-y|\leq C|x-y|^3\,,
	\end{equation}
	for a constant~$C > 0$ that depends on a bound on the \textit{second fundamental form} of $\M$ (as an embedded manifold in $\R^D$); see \cite[Chapter 6]{do1992riemannian} for a definition of the second fundamental form of a manifold embedded into another. Indeed, a geometric quantity that bounds the second fundamental form is the \textit{reach} of the manifold (denoted by $R$), and it can be defined as the largest radius $r>0$ such that for every $x \in \R^D$ with $ \inf_{y \in \M}|x -y| \leq r$ there is a unique closest point to $x$ in $\M$; see \cite[Proposition 2]{trillos2019error}). 
	
	However, just as for the intrinsic curvature, for our purposes it is also important to control the \textit{rate of change of the second fundamental form}, which, for our purposes, simply translates into having a bound like
	\begin{equation} \label{e.geodesic third derivative expansion}
	|d(x,y)-|x-y|-s(x,y)d(x,y)^3| \leq L_E \e_n^4,
	\end{equation}
	where $|s(x,y)|\leq C$ (see, e.g. \cite[Equation (18)]{trillos2023continuum}) and $s(x,y)$ is a symmetric function in its second variable, i.e., 
	\begin{equation}
	s(x, \exp_x(v) ) = s(x, \exp_x(-v)),  
	\label{eqn:symmetrySecondFundForm}
	\end{equation}
	for all $v \in \T_x\M$ with norm smaller than $\M$'s injectivity radius. The constant $L_E$ essentially provides a bound on the rate of change of the extrinsic curvature along the manifold. 

    The bounds on the rate of change of the second fundamental form can also be used to obtain the bound
    \begin{equation}
    \label{eqn:DifferenceGradients}
    \left|  \frac{P_x(y-x)}{|x-y|} - \frac{\log_x(y)}{d(x,y)} \right| \leq Cd(x,y)^2, \quad \forall x, y \in \M,
    \end{equation}
    where $P_x(\cdot)$ denotes the projection onto $T_x\M$. Indeed, this can be seen from the fact that, in case $d(x,y) \leq i_0$, we have  
  \[ -\frac{P_x(y-x)}{|x-y|} = \nabla_x |x-y| , \quad  -\frac{\log_x(y)}{d(x,y)} = \nabla_x d(x,y), \] 
  and from the expansion for the difference between $d(x,y)$ and $|x-y|$.
	
	\subsection{Taylor Expansions Along Geodesics}
	\label{app:Taylor}

	Let $\gamma:[0,1] \rightarrow \M$ be a constant speed geodesic with $\gamma(0) = x$ and $\dot{\gamma}(0) =v$. Let $f: \M \rightarrow \R$ be a twice continuously differentiable function and consider the function $t \in [0,1] \mapsto  \tilde{f}(t):= f(\gamma(t))$. We can carry out a standard first order Taylor expansion for $\tilde{f}$ around $0$ and write
	\[ \tilde{f}(t) = \tilde{f}(0) + \tilde{f}'(0) t + \mathcal{O}(t^2),  \]
	or in exact remainder form
	\[ \tilde{f}(t) =   \tilde{f}(0) + \tilde{f}'(0) t + \int_0^t (t-s) \tilde{f}''(s) \dd s.  \]
	By the chain rule, (i.e., the definition of the Riemannian gradient) 
	\[ \tilde{f}'(t) = \langle \nabla f (\gamma(t)) , \dot{\gamma}(t) \rangle. \]
	In particular, 
	\[  \tilde{f}'(0) = \langle \nabla f (x) , v \rangle. \]
	On the other hand,
	\begin{equation}
	{\tilde{f}}''(s)= \frac{d}{ds} \langle \nabla f(\gamma(s)) , \dot{\gamma}(s) \rangle =  \langle D^2 f(\gamma(s))(\dot \gamma(t)) , \dot{\gamma}(s) \rangle  +  \langle f(\gamma(s))  , \ddot{\gamma}(s) \rangle = \langle D^2 f(\gamma(s))(\dot \gamma(s)) , \dot{\gamma}(s) \rangle, 
	\label{eqn:SecondDerivative}
	\end{equation}
	where the second term is zero due to the fact that $\gamma$ is a geodesic.

	If $\gamma$ is the constant speed geodesic between $x$ and a point $y$ within distance $i_0$ from $x$, then we have $v = \log_x(y)$ and we can write
	\[ f(y)= f(x) + \langle \nabla f(x)  , v \rangle + \mathcal{O}(d_\M(x,y)^2),  \]
	as well as
	\[  f(y)= f(x) + \langle \nabla f(x)  , \log_x(y) \rangle +  \int_0^1(1- s)  \langle D^2 f(\exp_x(s v)) d(\exp_x)_{sv}(v), d(\exp_x)_{sv}(v) \rangle  \dd s,\]
	where here $d(\exp_x)_{sv}$ is the differential of $\exp_x$ at the point $sv \in T_x \M$, which maps vectors at $sv$ to vectors in $T_{\exp_x(sv)} \M$.
	
	Another useful identity in this context is
	\begin{equation}\label{eq:geodesic gamma computation}
	\langle \nabla f(y), -\log_y(x) \rangle =  \langle \nabla f(x), \log_x(y) \rangle + \mathcal{O}(\|D^2 f\|_{L^\infty(B_\M(x,\e_n))}d_\M(x,y)^2),
	\end{equation}
	which follows from the fundamental theorem of calculus, i.e.,  
	\[  \tilde{f}'(1) = \tilde{f}'(0) + \int_0^1 \langle D^2 f(\gamma(s))(\dot \gamma(s)) , \dot{\gamma}(s) \rangle  \dd s,   \]
	and the fact that $\dot{\gamma}(1)= - \log_y(x)$. 
	
	Finally, if in addition $f$ is assumed to be three times differentiable, we may consider the second order Taylor expansion with exact remainder given by
	\[  \tilde{f}(t) =   \tilde{f}(0) + \tilde{f}'(0) t +  \frac{\tilde{f}''(0)}{2}t^2 + \int_0^t \frac{(t-s)^2}{2} \tilde{f}'''(s) \dd s,  \]
	where we may compute $\tilde{f}'''(s)$ by differentiating \eqref{eqn:SecondDerivative}, which leads to
	\[ \tilde{f}'''(s) =  \sum_{i,j,k} \partial^3_{i,j,k} f(\exp_x(sv)) ((d\exp_x)_{sv}(v))_i ((d\exp_x)_{sv}(v))_j ((d\exp_x)_{sv}(v))_k.  \]
	
	\subsection{Tangent Bundle and Integration}
	\label{app:TangentBundle}
	
	The \textit{tangent bundle} $T\M$ of $\M$ is the set 
	\[ T\M : = \{  (x, v) \text{ s.t. } x \in \M, \quad v \in T_x\M  \}. \]
	As discussed in \cite[Chapter 3]{do1992riemannian}, $\M$'s Riemannian metric tensor induces a canonical Riemannian structure over $T\M$, and, in turn, a volume form that we denote by $\vol_{T\M}$. To describe integration with respect to $\vol_{T\M}$, consider, for $0< h < i_0$,  the set $\mathcal{B}_h \subseteq T\M$ defined by
	\[ \mathcal{B}_h := \{ (x,v) \in T\M \text{ s.t. } |v| \leq h  \},\]
	and let $g: T\M \rightarrow \R$ be a function supported on $\mathcal{B}_h$. Then
	\[ \int_{T \M } g((x,v)) \dd \vol_{T\M}((x,v)) = \int_\M \int_{B_h(0) \subseteq T_x \M } g((x,v)) \dd v \dx.          \]
	\textit{Liouville's theorem} (see \cite{do1992riemannian}) states that the \textit{geodesic flow} leaves $\vol_{T\M}$ invariant.

	\section{Proofs of Some Technical Lemmas}\label{A.proofs}

	\begin{lemma}[Poincar\'e–Wirtinger inequality; see Chapter 9 in \cite{brezis2011functional}]\label{lemma:poincare inequality}
		For any function $f\in H^1(\M)$ and density function $\rho\in\mathcal{P}_\M$, we have
		\begin{align}
		\int_\M |f(x)-\bar{f}|^2\rho(x)\dx\le\frac{1}{\lambda_\rho} \int_{\M}|\nabla f|^2\rho^2\dx,
		\end{align}
		where $\bar{f}$ is defined as
		\begin{align}
		\bar{f}\defeq \int_{\M} f\rho\dx\,.
		\end{align}
	\end{lemma}
	
	\begin{lemma}\label{lem:KL divergence upper bound}
		For density functions $f,g$ on $\M$ such that $g(x)>0$ for all $x\in \M$, we have
		\begin{align}\label{eq:appendix-lemma2-claim}
		\int_\M f(x) \log \frac{f(x)}{g(x)}\dx  \le \int_\M \frac{(f(x)-g(x))^2}{g(x)}\dx\,.
		\end{align}
	\end{lemma}
	\begin{proof}
		From the fact that $\log x\le x-1$ for any $x>0$, we have 
		\begin{equation}
		\begin{split}
		\int_\M f(x) \log \frac{f(x)}{g(x)}\dx
		&\le \int_\M f(x) \left(\frac{f(x)}{g(x)}-1\right)\dx\\
		&= \int_\M  \frac{(f(x)-g(x))^2}{g(x)}\dx,
		\end{split}
		\end{equation}
		where the last equality follows from the fact that $f,g$ are density functions on $\M$.
	\end{proof}

	\begin{lemma}[Poincar\'{e}  inequality for a weighted grid graph]
		\label{lem:PoincareGrid}
		Let $\ell \in \N$ and define the set 
		\[V_\ell \defeq  \{  v =(v_1, \dots, v_d) \in \Z^d  \: \mathrm{ s.t. } \: v_i \in \{0, \dots, \ell -1 \} \text{ for all } i = 1, \dots, d   \}.\] 
		For two elements $v, \tilde v  \in V$ we say $v \sim \tilde v$ if $|v- \tilde v | \leq 1$, where $|\cdot|$ denotes the Euclidean norm. Let $w : V_\ell \rightarrow \R$ be a function satisfying 
		\[  0 < c_1 \leq w(v)  \leq c_2, \quad \forall  v \in V_\ell. \]
		Then for any $g : V_\ell \rightarrow \R $ we have 
		\[ \sum_{v  \in V_\ell}   |  g(v)  - \overline{g}_w  |^2 \leq C \ell^2 \sum_{v, \tilde v \in V_\ell, \: v \sim \tilde v } |  g(v) - g(\tilde v)  |^2,   \]
		where 
		\[  \overline{g}_w  \defeq   \frac{1}{\sum_{v \in V_\ell} w(v) }  \sum_{v \in V_\ell } g(v) w(v)    \]
		and where the constant $C$ may depend on $d$ but not on $\ell$.
	\end{lemma}
	
	\begin{proof}
		If $w \equiv 1$ and $d=1$  the result is readily available from \cite{chung2007four} and can be extended to the case $d>1$ by a simple tensorization argument. In what follows we provide an extension of the desired inequality to weight functions $w$ that are not necessarily constant but are lower and upper bounded by positive constants. Our argument is based on Cheeger's inequality.

		We consider the weighted graph Laplacian $\mathcal{L}_w$ defined by
		\[   \mathcal{L}_w g(v) = \frac{1}{w(v)} \sum_{\tilde v \in V_\ell \text{ s.t. } \tilde v  \sim v } ( g(v) - g(\tilde v )).\] 
		Observe that $ \mathcal{L}_w$ is self-adjoint with respect to the inner product $\langle \cdot, \cdot \rangle_w$ defined by 
		\[\langle g , \tilde g \rangle_w  \defeq  \sum_{v \in V_\ell} g(v) \tilde g (v) w(v) .\] 
		It follows that
		\begin{equation}\label{eq:cheeger inequality}
		\sum_{v  \in V_\ell}   |  g(v)  - \overline{g}_w  |^2 w(v)  \leq \frac{1}{\lambda_{2,w}} \sum_{v, \tilde v \in V_\ell, \: v \sim \tilde v } |  g(v) - g(\tilde v)  |^2,
		\end{equation}
		for every $g: V_\ell \rightarrow  \R $. Here $\lambda_{2,w}$ denotes the first non-zero eigenvalue of $\mathcal{L}_w$.  Now, Cheeger's inequality (whose proof in \cite[Theorem 1]{chung2007four} can be immediately adapted to our setting) implies that
		\[ \lambda_{2, w} \geq \frac{\min_{v} d(v)}{8\max_{v} w(v)}(\mathrm{Cut}^*)^2\geq \frac{1}{8c_1}(\mathrm{Cut}^*)^2, \]
		where $d(v)$ is the degree of the vertex $v$,
		\[ \mathrm{Cut}^*  \defeq  \inf_{S \: \text{ s.t. } \:  0 <\mathrm{Vol}(S) \leq \frac{1}{2} \mathrm{Vol}(V_\ell)  }  \frac{\mathrm{Cut} (S)}{\mathrm{Vol}(S)} ,\]
		and 
		\begin{equation}
		\begin{split}
		\mathrm{Cut} (S) & \defeq  \sum_{v\in S,\tilde{v}\in S^C \ \ \mathrm{s.t.} \ v\sim \tilde{v}} 1,\\
		\mathrm{Vol}(S) & \defeq  \sum_{v\in S} 1.
		\end{split}
		\end{equation}

		To get the value of $\Cut^*$, it suffices to notice that, in the uniform setting, the Cheeger cut is minimized by a half-space, parallel to one of the axis, bisecting the cube induced by $V_\ell$. Therefore, we have 
		\begin{equation*}
		\mathrm{Cut}^*\geq \frac{c}{\ell}.
		\end{equation*}
		By plugging this estimate into the lower bound for $\la_{2,w}$ we obtain 
		\begin{equation*}
		c_1\sum_{v  \in V_\ell}   |  g(v)  - \overline{g}_w  |^2   \leq \frac{1}{\lambda_{2,w}} \sum_{v, \tilde v \in V_\ell, \: v \sim \tilde v } |  g(v) - g(\tilde v)  |^2
		\leq C \ell^2 \sum_{v, \tilde v \in V_\ell, \: v \sim \tilde v } |  g(v) - g(\tilde v)  |^2.
		\end{equation*}
		This completes the proof.
	\end{proof}

	\section{Upper bounds Using Kernel Density Estimation}\label{sec:upper bound kde}
	We provide an argument based on perturbation theory to analyze the plug-in estimator $(\la_l(\hat{p}_n), f_l(\hat{\rho}_n))$ that was discussed in the introduction of the paper. For this, we first obtain some deterministic estimates that we later combine with well-known error bounds for density estimation.

	\subsection{Estimates From Perturbation Theory}
	\label{app:EstimatesPerturbationTheory}
	For the discussion in this subsection, we assume less regularity on density functions and in particular we consider $\mathcal{P}^1_{\M}$, the class of density functions over $\M$ defined as $\mathcal{P}_\M$ (Definition \ref{def:DensityClass}) but without the requirement that second derivatives are uniformly bounded.
	
	Let $\rho_0, \rho_1 \in \mathcal{P}^1_{\M}$. By standard elliptic theory, for each $i=0,1$ there is a family of eigenvalues~$\{\lambda_{i,k}: k \in \NN\}$ of the weighted Laplace-Beltrami operator~$\Delta_{\rho_i}$ and associated eigenfunctions~$\{f_{i,k}:i=0,1; k\in\NN\}$. For each~$i=0,1,$ these eigenfunctions form a Hilbert basis of the weighted~$L^2$ space with weight~$\rho_i.$ Furthermore, for each~$i=0,1$, the eigenvalues~$\la_{i,k},$ have finite multiplicity and thus can be arranged in non-decreasing order, repeated according to their multiplicity, and satisfy~$\lambda_{i,k} \to \infty$ as~$k \to \infty.$ 
	
	For a~$t \in \R$ with~$|t| \ll 1, $ we define $\rho_t$ via
	\begin{equation*}
	\rho_t\defeq \rho_0+t (\rho_1 - \rho_0)\,. 
	\end{equation*}
	Since~$\int\rho_1\dx = \int\rho_0\dx, $ it follows that~$\int \rho_t\dx = \int \rho_0\dx$ for every~$t.$ Moreover, for~$|t|$  small, it is clear that~$\rho_t \in \mathcal{P}^1_{\M}$ (by potentially modifying slightly the constants in the definition of $\mathcal{P}^1_{\M}$). Let us now introduce the family of operators
	\begin{equation*}
	\mathcal{L}_t\defeq \Delta_{\rho_t}= -\frac{1}{\rho_t} \div(\rho_t^2\nabla \cdot  )\,, \quad t \in \R.
	\end{equation*}
	Then, clearly, for small enough $t$~$\mathcal{L}_t$ defines a family of elliptic second-order differential operators that depend analytically on~$t$ in a neighborhood of~$t = 0.$ It is well-known~(see \cite{Kato}) that if~$\lambda_0 = \lambda_{0,l}$ is a \emph{simple} eigenvalue of~$\Delta_{\rho_0}$ then there exists~$t_{l} > 0$ and an analytic branch~$\{\lambda_{t,l}, f_{t,l}\}_{|t|< t_{l}}$ of \emph{simple} eigenvalue-eigenfunction pairs for the operator~$\mathcal{L}_t.$ When~$\lambda_{0,l}$ is not simple, one must prove analyticity by studying the entire eigenspace associated with~$\lambda_{0,l};$ since this is a standard computation that can also be carried out here in the standard way~\cite{Kato}, and since eigenvalues are generically simple, we only consider the case of a simple eigenvalue in this appendix. Our goal is to compute the infinitesimal quantities
	\begin{equation*}
	\dot{\lambda}  \defeq  \frac{\,d}{\,dt}\big\vert \lambda_{t,l}, \mbox{ and } \Bigl\|\frac{\,d}{\,dt}\big\vert f_{t,l}\Bigr\|_{H^1(\M)}\, 
	\end{equation*}
	and we focus on the case $t=0$ for simplicity. Integration over time of the above bounds provides bounds on the difference between the eigenvalues and eigenfunctions of the operators in the family.  
	
	For notational ease, we set~$\dot{f}  \defeq  \frac{\,d}{\,dt}\big\vert_{t=0} f_{t,l}.$ Since we have 
	\begin{equation}
	\label{e.LtPDE}
	-\mathrm{div}\bigl( \rho_t^2 \nabla f_{t,l} \bigr) = \lambda_{t,l}\rho_{t,l}f_{t,l}\,,
	\end{equation}
	with the normalization
	\begin{equation} \label{e.Ltnormal}
	\int_\M f_{t,l}^2 \rho_t  \dx = 1\,
	\end{equation}
	for~$|t|$ small, we may differentiate \eqref{e.LtPDE} with respect to~$t$ and set~$t=0$ to obtain  
	\begin{equation} \label{e.phidot}
	-\mathrm{div}\bigl( \rho^2 \nabla \dot{f}\bigr) - \lambda \rho \dot{f} = 2 \mathrm{div}\bigl( \rho \dot{\rho} \nabla f\bigr) + \bigl(\lambda \dot{\rho} + \dot{\lambda} \rho \bigr)f\,,
	\end{equation}
	where we write~$\rho  \defeq  \rho_0, f  \defeq  f_0$ for brevity.
	
	By the Fredholm alternative applied to the operator $-\div(\rho^2 \nabla \cdot) - \lambda \rho \cdot $, and the fact that~$\lambda = \lambda_{0,l}$ is simple, the right hand side of~\eqref{e.phidot} must be orthogonal to~$f$ in the standard~$L^2$ inner product (since $f$ is in the kernel of the aforementioned operator). So, testing this equation with~$f$ yields 
	\begin{equation*}
	0 = - 2\int_{\M} \rho \dot{\rho} |\nabla f|^2 + \lambda \int \dot{\rho} f^2 + \dot{\lambda}.
	\end{equation*}
	Rearranging, we obtain
	\begin{equation*}
	\dot{\lambda} = - \lambda_{0,l} \int_{\M} (\rho_1 - \rho_0)f^2 \dx + 2 \int_{\M} \rho_0(\rho_1 - \rho_0) |\nabla f|^2\,\dx\,.  
	\end{equation*}
	Using the regularity of $f$ (due to the regularity of $\rho$), we immediately get
	\begin{equation} \label{e.lambdadot}
	|\dot{\lambda}| \leqslant C_l \|\rho_1 - \rho_0\|_{L^2(\M)}\,,
	\end{equation}
	for a constant~$C_l$ that depends on, for example,~$\|\nabla f_{0,l}\|_{L^\infty(\M)}$, which can be uniformly controlled over $\rho_0 \in \mathcal{P}^1_{\M}$.

	\medskip 
	
	Turning next to~$\dot{f}$, we note from \eqref{e.phidot} that it solves equation 
	\begin{equation} \label{e.phidot.2}
	-\mathrm{div}\bigl( \rho^2 \nabla \dot{f}\bigr) - \lambda \rho \dot{f} = g,
	\end{equation}
	for a right hand side satisfying
	\[ \lVert g \rVert_{H^{-1}(\M)} \leq C_l \lVert \rho_1- \rho_0 \rVert_{L^2(\M)}, \]
	for a constant $C_l$ that, as before, may depend on $\lVert \nabla f_{0,l} \rVert_{L^\infty(\M)}$. We introduce some notation that will be useful in the rest of the discussion. First, define
	\[ \lVert u \rVert_{H^1(\rho^2)} := \left(\int_\M |\nabla u|^2 \rho^2 \dx \right)^{1/2}\,, \quad u \in H^1(\M), \]
	and consider the subspaces 
	\[ S_{l-1}:= \text{span}\{ f_{0,1}, \dots, f_{0,l-1} \}, \quad  S_{l}:= \text{span}\{ f_{0,l} \}, \quad \text{and} \quad  S_{l+1}:= \text{span}\{ f_{0,l+1}, f_{0,l+2}, \dots \}.\] 
	For a given $u$, we denote by $u_{l-1}, u_l, u_{l+1}$ the projections of $u$ onto $S_{l-1}, S_l, S_{l+1}$, respectively, with respect to the $H^1(\rho^2)$ inner product. Note that 
	\[ \lVert u \rVert_{H^1(\rho^2)}^2 = \lVert u_{l-1} \rVert_{H^1(\rho^2)}^2   + \lVert u_l \rVert_{H^1(\rho^2)}^2  + \lVert u_{l+1} \rVert_{H^1(\rho^2)}^2.   \]
	In what follows we consider $u:= \dot{f}$ and seek to bound each of the terms on the right hand side of the above expression. 
	
	To bound $\lVert u_{l-1} \rVert_{H^1(\rho^2)}$, we take $\xi \in S_{l-1}$ with $\lVert \xi \rVert_{L^2(\rho)} \leq 1$ and test equation \eqref{e.phidot.2} against $\xi$ to get
	\begin{align*}
	- \lambda_{0,l} \int_\M u \xi \rho \dx & = \int_\M g \xi \dx -  \int_\M \nabla u \cdot \nabla \xi \rho^2 \dx 
	\\& \leq  \int_\M g \xi \dx + \lVert u_{l-1} \rVert_{H^1(\rho^2)}  \lVert \xi \rVert_{H^1(\rho^2)}
	\\& \leq \int_\M g \xi \dx + \lambda_{0,l-1}\lVert u_{l-1} \rVert_{L^2(\rho)}  \lVert \xi \rVert_{L^2(\rho)}.
	\end{align*}
	Taking the sup over all such $\xi$, and rearranging the resulting expression, we get
	\[  \lVert  u_{l-1}\rVert_{L^2(\rho)} \leq \frac{1}{\lambda_{0,l} - \lambda_{0, l-1}} \sup_{\xi \in S_{l-1} \text{ s.t. } \lVert \xi \rVert_{L^2(\rho)} \leq 1} \int_\M g \xi \dx. \]
	Now, a straightforward computation reveals that if $\xi \in S_{l-1}$ and $\lVert  \xi \rVert_{L^2(\rho)} \leq 1$, then $\lVert  \xi \rVert_{H^1(\rho^2)} \leq \sqrt{\lambda_{0,l-1}}$. We conclude from this that
	\[ \lVert u_{l-1} \rVert_{H^1(\rho^2)} \leq \sqrt{\lambda_{0,l-1}} \lVert u_{l-1} \rVert_{L^2(\rho)} \leq C\frac{\lambda_{0,l-1}}{\lambda_{0,l} - \lambda_{0,l-1}} \lVert g \rVert_{H^{-1}(\M)} \leq C_l  \frac{\lambda_{0,l-1}}{\lambda_{0,l} - \lambda_{0,l-1}} \lVert \rho_1 - \rho_0 \rVert_{L^2(\M)} . \]
	
	Next, we bound $\lVert u_{l+1} \rVert_{H^1(\rho^2)}$. For this, we take $\xi \in S_{l+1}$ with $\lVert \xi \rVert_{H^1(\rho^2)} \leq 1$ and test equation \eqref{e.phidot.2} against $\xi$ to get
	\begin{align*}
	\int_\M \nabla u \cdot \nabla \xi \rho^2 \dx  & = \int_\M g \xi \dx + \lambda_{0,l}\int_{\M} u \xi \rho \dx
	\\& \leq C \lVert g \rVert_{H^{-1}(\M)} + \lambda_{0,l} \lVert u_{l+1} \rVert_{L^2(\rho)} \lVert \xi \rVert_{L^2(\rho)}
	\\& \leq  C \lVert g \rVert_{H^{-1}(\M)}  +  \frac{\lambda_{0,l}}{\lambda_{0, l+1}} \lVert u_{l+1} \rVert_{H^1(\rho^2)} \lVert \xi \rVert_{H^1(\rho^2)}. 
	\end{align*}
	Taking the sup over all such test functions $\xi$, and rearranging the resulting expression, we obtain 
	\[ \lVert u_{l+1} \rVert_{H^1(\rho^2)} \leq C \frac{\lambda_{0, l+1}}{\lambda_{0,l+1}- \lambda_{0,l}} \lVert g \rVert_{H^{-1}(\M)} \leq C_l \frac{\lambda_{0, l+1}}{\lambda_{0,l+1}- \lambda_{0,l}} \lVert \rho_1 - \rho_0 \rVert_{L^2(\M)}.       \]
	
	Finally, to bound $\lVert u_l  \rVert_{H^1(\rho^2)}$, we note that
	\begin{align*}
	\left| \int_\M \nabla u \cdot \nabla f_{0,l}  \rho^2 \dx \right| & =   \lambda_{0,l} \left| \int_\M u f_{0,l} \rho \dx \right|
	\\& = \frac{\lambda_{0,l}}{2} \left| \int_\M f_{0,l}^2 \dot{\rho} \dx \right|
	\\& \leq C_l \lambda_{0,l} \lVert \rho_1 - \rho_0 \rVert_{L^2(\M)},
	\end{align*}
	where in the second line we used the identity 
	\[ 2\int_\M  \dot{f} f \rho \dx + \int_\M f^2 \dot{\rho} \dx = 0, \]
	which follows by differentiating \eqref{e.Ltnormal} in time. We conclude that 
	\[ \lVert u_l \rVert_{H^1(\rho^2)} \leq C_l \Vert \rho_1 - \rho_0 \rVert_{L^2(\M)}.  \]

	Since $ \lVert \dot{f} \rVert_{H^1(\M)} \leq C \lVert \dot{f} \rVert_{H^1(\rho^2)}$, combining the above estimates we conclude that
	\begin{equation}
	\label{eq:gradient of eigenfunction is upper bounded by the density perturbation}
	\lVert \dot{f} \rVert_{H^1(\M)} \leq C_l \frac{\lambda_{0,l+1}}{\gamma_l} \lVert \rho_1 - \rho_0 \rVert_{L^2(\M)}, 
	\end{equation}
	where $C_l$ is a constant that depends on $\lVert \nabla f_{0,l}  \rVert_{L^\infty(\M)}$ and where we recall $\gamma_l$ is the spectral gap \eqref{eq-def:gamma_l}. 

	\subsection{Eigenpair Estimators via KDE}


	Let $\rho_0$ be a density in the class $\mathcal{P}_{\M}$ introduced in Definition \ref{def:DensityClass} and suppose that~$\X_n= \{ x_1, \dots, x_n\}$ are i.i.d samples from~$\rho_0$. Let $\hat{\rho}_n$ be the \textit{kernel density estimator}
	\[\hat{\rho}_n(x) \defeq \frac{1}{n} \sum_{i=1}^n K_{r_n}( x_i- x), \text{ for } x\in\M,\]
	for a smooth kernel $K$ that for simplicity here we take to be a standard Gaussian kernel. We set $r_n= C n^{-1/(d+4)}$ and use the notation $K_{r_n}(x)= \frac{1}{r_n^d} K\left( \frac{x}{r_n}\right)$.
	
	The following result follows from an application of standard concentration inequalities and a straightforward estimate for the bias of the above density estimator.  
	
	\begin{lemma} \label{l.KDE}
		
		For any~$s>0$ and any~$p \in [1,\infty),$ we have 
		\begin{equation}\label{eq:kernel density estimation rate p}
		\P\bigl[ \|\rho_0-\hat{\rho}_n\|_{L^p(\M)}  > C_p s + Cr_n^2 \bigr]	\le  C\exp(-cnr_n^d s^2),
		\end{equation}
		
		\begin{equation}\label{eq:kernel density estimation rate infty}
		\P\bigl[ \|\rho_0-\hat{\rho}_n\|_{L^\infty(\M)}  > s \bigr]	\le  2n\exp(-Cnr_n^d s^2),
		\end{equation}
		and 
		\begin{equation}
		\P\bigl[ \|\nabla \hat{\rho}_n\|_{L^\infty(\M)}  > C \bigr]	\le  2n\exp(-Cnr_n^{d+2}).
		\end{equation} 
		In the above, $C$ is a constant that only depends on the parameters defining $\mathcal{P}_\M$, and   $C_p$ is a constant that may, in addition, depend on the power $p$.
	\end{lemma}

	
	A particular consequence of the above bounds is that, with very high probability, $\hat{\rho}_{n} \in \mathcal{P}^1_{\M}$ (adjusting some of the constants in the family if necessary) when we take $r_n \sim n^{-1/(d+4)}$. We may thus take $\rho_1= \hat{\rho}_n$ in the discussion from subsection \ref{app:EstimatesPerturbationTheory} and use in particular~$\eqref{e.lambdadot}$ and \eqref{eq:gradient of eigenfunction is upper bounded by the density perturbation} to bound 
	\begin{align*}
	\E\biggl[  |\dot{\lambda}| + \|\dot{f}\|_{H^1(\M)}\biggr]  & \leqslant \frac{C_l}{\gamma_l}\E[\|\hat{\rho}_n - \rho_0\|_{L^2(\M)}\mathbf{1}_{\{\|\nabla \hat{\rho}_n\|_{L^\infty(\M)}\leq C\}}] + C \exp(-cnr_n^{d+2})
	\\ & \leq  \frac{C_l}{\gamma_l}\biggl( \frac{1}{n}  \biggr)^{\frac2{d+4}}\,. 
	\end{align*}
	
	
	\begin{remark}
		An interesting and nontrivial observation that follows from the computations in this section is that, even if $\Delta_\rho$ involves the gradient of $\rho$, in order to get the optimal estimation rates for $\Delta_{\rho}$'s eigenpairs we do not require estimating $\nabla \rho$, but only $\rho$. In fact, the only thing needed is that the gradient of the density estimator is bounded so as to apply the deterministic bounds from perturbation theory of elliptic operators that we presented in section \ref{app:EstimatesPerturbationTheory}. Thanks to these deterministic bounds, an upper bound for density estimation in $L^2(\M)$ implies the same upper bound for eigenpair estimation in the $H^1(\M)$-sense. However, as discussed in the introduction, the lower bounds for eigenpair estimation do not follow from the lower bounds for density estimation and instead it was important to carry out the detailed analysis that we presented in section \ref{sec:lower bound}.
	\end{remark}
	
	\section{Concentration Inequalities}
	\label{app:Concentration}
	
We recall the following standard concentration bound for sums of i.i.d. random variables.
\begin{lemma}[Bernstein's inequality]
\label{lemm:Bernstein}
Let $\xi_1, \dots, \xi_n$ be i.i.d. real-valued random variables for which
\[ \var(\xi) \leq \sigma^2 , \quad |\xi_i| \leq M.\]
Then
\[ \P\left[ \left| \frac{1}{n}\sum_{i=1}^n \xi_i - \E \left[ \frac{1}{n}\sum_{i=1}^n \xi_i \right] \right| \geq t \right] \leq 2 \exp\left(- \frac{cnt^2}{\sigma^2 + Mt} \right), \quad \forall t >0.  \]
\end{lemma}

\medskip

Next, we state a Bernstein-type concentration bound for $U$-statistics of the form 
	\begin{equation}
	U_n=\avsum_{x,y \in \X_n} \eta_{\e_n}\left(d(x,y)\right) \mathcal{K}\left(x, y\right)\,,
	\label{def:UStatistic}
	\end{equation}
	for a kernel function $\mathcal{K}: \M \times \M \to \R$.
    
	\begin{lemma} \label{l.boundeddiff}
		For a given function $\mathcal{K}: \M \times \M \rightarrow \R$, consider the $U$-statistic $U_n$ defined as in \eqref{def:UStatistic} by using data sampled from a distribution $\rho$ in the class $\mathcal{P}_\M$ for some manifold $\M$ in the family $\MM$. Suppose that 
		\begin{equation}
		\label{defBK}
		\int_\M\left(\int_{B_1(0) \subseteq \T_x\M} \eta(|v|)\mathcal{K}(x, \exp_x(\e_n v) )\dd v\right)^2  \dx \leq B_\mathcal{K} <  \infty, 
		\end{equation}
		and 
		\begin{equation}
		\label{def:CK}
		\sup_{x\in\M}\sup_{v\in B_{1}(0) \subseteq T_x\M } \left|\mathcal{K}(x,\exp_x(\e_n v))\right|\leq C_\mathcal{K}< \infty.
		\end{equation}
		Then
		$$
		\mathbb{P}\left[\left|U_n - \E(U_n) \right|>t\right] \leq 2 \exp \left(-\frac{c n \e_n^d t^2 }{B_{\mathcal{K}}+t  C_\mathcal{K}}\right) \,,
		$$		
		where
		$c$ is a constant that only depends on $\eta$ and the geometric and smoothness constants on $\mathcal{M}, \rho$ mentioned in Definitions \ref{def:ManifoldClass} and \ref{def:DensityClass1}.
	\end{lemma}
	
	The above result can be adapted from the proof of \cite[Theorem 2]{arcones1995bernstein} to the manifold case and for this reason details are omitted. Indeed, the curved manifold setting does not change the structure of the proof and the only difference with the analysis in the flat case is the additional geometric constants that appear from writing expectations as integrals in normal coordinates, which can be controlled by the quantities appearing in Definitions \ref{def:ManifoldClass} and \ref{def:DensityClass1}.
	
	\medskip

\end{document}